\documentclass[pdflatex,sn-basic]{sn-jnl-noheader} 
\jyear{2021}

\usepackage{relsize}
\usepackage{subcaption}
\usepackage{tabularx}
    \newcolumntype{P}[1]{>{\centering\arraybackslash}p{#1}}
    \newcolumntype{Y}{>{\centering\arraybackslash}X}
\usepackage{./Latex/macros}

\theoremstyle{thmstyleone}%
\newtheorem{theorem}{Theorem}
\newtheorem{proposition}{Proposition}

\theoremstyle{thmstyletwo}%
\newtheorem{remark}{Remark}%

\theoremstyle{thmstylethree}%
\newtheorem{definition}{Definition}%

\begin{document}

\title[FADE: FAir Double Ensemble Learning]{FADE: FAir Double Ensemble Learning for Observable and Counterfactual Outcomes}

\author*[1,2]{\fnm{Alan} \sur{Mishler}}\email{contact@alanmishler.com}

\author[1]{\fnm{Edward} \sur{Kennedy}}\email{edward@stat.cmu.edu}

\affil*[1]{\orgdiv{Department of Statistics \& Data Science}, \orgname{Carnegie Mellon University}, \orgaddress{\city{Pittsburgh}, \state{PA}, \postcode{15213}, \country{USA}}}

\affil[2]{\orgdiv{J. P. Morgan AI Research}, \orgaddress{\city{New York}, \state{NY}, \country{USA}}}

\abstract{Methods for building fair predictors often involve tradeoffs between fairness and accuracy and between different fairness criteria, but the nature of these tradeoffs varies. Recent work seeks to characterize these tradeoffs in specific problem settings, but these methods often do not accommodate users who wish to improve the fairness of an existing benchmark model without sacrificing accuracy, or vice versa. These results are also typically restricted to observable accuracy and fairness criteria. We develop a flexible framework for fair ensemble learning that allows users to efficiently explore the fairness-accuracy space or to improve the fairness or accuracy of a benchmark model. Our framework can simultaneously target multiple observable or counterfactual fairness criteria, and it enables users to combine a large number of previously trained and newly trained predictors. We provide theoretical guarantees that our estimators converge at fast rates. We apply our method on both simulated and real data, with respect to both observable and counterfactual accuracy and fairness criteria. We show that, surprisingly, multiple unfairness measures can sometimes be minimized simultaneously with little impact on accuracy, relative to unconstrained predictors or existing benchmark models.}


\maketitle

\section{Introduction} \label{sec:intro}
Classification and regression models are increasingly widely used to inform or render decisions in domains such as healthcare, criminal justice, education, hiring, and consumer finance. Given the high-stakes nature of such decisions, it is important to ensure that these models are both accurate, to maximize their overall benefits and minimize their overall harms; and fair, so that the benefits and harms do not accrue disproportionately to already (under)privileged groups. In recent years, there have been many well-publicized cases of algorithmic systems whose performance varies over sensitive features such as race and gender in ways that appear to harm marginalized populations \citep{Angwin2016, buolamwini_gender_2018, obermeyer_dissecting_2019}.

In response to concerns such as these, the algorithmic fairness community has developed a wide array of methods for removing or minimizing unfairness in models. In some cases, the most accurate models under consideration do not satisfy a chosen fairness criterion, so there is a fairness-accuracy tradeoff \citep{friedler_comparative_2019, menon_cost_2018, zhao_inherent_2019}. Many methods therefore aim to maximize predictive accuracy subject to a bound on some quantitative unfairness criterion \citep{zafar_fairness_2017, donini_empirical_2018, woodworth_learning_2017}. Some methods adopt a complementary perspective, seeking to minimize unfairness subject to an accuracy constraint \citep{zafar_fairness_2017, coston_characterizing_2021}. Many strict versions of fairness criteria are pairwise unsatisfiable in real-world settings, so there may also be fairness-fairness tradeoffs \citep{Chouldechova2017, kleinberg2017, kim_fact_2020}.  

In some cases, however, tradeoffs are small to nonexistent: model fairness can be increased with minimal loss of accuracy, or vice versa \citep{dutta_is_2020, coston_characterizing_2021, rodolfa_empirical_2021}. Recently there has been growing interest in characterizing these tradeoffs both theoretically and empirically for specific problems and specific classes of models \citep{berk_convex_2017, kim_fact_2020, liu_accuracy_2021}. Current methods for illuminating these tradeoffs do not readily accommodate users who wish to improve the fairness and/or accuracy of an existing benchmark model rather than exploring the fairness-accuracy space. Additionally, these methods are designed to handle \emph{observable} accuracy and fairness criteria, i.e. criteria that depend on observable outcomes. They do not address \emph{counterfactual} accuracy and fairness criteria, which depend on counterfactual outcomes and which are relevant to many settings in which algorithms are used to support decision making. In general, there are very few methods designed to build predictors that satisfy counterfactual versions of common fairness criteria like equalized odds \citep{mishler_fairness_2021}.   

To address these limitations, we propose \emph{FAir Double Ensemble learning (FADE)}, a simple and flexible framework that builds predictors as weighted combinations of basis functions that are chosen by the user. Within this framework, we develop three methods: (1) minimizing risk subject to fairness constraints, (2) minimizing unfairness subject to a risk constraint, and (3) efficiently generating a large class of unfairness-penalized predictors. The weights in method (3) have a closed-form expression that varies smoothly over a vector of unfairness penalty parameters, allowing users to trace out paths in fairness-accuracy spaces. It is computationally extremely fast to compute and evaluate thousands or even tens of thousands of models of this form. These methods accommodate users who wish to improve the fairness of an existing model without sacrificing accuracy, or vice versa, or who wish to understand fairness-accuracy and fairness-fairness tradeoffs in their problem.

In sum, FADE has the following properties:
\begin{itemize}
    \item It allows users to target specific fairness and accuracy constraints as well as to efficiently explore fairness-accuracy and fairness-fairness tradeoffs.
    \item It can target a range of both observable and counterfactual accuracy and fairness criteria, separately or simultaneously.
    \item It enables users to combine previously trained and newly trained predictors, thereby collapsing the distinction between in-processing and post-processing approaches to building fair predictors.
    \item In the context of counterfactual accuracy and fairness, it utilizes doubly robust estimators, which yield fast convergence rates under relatively weak nonparametric assumptions. The excess risk and excess unfairness of our estimators, suitably defined, shrink to 0 at up to $\sqrt{n}$ rates even when relevant nuisance parameters are estimated at slower rates that are typical in nonparametric machine learning.
\end{itemize}
The remainder of the paper is organized as follows. In Section \ref{sec:background}, we discuss background and related work. In Section \ref{sec:setup_and_estimands}, we formalize our problem and define three estimands, all of which are formulated as optimization problems. The first estimand involves minimizing risk subject to a fairness constraint, the second involves minimizing unfairness subject to a risk constraint, and the third utilizes penalty terms rather than constraints to trace out paths in fairness-accuracy space. Section \ref{sec:identification} provides a set of assumptions necessary to relate counterfactual quantities to observable data. In Section \ref{sec:estimation}, we define our estimators and give theoretical guarantees for their performance. We illustrate FADE on simulated data in Section \ref{sec:simulations} and on real data in Sections \ref{sec:compas} and \ref{sec:adult}.\footnote{Code to reproduce the results will be made available in the Github repository \href{https://github.com/amishler/FAir-Double-Ensemble-Learning}{amishler/FAir-Double-Ensemble-Learning}.} In a counterfactual setting, FADE substantially improves both the fairness and accuracy of the COMPAS recidivism predictor (Section \ref{sec:compas}). In an observable setting, FADE yields many predictors that perform comparably to or better than other fairness methods on an income prediction task, while allowing users much more flexibility in the final model form (Section \ref{sec:adult}). We conclude in Section \ref{sec:conclusion}.

\section{Background and related work} \label{sec:background}
We use the terms ``predictor'' and ``model'' interchangeably to refer to any mapping from covariates to outputs that is intended to estimate an unknown quantity, whether that quantity is an unobserved label or an as-yet unrealized outcome. We use ``accuracy'' or ``risk'' to refer to any measure that tracks how well a predictor estimates the target quantity, such as mean-squared error or 0-1 error, and ``performance'' to refer to a model's joint accuracy and fairness characteristics.

\subsection{Ways of achieving fairness}
The fairness literature generally distinguishes three approaches for developing fair predictors. \emph{Pre-processing} approaches transform the input data to remove bias \citep{Calmon2017, feldman_certifying_2015, kamiran_data_2012, zemel_learning_2013}. \emph{In-processing} or \emph{in-training} approaches enforce fairness via constraints or regularization terms during the learning process \citep{donini_empirical_2018, hutchison_fairness-aware_2012, woodworth_learning_2017, zafar_fairness_2017}. \emph{Post-processing} approaches learn functions to map the outputs of existing predictors to new outputs \citep{hardt_equality_2016, pleiss_fairness_2017, kim_multiaccuracy_2019}. Our approach enables users to combine previously existing predictors with newly trained predictors or other basis functions, essentially collapsing the distinction between in-processing and post-processing.

\subsection{Observational and counterfactual fairness} \label{subsec:fairness_criteria}
Many popular fairness criteria place restrictions on the joint distribution of predictions, outcomes, and a sensitive feature. For example, the criterion of \emph{independence}, also known as \emph{statistical parity} or \emph{demographic parity}, requires that the predictions be marginally independent of the sensitive feature \citep{calders_building_2009, barocas2018fairness}, while \emph{separation} or \emph{equalized odds} requires that they be independent conditional on the outcome \citep{hardt_equality_2016}. Criteria such as equalized odds that depend on the outcome may be defined with respect to observable or potential (aka ``counterfactual'') outcomes. Counterfactual versions of these criteria are appropriate in risk assessment settings, i.e. settings in which the model is meant to estimate the risk of an adverse outcome absent a specific intervention \citep{coston_counterfactual_2020}. In these settings, the potential outcomes of interest are the outcomes that would occur if, possibly counterfactually, a decision variable were set to some baseline level \citep{neyman1923justification, Holland1986}. Examples arise in areas such as healthcare, where doctors must predict who would develop complications without further treatment; criminal justice, where judges must predict who would recidivate if released pretrial; and consumer finance, where banks must predict who would default if issued a loan.

A distinct set of causal fairness criteria consider counterfactuals with respect to the sensitive feature rather than with respect to a decision variable \citep{Kilbertus2017, Kusner2017, nabi_fair_2018, zhang_fairness_2018, nabi_learning_2019, wang_equal_2019}. These criteria consider questions like ``what would the risk prediction be if the defendant had been of a different race their whole life?'' rather than ``what would the outcome be if this person were released pretrial?'' We do not consider these criteria here; see \citep{mishler_fairness_2021}, Section 3.2 for further discussion. 

Most of the existing fairness literature is concerned with observable fairness criteria and accuracy measures. To our knowledge, only two papers have developed methods to satisfy the type of counterfactual criteria described above. \cite{mishler_fairness_2021} developed a post-processing method that maximizes accuracy while satisfying (approximate) counterfactual equalized odds or related fairness criteria. Their approach takes as input a binary classifier and outputs a randomized binary classifier. In contrast, our method applies to both classification and regression, and it combines in-processing and post-processing.

\cite{coston_characterizing_2021} developed a method to minimize various unfairness measures subject to an accuracy constraint. They considered both the observable setting and a \emph{selective labels} setting, when the outcome of interest is observed only for a non-representative subset of the population. Although the terminology differs, this is essentially equivalent to the counterfactual setting that we consider. Their method involves iterative optimization and outputs a randomized predictor that is constructed as a distribution over a set of models.

In contrast to both of the above methods, our framework can handle any of the following: (1) minimizing risk subject to fairness constraints, (2) minimizing unfairness subject to an accuracy constraint, and (3) efficiently producing a large set of models that vary in their risk and fairness properties. Method (3) utilizes a set of closed-form solutions that are extremely fast to compute. Our methods also output deterministic rather than randomized predictors, though in a classification setting these can be turned into randomized classifiers by treating the output in $[0, 1]$ as a probability and sampling from the corresponding Bernoulli distribution, as described in Section \ref{subsec:model_validation}. Our methods apply to a large class of observable and counterfactual accuracy and fairness criteria. 

\subsection{Fairness-accuracy and fairness-fairness tradeoffs}
Within a candidate set of models, the most accurate model and the most fair model may not be the same model, in which case there is a fairness-accuracy tradeoff. The shape of this tradeoff depends on the model set, the accuracy and fairness criteria, and the distribution of the data \citep{dutta_is_2020}. While some papers emphasize the unavoidable existence of such tradeoffs \citep{calders_building_2009, Corbett-Davies2017, menon_cost_2018, woodworth_learning_2017, zhao_inherent_2019}, other papers have found that in practical settings they are sometimes so small as to be irrelevant; that is, relative to a baseline model, it may be possible to substantially improve a given fairness criterion with little to no decrement in accuracy, or vice versa \citep{coston_characterizing_2021, rodolfa_empirical_2021}. 

Different fairness criteria may also trade off with one another. In their strictest form, many fairness criteria are mutually unsatisfiable in real-world conditions \citep{Chouldechova2017, kleinberg2017}. In practice, many methods make use of continuous-valued relaxations of these criteria, which may be more or less simultaneously satisfiable, to a degree that again depends on the modeling choices and data distribution.

Recent work aims to characterize fairness-accuracy and fairness-fairness tradeoffs both theoretically \citep{dutta_is_2020, kim_fact_2020} and empirically \citep{berk_convex_2017, liu_accuracy_2021}. Like \cite{berk_convex_2017}, our penalized predictor method uses fairness regularization terms to trace out different paths in fairness-accuracy space; however, their results consider observable accuracy and fairness measures, whereas ours encompass both observable and counterfactual measures. We also consider a (broad) class of fairness criteria that yield closed-form solutions, and we provide theoretical guarantees for our methods.

In some cases, users have clear accuracy or fairness constraints that they wish their models to satisfy. These constraints might derive from moral, legal, or business considerations. For example, a business might wish to ensure that a hiring algorithm generates positive recommendations for roughly equal percentages of male and female applicants, in order to avoid potential disparate impact. Conversely, a business might wish to improve the fairness of an existing model without sacrificing accuracy (profit). We provide an explicit correspondence between our constrained and penalized predictors and show how the set of penalized models can be ``seeded'' with models that target specific fairness or accuracy constraints.

Our method also makes it easy for users and auditors to understand whether a model in use could be made more fair without a substantial loss of accuracy, or vice versa. This is useful both for improving model performance and for understanding whether a particular level of unfairness can be justified as a type of ``business necessity,'' or whether fairness can be improved without compromising accuracy \citep{coston_characterizing_2021}.

\section{Setup and estimands} \label{sec:setup_and_estimands}
Our data is of the form $Z = (A, X, S, D, Y) \sim \Pb$, for sensitive feature $A \in \{0, 1\}$, additional covariates $X \in \calX$, previously trained predictor(s) $S \in \calS$, decision or treatment $D \in \calD$, and outcome or label $Y \in [\ell_y, u_y]$ with bounds $\ell_y, u_y$. If no previously trained predictors are available, then we have $S = \emptyset$. We denote by $Y_i^0$ the potential outcome $Y_i^{D=0}$, that is, the outcome or label that would be observed for individual $i$ if, possibly contrary to fact, the decision were set to $D_i = 0$. For example, $Y^0$ could indicate whether an individual would recidivate if released pretrial. We assume that $Y^0$ also lies in $[\ell_y, u_y]$. In settings where we are interested in observational rather than counterfactual fairness, we may have $D = \emptyset$. We assume that $Z \subseteq \calZ\subset  \Rb^p$ for compact $\calZ$.

Let $W = (A, X, S) \in \calW$ represent the collected covariates. We let $\Yt$ denote either $Y$ and $Y^0$ as appropriate, since we are interested in both observational and counterfactual fairness and accuracy measures. We refer to $\Yt = Y$ as the \emph{observable} setting and $\Yt = Y^0$ as the \emph{counterfactual setting}. Broadly speaking, we seek functions of the form $f:\calW \mapsto [\ell_y, u_y]$ that are both accurate and fair in predicting $\Yt$. Our goals are (1) to enable users to target specific fairness or accuracy constraints, and (2) to trace out the fairness and accuracy properties of a large set of models, both in order to understand setting-specific fairness-accuracy and fairness-fairness tradeoffs and in order to maximize the user's ability to choose a desirable model.

\begin{remark}[Additional notation]
We let $\Vert \cdot \Vert$ denote an appropriate $L_2$ norm. That is, for any random variable $f(Z)$ taking values in $ \Rb$, $\Vert f(Z) \Vert = (\int (f(Z))^2 d\Pb(Z))^{1/2}$ denotes the $L_2(\Pb)$ norm, while for a non-random vector $v \in  \Rb^k$, $\Vert v \Vert = (\sum_{j=1}^k v_j^2)^{1/2}$ denotes the Euclidean $L_2$ norm. For a random vector $f(Z)$ taking values in $ \Rb^k$, $\Vert f(Z) \Vert = (\sum_{j=1}^k \Vert f_j(Z) \Vert^2)^{1/2}$.
\end{remark}

\subsection{FADE summary}
The ``Ensemble learning'' part of ``FAir Double Ensemble learning'' has its usual sense, referring to an ensemble of predictors. The ``Double'' part captures several features of our approach: (1) it combines in-processing and post-processing; (2) it accommodates both observable and counterfactual outcomes; (3) it (optionally) has two stages, first learning predictors and then learning their ensemble weights; and (4) in the counterfactual setting, it utilizes doubly robust estimators, which also appear in the literature under the heading ``double machine learning'' \citep{tsiatis_semiparametric_2006, chernozhukov_doubledebiased_2018}. These features are illustrated in the remainder of this section and in Section \ref{sec:estimation}.

\subsection{Accuracy and fairness measures}
The risk (accuracy) measure we consider is the MSE:
\begin{align*}
    \risk(f) = \E[(f(W) - \Yt)^2]
\end{align*}
We consider (un)fairness measures $\UF(f)$ that can be expressed in the form
\begin{align}
    \UF(f) &= \lvert\E[g(W, \Yt)f(\W)]\rvert \label{eq:unfairness}
\end{align}
where $g(W, \Yt)$ is a bounded \emph{fairness function} that depends only on $W$ and $\Yt$, not on $D$. This accommodates a broad range of measures, including measures described by the following proposition. All proofs are given in the Appendix.
\begin{proposition} \label{proposition:fairness_functions}
Let $\alpha_0, \alpha_1 \in  \Rb$ and let $h_0, h_1$ be mappings from $\{0, 1\} \times \Yt$ to $\{0, 1\}$. Consider an unfairness measure
\begin{align*}
    \UF(f) = \lvert\alpha_0\E[f(W) \mid h_0(A, \Yt) = 1] - \alpha_1\E[f(W) \mid h_1(A, \Yt) = 1] \rvert
\end{align*}
and assume that $\Pb(h_0(A, \Yt) = 1) > 0, \Pb(h_1(A, \Yt) = 1) > 0$. Then there exists a fairness function $g(W, \Yt)$ such that $\UF(f) = \lvert\E[g(W, \Yt)f(W)] \rvert$, namely
\begin{align*}
    g(W, \Yt) &= \alpha_0\frac{h_0(A, \Yt)}{\E[h_0(A, \Yt)]} - \alpha_1\frac{h_1(A, \Yt)}{\E[h_1(A, \Yt)]}
\end{align*}
\end{proposition}
That is, \eqref{eq:unfairness} is compatible with any fairness measure that can be expressed as a (weighted) difference of average predictions conditioned on events that are a function of the sensitive feature and the outcome. Functions of this form must in general be estimated, since they depend on unknown expected values. We focus in this paper on the following measures, which we refer to equivalently as disparities. We first express each in a canonical form, and then we identify the corresponding fairness function $g(W, \Yt)$, dropping the arguments $(W, \Yt)$ for convenience.

\begin{definition} \label{definition:rate_diff} The \emph{rate disparity (rate-diff)} is
\begin{align*}
     \lvert\E[f(W) \mid A = 0] - \E[f(W) \mid A = 1] \rvert
\end{align*}
with fairness function
\begin{align*}
    g^\text{rate} &= \frac{1 - A}{\E[1 - A]} - \frac{A}{\E[A]}
\end{align*}
\end{definition}

\begin{definition} \label{definition:FPR-diff} For $\Yt \in \{0, 1\}$, the \emph{generalized False Positive Rate disparity (FPR-diff)} is
\begin{align*}
     \lvert\E[f(W) \mid A = 0, \Yt = 0] - \E[f(W) \mid A = 1, \Yt = 0] \rvert
\end{align*}
with fairness function 
\begin{align*}
    g^\text{FPR} &= \frac{(1 - \Yt)(1 - A)}{\E[(1 - \Yt)(1 - A)]} - \frac{(1 - \Yt)A}{\E[(1 - \Yt)A]}
\end{align*}
\end{definition}

\begin{definition} \label{definition:FNR-diff} For $\Yt \in \{0, 1\}$, the \emph{generalized False Negative Rate disparity (FNR-diff)} is
\begin{align*}
     \lvert\E[1 - f(W) \mid A = 0, \Yt = 1] - \E[1 - f(W) \mid A = 1, \Yt = 1] \rvert
\end{align*}
with fairness function
\begin{align*}
    g^\text{FNR} &= \frac{\Yt A}{\E[\Yt A]} - \frac{(1 - \Yt)(1 - A)}{\E[\Yt(1 - A)]}
\end{align*}
\end{definition}
These definitions are closely related to common fairness criteria described in Section \ref{subsec:fairness_criteria}. The criterion of \emph{independence} requires predictions $f(W)$ to be independent of the sensitive feature $A$. Rate-diff measures violations of this criterion \citep{calders_three_2010}. Equal opportunity requires the false negative rates to be equal across the two groups, while equalized odds requires both the false positive and the false negative rates to be equal \citep{hardt_equality_2016}. FNR-diff therefore measures violations of equal opportunity, while FPR-diff and FNR-diff together measure violations of equalized odds.

For continuous-valued predictors, it may be challenging to attain full (conditional) independence. Hence it is common to focus only on average conditional predictions \citep[e.g.][]{Corbett-Davies2017}.

\subsection{Predictor classes}
We consider predictors that lie in the linear span of a set of basis functions $b = b(\W) = (b_1(\W), \ldots b_k(\W))$, where each function $b_j(W)$ maps from $\calW$ to $ \Rb$. That is, for a given $b$ we seek predictors in the set $\calF_b$, where
\begin{align*}
    \calF_b = \{b^T\beta: \beta \in  \Rb^k\}
\end{align*}
Predictors of this form are commonly referred to as (a linear) ensemble, stacked predictors, or aggregated predictors  \citep{breiman_stacked_1996, juditsky_functional_2000, goos_optimal_2003, polley_super_2010}. In the context of our method, we refer to these as FADE predictors. The vector $b$ is determined by the user. It can include for example previously trained predictors $S$, newly trained predictors, or arbitrary orthogonal basis functions such as trigonometric functions or polynomials. Our approach makes it easy for users to search across a range of different bases $b$.

In this paper, we consider a regime in which $k < n$, where $n$ is the sample size, since this simplifies estimation. In practice, users might wish to use bases of dimension $k \geq n$, such as spline bases or kernel basis functions. We briefly consider these and other possibilities in Appendix \ref{appendix:growing_bases}. In our asymptotic analyses, we also generally assume that the basis is eventually fixed, meaning $k \not\rightarrow \infty$. We intend to analyze settings in which $k \geq n$ and/or $k \rightarrow \infty$ in future work.  

Depending on $b$, the set $\calF_b$ may be relatively rich. For example, $b$ could be a truncated orthonormal basis of the space $L_2(\calW)$ of square-integrable functions, in which case $\calF_b$ could approximate $L_2$ to a degree chosen by the user.

Our theoretical analyses utilize the following two assumptions about the basis.

\begin{assumption}[PSD outer product] \label{assumption:psd}
Uniformly in $n$, the eigenvalues of $\E[bb^T]$ are bounded above and away from 0.
\end{assumption}
This assumption asserts that the basis functions $b_1(W), \ldots b_k(W)$ are not too collinear. It means that $\E[bb^T]$ is always positive semi-definite. In a regime in which the basis is eventually fixed, this assumption simply requires that the basis functions are never perfectly collinear. 

\begin{assumption}[Bounded basis norm] \label{assumption:bounded_basis}
Uniformly in $n$, $\sup_{w\in\calW} \Vert b(w) \Vert < \infty$, where $\Vert b(w) \Vert$ is the Euclidean $L_2$ norm.
\end{assumption}
When $k \not\rightarrow \infty$, this assumption simply requires the norm $\Vert b(w) \Vert$ to be finite over the covariate space. In a setting in which $k$ is allowed to grow to infinity, this assumption can be relaxed to one which controls the growth rate of $\sup_{w\in\calW} \Vert b(w) \Vert$. See Remark \ref{remark:k_to_infinity} in Section \ref{sec:estimation}.

\subsection{Estimands} \label{subsec:estimands}
We first define two estimands that are solutions to constrained least squares problems. These estimands represent users who have clear target fairness or accuracy constraints. We then show that these estimands can be equivalently expressed via penalized least squares problems that admit closed-form solutions. These solutions are indexed by an unfairness penalty parameter; by varying this parameter, we may trace out curves in accuracy-fairness space over $\calF_b$.

Suppose that there are $t$ fairness measures that can be expressed via fairness functions $g_j, j = 1, \ldots, t$. For a given $k$-dimensional basis $b$, define the risk-minimization (\emph{risk-min}) parameter $\betaopt_r$ and the unfairness-minimization (\emph{unfair-min}) parameter $\betaopt_u$ as the solutions to the following optimization problems:
\begin{alignat*}{2}
    & \betaopt_r = && \argmin_{\beta\in \Rb^k}\E[(b^T\beta - \Yt)^2] \\
    & && \text{subject to } \-\ (\E[g_j b^T\beta])^2 \leq \epsilon_j^2, \quad j = 1, \ldots t \\
    & \betaopt_u = && \argmin_{\beta\in \Rb^k} \sum_{j=1}^t \alpha_j(\E[g_j b^T\beta])^2 \\
    & && \text{subject to } \E[(b^T\beta - \Yt)^2] \leq \epsilon
\end{alignat*}
for user-chosen constraints $\epsilon_j \geq 0$, $\epsilon > 0$, and weights $\alpha_j$. That is, $\betaopt_r$ indexes the most accurate predictor in $\calF_b$ among those that satisfy $t$ specified fairness constraints, and $\betaopt_u$ indexes the most fair predictor among those that satisfy a specified risk constraint.

We constrain $\epsilon > 0$ because otherwise we'd be insisting on a perfectly accurate predictor, which is generally impossible in practice. The risk-min problem is always feasible with $\epsilon_j \geq 0$, since the predictor defined by $\beta = 0$ always satisfies the fairness constraints. Under Assumption \ref{assumption:psd}, $\betaopt_r$ is unique, since the objective is strictly convex. The unfair-min problem may be infeasible, if there is no predictor in $\calF_b$ whose risk is less than or equal to $\epsilon$. This may not be an issue in practice, if $\epsilon$ represents (an estimate of) the risk of an existing benchmark model. With slight modifications, all our subsequent results would carry through if this constraint were explicitly expressed with respect to a benchmark model; for the sake of simplicity, however, we leave it in this form. If the unfairness-minimization problem is feasible and $\sum_{j=1}^t \alpha_j\E[g_j b]\E[g_j b]^T$ is positive definite, then $\betaopt_u$ is unique, since the objective is strictly convex.

Note that the fairness constraints in the risk-min problem can be equivalently written in affine form, as $\lvert\E[g_j f(\W)]\rvert \leq \epsilon_j$, meaning that $\betaopt_r$ is the solution to a quadratic program. We express the constraints in squared form for notational consistency with the penalized estimand, which is defined as follows.

For any $\lambda = (\lambda_1, \ldots, \lambda_t)$, with all $\lambda_j \geq 0, j = 1, \ldots, t$, define the penalized-minimization (\emph{penalized-min}) estimand $\betaopt_\lambda$ as:
\begin{align*}
    & \betaopt_{\lambda} = \argmin_{\beta\in \Rb^k} \E[(b^T\beta - \Yt)^2] + \sum_{j=1}^t \lambda_j \left(\E[g_j b^T\beta]\right)^2
\end{align*}
This can be written in an equivalent closed form:
\begin{align*}
     & \betaopt_{\lambda} = \left(\E[bb^T] + \sum_{j=1}^t \lambda_j \E[g_j b]\E[g_j b]^T\right)^{-1}\E[\Yt b]
\end{align*}
Under Assumption \ref{assumption:psd}, the matrix inverse always exists, since each matrix $\E[g_j b]\E[g_j b]^T$ is positive semi-definite and $\E[bb^T]$ is positive definite, and hence by Weyl's inequality the entire matrix is positive definite.

We now establish a correspondence between the constrained and penalized forms. Let $\mathcal{I}$ denote the set of active fairness constraints at $\betaopt_r$, that is, $\mathcal{I} = \{j \in \{1, \ldots, t\}: (\E[g_j b^T\betaopt_r])^2 = \epsilon_j^2\}$. 
\begin{assumption}[LICQ] \label{assumption:licq}
The set of vectors $\{\E[g_jb]\E[g_jb]^T\betaopt_r: j \in \mathcal{I}\}$ is linearly independent.
\end{assumption}
Assumption \ref{assumption:licq}, which is expected to hold in practical settings, is the Linear Independence Constraint Qualification (LICQ), which yields an injective mapping from the constrained to the penalized form.
\begin{proposition} \label{proposition:constrained_to_lagrange}
Under Assumption \ref{assumption:psd}, for any $\betaopt_r$ there exists a $\lambda \in \Rpos^t$ such that $\betaopt_\lambda = \betaopt_r$. If Assumption \ref{assumption:licq} holds, then this $\lambda$ is unique. 
\end{proposition}
We will utilize the penalized form to efficiently construct a large set of predictors that vary in their accuracy and fairness properties. We will see that we can exploit an empirical analogue of Proposition \ref{proposition:constrained_to_lagrange} to ``seed'' this set with models that target specified fairness constraints.

\begin{proposition} \label{proposition:lagrange_to_constrained}
Fix $\lambda \in \Rpos^t$, $\lambda \neq 0$. Under Assumption \ref{assumption:psd}, $\betaopt_\lambda = \betaopt_r$ with fairness constraints $\epsilon_j^2 = (\E[g_j b^T\betaopt_\lambda])^2$.
\end{proposition}
Proposition \ref{proposition:lagrange_to_constrained} expresses the converse direction of the relationship between $\betaopt_r$ and $\betaopt_\lambda$. The constraints $\epsilon_j$ that define $\betaopt_r$ are arguably easier to reason about than the penalties $\lambda_j$ that define $\betaopt_\lambda$. This proposition therefore facilitates interpretation of penalized estimands in terms of their corresponding constrained forms.

An analogous penalized form can be written that corresponds to $\betaopt_u$, with results that match Propositions \ref{proposition:constrained_to_lagrange} and \ref{proposition:lagrange_to_constrained}. For estimation purposes, however, we focus in this paper on $\betaopt_\lambda$, so we do not develop that form here.
\begin{remark}[Predictor truncation]
Any $\beta \in  \Rb^k$ indexes a predictor $b^T\beta \in \calF_b$. Since $Y^0$ and $Y$ are bounded in $[\ell_u, u_y]$, however, the resulting predictor may be truncated to lie in $[\ell_u, u_y]$, if the bounds are known.
\end{remark}
Our final estimands consist of the risk and (un)fairness properties of any fixed predictor $f_\beta$:
\begin{align*}
    \risk(f_\beta) &= \E[(f_\beta - \Yt)^2] \\
    \UF_j(f_\beta) &= \E[g_j f_\beta], \quad j = 1, \ldots t
\end{align*}
In particular, once we have computed some estimate $\betahat$ of $\betaopt_r$ or $\betaopt_u$, or $\betaopt_\lambda$, it is of interest to estimate the risk and fairness of the resulting predictor $f_\betahat$.

\section{Identification} \label{sec:identification}
When $\Yt = Y^0$, i.e. when the risk and fairness functions are defined with respect to counterfactual rather than observable outcomes, we require assumptions in order to identify these quantities in terms of the observed data. For ease of notation, we first define three nuisance parameters that appear in the estimands and associated estimators.
\begin{align*}
    \pi = \pi(W) &= \Pb(D = 1 \mid W) \\
    \mu_0 = \mu_0(W) &= \E[Y \mid W, D=0] \\
    \nu_0 = \nu_0(W) &= \E[Y^2 \mid W, D=0]
\end{align*}
$\pi(W)$ is the propensity score, while $\mu_0$ and $\nu_0$ are regressions with respect to the observed outcome and the squared observed outcome. In a classification setting with $Y \in \{0, 1\}$, we have $Y^2 = Y$, so $\nu_0 = \mu_0$. We make the following common causal inference assumptions:

\begin{assumption}[Consistency] \label{assumption:consistency}
$Y = DY^1 + (1-D)Y^0$.
\end{assumption}
\begin{assumption}[Positivity] \label{assumption:positivity}
$\exists \delta \in (0, 1) \text{ s.t. } \Pb(\pi(W) \leq 1 - \delta) = 1$.
\end{assumption}
\begin{assumption}[Ignorability] \label{assumption:ignorability}
$Y^0 \ind D \mid W$.
\end{assumption}

Consistency signifies that for each individual, the treatment received matches the outcome that is observed, meaning for example that one person's treatment status does not affect other people's outcomes. Positivity requires that within covariate strata $W$, individuals have some chance of not receiving treatment, meaning that there is no stratum of measure $> 0$ in which all individuals are guaranteed to receive treatment. Finally, ignorability precludes unmeasured confounders that affect both treatment status and the potential outcome. Positivity and ignorability may be satisfied in randomized experiments in which treatment is assigned (conditionally) at random, or in observational studies given an appropriate set of covariates $W$.

Note that these assumptions may not hold exactly in practice. For example, in an observational study, the measured covariates may not be sufficient to fully deconfound $D$ and $Y^0$. The enterprise of \emph{sensitivity analysis} in causal inference parameterizes violations of these assumptions and models their effect on downstream estimation \citep{rosenbaum_sensitivity_1987, liu_introduction_2013, richardson_nonparametric_2014, bonvini_sensitivity_2021}. We leave sensitivity analysis in our setting for future work.

For convenience, we also define the following:
\begin{align*}
    \phi &= \phi(Z) = \frac{1-D}{1-\pi}(Y - \mu_0) + \mu_0 \\
    \phibar &= \phibar(Z) = \frac{1-D}{1-\pi}(Y^2 - \nu_0) + \nu_0
\end{align*}
Under the identifying assumptions, these are the uncentered influence functions for $\E[Y^0]$ and $\E[(Y^0)^2]$, respectively \citep{bickel1993, van_der_laan_unified_2003, tsiatis_semiparametric_2006, kennedy_semiparametric_2016}.

\begin{proposition} \label{proposition:identification}
Under Assumptions \ref{assumption:positivity}--\ref{assumption:ignorability}, the counterfactual risk, FPR-diff, and FNR-diff for any function $f: \calW \mapsto  \Rb$ are identified as follows.
\begin{align}
    \E[(f - Y^0)^2] &= \E[(f - \mu_0)^2] + \var(Y^0) \label{eq:risk1} \\
                    &= \E[f^2 - 2f\mu_0 + \nu_0] \label{eq:risk2} \\
    \E[g^\text{cFPR}f(W)] &= \E\left[\left\{ \frac{(1 - \mu_0)(1 - A)}{\E[(1 - \mu_0)(1 - A)]} - \frac{(1 - \mu_0)A}{\E[(1 - \mu_0)A]}\right\}f(W)\right] \nonumber \\
    \E[g^\text{cFNR}f(W)] &= \E\left[\left\{ \frac{\mu_0A}{\E[\mu_0A]} - \frac{(1 - \mu_0(1 - A)}{\E[\mu_0(1 - A)]} \right\}f(W)\right] \nonumber
\end{align}
These expressions also hold if $\mu_0$ is replaced with $\phi$ and $\nu_0$ is replaced with $\phibar$.
\end{proposition}
We do not include the rate-diff, since this involves only the sensitive feature and the decision variable, not outcomes, and is therefore trivially identified.

\begin{remark}[Multiple risk expressions] \label{remark:nu0}
Expressions \eqref{eq:risk1} and \eqref{eq:risk2} show that when estimating the risk-min parameter $\betaopt_r$, we can either minimize an estimate of $\E[(f - \mu_0)^2]$ or an estimate of $\E[f^2 - 2f\mu_0]$; the terms $\var(Y^0)$ and $\nu_0$ are constant with respect to $f$ and so drop out of the minimization. The nuisance parameter $\nu_0$ will only be required when we wish to estimate the actual risk of a given predictor, as well as when estimating the unfair-min parameter $\betahat_u$, since that involves a constraint on the actual risk. Note that $\nu_0$ would not be required to solve unfair-min if the accuracy constraint were defined with respect to an existing benchmark model, since the two $\nu_0$ terms in the constraint would cancel out.
\end{remark}

\section{Estimation} \label{sec:estimation}
We require a training set $\datatrain$, which is used to construct estimates $\betahat$ of the optimal weights $\betaopt_r$ or $\betaopt_u$ or $\betaopt_\lambda$, and a test set $\datatest$, which is used to estimate the risk and fairness values of the resulting predictor(s) $f_\betahat$. If the user wishes to train new basis predictors, then an additional dataset $\datalearn$ is also required. This is not needed if the user is only aggregating arbitrary basis functions, like trigonometric functions, or previously existing predictors $S$.

In order to obtain fast rates for our estimators, in the counterfactual setting we split $\datatrain$ and $\datatest$ into separate folds for estimating the nuisance parameters and the target parameters. The sample splitting scheme is shown in Figure \ref{f:sample_splitting}. For simplicity, we illustrate a single split, but in practice cross-fitting can be used within each dataset.

\begin{figure}[ht]
\centering
{\renewcommand{\arraystretch}{1.75}
\begin{tabularx}{0.65\textwidth}{|Y|Y|}
    \multicolumn{2}{c}{\text{\large $\datalearn$}} \\
    \noalign{\vspace{2pt}}
    \hline 
    \multicolumn{2}{|c|}{Learn basis predictors $b_j(W)$ for $j \subseteq \{1, \ldots, k\}$\vphantom{$\datatrain^{\text{target}}$}} \\
    \hline
    \noalign{\vspace{10pt}}
     \multicolumn{2}{c}{\text{\large $\datatrain$}} \\
    \noalign{\vspace{2pt}}
     \hline
    $\datatrain^{\text{nuis}}$ & $\datatrain^{\text{target}}$ \\
    \hline
    \noalign{\vspace{-12pt}}
     \multicolumn{1}{@{}c@{}}{$\underbracket[1pt]{\hspace*{\dimexpr12\tabcolsep+2\arrayrulewidth}\hphantom{012}}_{\substack{\\\mathlarger \pihat, \-\ \mathlarger \muhat_0, \-\ \mathlarger \nuhat_0}}$} &
    \multicolumn{1}{@{}c@{}}{$\underbracket[1pt]{\hspace*{\dimexpr12\tabcolsep+2\arrayrulewidth}\hphantom{012}}_{\substack{\\ \mathlarger \betahat}}$} \\    
    \multicolumn{2}{c}{\text{\large $\datatest$}} \\
    \noalign{\vspace{1.5pt}}
     \hline
    $\datatest^{\text{nuis}}$ & $\datatest^{\text{target}}$ \\
    \hline
    \noalign{\vspace{-12pt}}
    \multicolumn{1}{@{}c@{}}{$\underbracket[1pt]{\hspace*{\dimexpr12\tabcolsep+2\arrayrulewidth}\hphantom{012}}_{\substack{\\\mathlarger \pihat, \-\ \mathlarger \muhat_0, \-\ \mathlarger \nuhat_0}}$} &  
    \multicolumn{1}{@{}c@{}}{$\underbracket[1pt]{\hspace*{\dimexpr12\tabcolsep+2\arrayrulewidth}\hphantom{012}}_{\substack{\\ \text{\small{Risk \& fairness of }} \mathlarger f_\betahat}}$}
    \end{tabularx}}
    \caption{Sample splitting scheme. $\datalearn$ is not needed if all the basis functions already exist. $\datatrain^{\text{nuis}}$ and $\datatest^{\text{nuis}}$ are used to estimate nuisance parameters, while $\datatrain^{\text{target}}$ and $\datatest^{\text{target}}$ are used to estimate target parameters. $\betahat$ represents a weight vector that is an estimate of $\betaopt_r$ or $\betaopt_u$ or $\betaopt_\lambda$, while $f_\betahat$ represents the predictor indexed by $\betahat$. Splitting $\datatrain$ and $\datatest$ is only required in the counterfactual setting, since there are no nuisance parameters in the observable setting. In practice, cross-fitting may be used within both $\datatrain$ and $\datatest$.}
    \label{f:sample_splitting}
\end{figure}

We solve empirical versions of the identified minimization problems that define the estimands. Let $\phihat, \phibarhat$ denote estimates of $\phi$ and $\phibar$ constructed from estimates $\pihat$, $\muhat_0$, $\nuhat_0$.

For any fixed function $f: \calZ \mapsto  \Rb$, let $\Pn(f(Z)) = n^{-1}\sum_{i=1}^n f(Z)$ and $\Pb(f) = \int f d\Pb(Z)$ denote the sample and population expectations of $f$, so that for example $\Pb(\phi) = \E[\phi]$ while $\Pb(\phihat) = \E[\phihat \mid \datatrain]$ or $\E[\phihat \mid \datatest]$ is the expected value of $\phihat(Z)$ once the relevant nuisance function estimate $\phihat$ has been constructed. That is, $\Pb(\phihat)$ is a random variable that depends on the nuisance data, while $\Pn(\phihat)$ is a random variable that depends on both the nuisance and target data. In order to avoid excess notation, for quantities like $\Pb(\phihat)$ and $\Pn(\phihat)$, we will rely on context to make it clear whether $\phihat$ depends on $\datatrain$ or $\datatest$.

For notational convenience, let $\ghat_j$ with no arguments denote $g_j(W, Y)$ in the observable setting and $g_j(W, \phihat)$ in the counterfactual setting. That is $\ghat_j = g_j$ in the observable setting, since there is no nuisance quantity to estimate, but $\ghat_j \neq g_j$ in the counterfactual setting. The occasional use of $\ghat_j$ for both settings allows us to concisely state certain conditions and results.

\begin{assumption}[Bounded propensity estimator] \label{assumption:bounded_propensity_estimator}
$\exists \gamma \in (0, 1) \text{ s.t. }$ $\Pb(\pihat(A, X, \Rin) \leq 1 - \gamma) = 1$.
\end{assumption}
Assumption \ref{assumption:bounded_propensity_estimator} is the empirical analogue of the positivity assumption (\ref{assumption:positivity}). It can be trivially satisfied by truncating $\pihat$ at $1 - \delta$, the positivity threshold in Assumption \ref{assumption:positivity}.

\begin{assumption}[Consistent nuisance estimators] \label{assumption:consistent_nuisance_estimators}
$\Vert \pihat - \pi \Vert = o_\Pb(1)$ and $\Vert \muhat_0 - \mu_0 \Vert = o_\Pb(1)$ and $\Vert \nuhat_0 - \nu_0 \Vert = o_\Pb(1)$.
\end{assumption}
This assumption is reasonable if nonparametric methods are used to construct the nuisance parameter estimates. With slight procedural modifications, this assumption can be relaxed to require consistency in the influence function estimators $\phihat$ and $\phibarhat$ without necessarily requiring consistency in each of the nuisance parameter estimators. For simplicity, we do not address this.

\subsection{Constrained FADE estimators}
The risk-min and unfair-min estimators $\betahat_r$ and $\betahat_u$ are defined for the observable and counterfactual settings in Tables \ref{t:constrained_risk_min_estimator} and \ref{t:constrained_unfair_min_estimator}.
{\renewcommand{\arraystretch}{1.8}
\begin{table}[ht]
    \centering
    \begin{tabularx}{\linewidth}{|Y|Y|}
        \hline
        Observable ($\Yt = Y$) & Counterfactual ($\Yt = Y^0$) \\
         \hline
        {\begin{align*}
         \betahat_r = & \argmin_{\beta\in \Rb^k}\Pn[(b^T\beta - Y)^2] \\
        \text{s.t. } & \left(\Pn[g_j(W, Y) b^T\beta]\right)^2 \leq \epsilon_j^2, \-\ j = 1, \ldots t
        \end{align*}} &
         {\begin{align*}
        \betahat_r = & \argmin_{\beta\in \Rb^k}\Pn[(b^T\beta - \phihat)^2] \\
        \text{s.t. } & \left(\Pn[g_j(W, \phihat) b^T\beta]\right)^2 \leq \epsilon_j^2, \-\ j = 1, \ldots t
         \end{align*}} \\
        \hline
    \end{tabularx}
    \caption{Definition of the unfair-min estimator $\betahat_r$ in the observable and counterfactual settings.}
    \label{t:constrained_risk_min_estimator}
\end{table}}

{\renewcommand{\arraystretch}{1.8}
\begin{table}[ht]
    \centering
    \begin{tabularx}{\linewidth}{|Y|Y|}
        \hline
         Observable ($\Yt = Y$) & Counterfactual ($\Yt = Y^0$) \\
         \hline
         {\begin{align*}
        \betahat_u = & \argmin_{\beta\in \Rb^k} \sum_{j=1}^t \alpha_j\left(\Pn[g_j(W, Y) b^T\beta]\right)^2 \\
        \text{s.t. } & \Pn\left[(b^T\beta - Y)^2\right] \leq \epsilon^2
        \end{align*}} & 
        {\begin{align*}
        \betahat_u = & \argmin_{\beta\in \Rb^k} \sum_{j=1}^t \alpha_j\left(\Pn[g_j(W, \phihat) b^T\beta]\right)^2 \\
        \text{s.t. } & \Pn\left[(b^T\beta)^2 - 2(b^T\beta)\phihat + \phibarhat\right] \leq \epsilon^2
        \end{align*}} \\
        \hline
    \end{tabularx}
    \caption{Definition of the risk-min estimator $\betahat_u$ in the observable and counterfactual settings.}
    \label{t:constrained_unfair_min_estimator}
\end{table}}

As with the corresponding estimands $\betaopt_r$ and $\betaopt_u$, the optimization problem that defines $\betahat_r$ is always feasible, while the problem that defines $\betahat_u$ may not be, if there is no predictor in $\calF_b$ with estimated risk less than or equal to $\epsilon$. If $\Pn[bb^T]$ is positive definite, then $\betahat_r$ is unique, since the objective is strictly convex. Under Assumption \ref{assumption:psd}, this will hold with probability approaching 1 in $n$, or with probability 1 if, say, at least one of the covariates in $W$ is continuously distributed. If the problem that defines $\betahat_u$ is feasible and $\sum_{j=1}^t \alpha_j\Pn[\ghat_j b]\Pn[\ghat_j b]^T$ is positive definite, then $\betahat_u$ is unique, since the objective is strictly convex.

We next consider the excess risk and the excess unfairness for the constrained predictors. In the counterfactual setting, we require assumptions on the rate at which the nuisance parameters are estimated.
\begin{assumption}[Nuisance parameter rates] \label{assumption:nuisance_rates}
    \begin{align*}
        \Vert \pihat - \pi \Vert \Vert \muhat_0 - \mu_0 \Vert &= o_\Pb(1/\sqrt{n}) \\
        \Vert \pihat - \pi \Vert \Vert \nuhat_0 - \nu_0 \Vert &= o_\Pb(1/\sqrt{n})
    \end{align*}
\end{assumption}
Assumption \ref{assumption:nuisance_rates} says that the product of errors in the nuisance parameter estimators goes to 0 faster than $\sqrt{n}$. That can be satisfied for example if the nuisance parameters are estimated at faster than $n^{1/4}$ rates, which can be achieved in nonparametric settings under appropriate smoothness or sparsity conditions \citep{gyorfi_distribution-free_2002, raskutti2011minimax}. 
\begin{definition}The \emph{excess risk} is defined for $\betahat_r$ and $\betahat_u$ as:
\begin{align*}
    & \Pb[(b^T\betahat_r - \Yt)^2] - \Pb[(b^T\betaopt_r - \Yt)^2] \tag{risk-min} \\
    & \Pb[(b^T\betahat_u - \Yt)^2] - \epsilon^2                \tag{unfair-min}
\end{align*}
\end{definition} 
\begin{theorem}[Excess risk in the constrained setting]
\label{LS_thm:excess_risk_constrained} Under Assumptions \ref{assumption:psd}--\ref{assumption:bounded_basis} for the observable setting, and Assumptions \ref{assumption:psd}--\ref{assumption:bounded_basis} and \ref{assumption:consistency}--\ref{assumption:nuisance_rates} for the counterfactual setting:
\begin{align*}
    \Pb[(b^T\betahat_r - \Yt)^2] - \Pb[(b^T\betaopt_r - \Yt)^2] &= O_\Pb(1/\sqrt{n}) \tag{risk-min} \\
    \Pb[(b^T\betahat_u - \Yt)^2] - \epsilon^2 &= O_\Pb(1/\sqrt{n})\tag{unfair-min}
\end{align*}
\end{theorem}

\begin{definition}
The \emph{excess unfairness} for $\betahat_r$ and $\betahat_u$ is defined as:
\begin{align*}
    & \max_{j = 1, \ldots, t}\left\{\left((\Pb[g_j b^T\betahat_r])^2  - \epsilon_j^2\right)_+\right\} \tag{risk-min} \\
    & \sum_{j=1}^t \alpha_j \left\{(\Pb[g_j b^T\betahat_u])^2 - (\Pb[g_j b^T\betaopt_u])^2\right\} \tag{unfair-min}
\end{align*}
where $(\cdot)_+ = \max\{\cdot, 0\}$ denotes the positive part function. 
\end{definition}

\begin{theorem}[Excess unfairness in the constrained setting] \label{LS_thm:excess_unfairness_constrained}
Under Assumptions \ref{assumption:psd}--\ref{assumption:bounded_basis} for the observable setting, and Assumptions \ref{assumption:psd}--\ref{assumption:bounded_basis} and \ref{assumption:consistency}--\ref{assumption:nuisance_rates} for the counterfactual setting:
\begin{align*}
     \max_{j = 1, \ldots, t}\left\{\left((\Pb[g_j b^T\betahat_r])^2  - \epsilon_j^2\right)_+\right\} &= O_\Pb(1/\sqrt{n}) \tag{risk-min} \\
     \sum_{j=1}^t \alpha_j \left\{(\Pb[g_j b^T\betahat_u])^2 - (\Pb[g_j b^T\betaopt_u])^2\right\} &= O_\Pb(1/\sqrt{n}) \tag{unfair-min}
\end{align*}
\end{theorem}
These results show that if a user has specific fairness or risk constraints in mind, in the observable setting, they can generate a predictor in a rich linear space that is asymptotically guaranteed to meet these constraints, while minimizing the corresponding risk or unfairness. In the counterfactual setting, they can do the same thing with the same fast rate guarantees, as long as the nuisance parameters are estimated at fast enough rates.

Of course, any particular estimates $\betahat_r, \betahat_u$ may violate their target risk and fairness constraints by arbitrary amounts, since the constraints used to compute them are themselves estimated. Suppose that $\betahat_r$ was evaluated on the test set, and one of its estimated unfairness values was found to exceed the constraint $\epsilon_j$ by an unacceptable amount. To remedy this, the user could lower the value of $\epsilon_j$ and compute a new $\betahat_r$ under this more stringent constraint. They could repeat this process until they found a $\betahat_r$ with acceptable estimated fairness. Since $\betahat_r$ is the solution to a quadratic program, however, this is computationally costly, and there is no guarantee that additional searching will yield improvements. A predictor that is more fair with respect to one fairness constraint may be \emph{less} fair with respect to other constraints, or may incur unacceptable additional risk.

Ideally, the user might wish to treat the fairness constraints as tuning parameters, selecting a large set of constraint vectors $(\epsilon_1, \ldots, \epsilon_t) \in \Rpos^t$, computing $\betahat_r$ for each vector, and comparing the risk and fairness properties of all the resulting predictors. In the next section, we use the closed-form penalized estimators to accomplish something equivalent to this, with trivial additional computational cost.

\subsection{Penalized FADE estimators}
In the observable and counterfactual settings, the estimator $\betahat_\lambda$ takes the following equivalent forms, which mirror the two expressions given for $\betaopt_\lambda$: 
\begin{align*}
    \betahat_\lambda &= \argmin_{\beta\in \Rb^k} \Pn[(b^T\beta - Y)^2] + \sum_{j=1}^t \lambda_j \left(\Pn[g_j b^T\beta]\right)^2 \tag{Observable} \\
    &= \left(\Pn(bb^T) + \sum_{j=1}^t \lambda_j\Pn(g_j b)\Pn(g_j b)^T\right)^{-1}\Pn(bY) \\
    \betahat_\lambda &= \argmin_{\beta\in \Rb^k} \Pn[(b^T\beta - \phihat)^2] + \sum_{j=1}^t \lambda_j \left(\Pn[\ghat_j b^T\beta]\right)^2 \tag{Counterfactual} \\
    &= \left(\Pn(bb^T) + \sum_{j=1}^t \lambda_j\Pn(\ghat_j b)\Pn(\ghat_j b)^T\right)^{-1}\Pn(b\phihat)
\end{align*}
assuming that the relevant matrix inverse exists. A sufficient condition for it to exist is that $\Pn[bb^T]$ is positive definite, which, as discussed above, will happen with probability 1 or approaching 1 under Assumption \ref{assumption:psd}.

The procedure we propose is given in Figure \ref{f:penalized_procedure}. The user first chooses a large set of vectors $\Lambda_n \subset \Rpos^t$, which we assume may depend on sample size. They then compute the solution set $\widehat{\calB}_n = \{\betahat_\lambda: \lambda \in \Lambda_n\}$, estimate the risk and fairness properties of each $f_\beta: \beta \in \widehat{\calB}_n$, and select a predictor with a favorable performance profile.
\begin{figure}[ht]
\fbox{\begin{minipage}{\hsize}
    \begin{enumerate}
    \item Pick a large set of vectors $\Lambda_n \subset \Rpos^t$.
    \item Compute the solution set $\widehat{\calB}_n = \{\betahat_\lambda: \lambda \in \Lambda_n\}$.
    \item Compute the estimated risk and fairness properties of each $f_\beta: \beta \in \widehat{\calB}_n$.
    \item Select a predictor $f_\beta$ with favorable risk and fairness properties.
    \end{enumerate}
    \end{minipage}}
    \caption{Penalized FADE estimation procedure.}
    \label{f:penalized_procedure}
\end{figure}

Propositions \ref{proposition:constrained_to_lagrange} and \ref{proposition:lagrange_to_constrained} established a correspondence between the constrained and penalized estimands, so each $\betahat_\lambda$ may be regarded as an estimate either of the penalized-min estimand $\betaopt_\lambda$ or of some risk-min estimand $\betaopt_r$. The value of the penalized perspective is that Step 2 in this procedure can be carried out extremely efficiently. Since each matrix $\Pn(\ghat_j b)\Pn(\ghat_j b)^T$ has rank 1, the overall matrix inverse can be computed by computing $\Pn(bb^T)^{-1}$ and then applying a series of simple algebraic operations, per the Sherman-Morrison update formula. This is expressed in the following proposition.
\begin{proposition} \label{proposition:fast_computation}
Let
\begin{align*}
    & \lambdabar_j = (\lambda_1, \ldots, \lambda_j), \text{ so that } \lambdabar_t = \lambda \\
    & m_j = \Pn(\ghat_j b) \\
    & \boldQhat_0 = \Pn(bb^T)^{-1} \\
    & \boldQhat_1(\lambda_1) = \boldQhat_0 - \frac{\lambda_1\boldQhat_0 m_1 m_1^T\boldQhat_0}{1 + \lambda_1m_1^T\boldQhat_0m_1} \\ 
    & \boldQhat_j(\lambdabar_j) = \boldQhat_{j-1}(\lambdabar_{j-1}) - \frac{\lambda_j\boldQhat_{j-1}(\lambdabar_{j-1})m_jm_j^T\boldQhat_{j-1}(\lambdabar_{j-1})}{1 + \lambda_jm_j^T\boldQhat_{j-1}(\lambdabar_{j-1})m_j}, \-\ \text{ for } j = 2, \ldots, t
\end{align*}
Then
\begin{align*}
    \betahat_\lambda = \begin{cases} \boldQhat_t(\lambda_t)^{-1}\Pn(b\phihat) & \text{(Counterfactual)} \\
                                     \boldQhat_t(\lambda_t)^{-1}\Pn(bY)      & \text{(Observable)}
                        \end{cases}
\end{align*}
\end{proposition}
Proposition \ref{proposition:fast_computation} says that to compute the set $\widehat{\calB}_n$ requires only a single matrix inversion, to compute $\boldQhat_0$. Each vector $m_j$ also only needs to be computed once. The remaining operations are algebraic. Since $\boldQhat_0$ is a $k \times k$ matrix and each $m_j$ is a vector of length $k$, if $b$ is a relatively small basis, then $\boldQhat_0$ will be fast to compute, and all the remaining algebraic operations will be fast. In our simulations and real data analyses, we show that we can get good results with a very small number of basis functions (e.g. 4 to 6), which yield extremely fast computations.

How should $\Lambda_n$ be chosen in Step 1? Since $\Lambda_n \subset \Rpos^t$, one simple possibility is to take a one-dimensional grid of points between $0$ and some arbitrary large number and then construct the $t-$dimensional Cartesian product. Since $\betahat_\lambda$ is smooth in $\lambda$, and since the risk and fairness measures are smooth in $\betahat$, we can expect that such a grid will enable us to move smoothly around the fairness-accuracy space, and that we won't be missing desirable predictors that lie in between the grid points\footnote{The movement won't be entirely smooth if predictions are truncated to lie in $[\ell_y, u_y]$.}.

Another possibility is to ``seed'' $\Lambda_n$ with values that correspond to a particular $\betahat_r$. That is, fix some constraints $\epsilon_j$ and solve the dual of the risk-min program that defines $\betahat_r$. The dual solution $\lambdaopt$ indexes a penalized-min problem whose solution is $\betahat_r$, so $\Lambda_n$ can then be constructed as a grid around this $\lambdaopt$. See the proof of Proposition \ref{proposition:constrained_to_lagrange} for a more detailed explanation.

This ``seeding'' approach provides a way to ensure that the set $\{\betahat_\lambda: \lambda \in \Lambda_n\}$ includes estimators that in some sense target reasonable constraints, particularly for users with specific constraints in mind. This approach requires solving just a single constrained optimization problem, to establish a point of reference in fairness-accuracy space.

The procedure we have described allows users to efficiently construct and evaluate a very large set of models that fall in different points in fairness-accuracy space. In sections \ref{sec:simulations}, \ref{sec:compas}, and \ref{sec:adult}, we show that this procedure enables us to find high-performing models in both observable and counterfactual settings, with simulated and real data. With minimal searching over possible bases, we are able to find models that substantially outperform existing models and methods with respect to both fairness and accuracy.

\begin{remark}[Penalized version of unfair-min]
Since $\betaopt_\lambda$ is constructed as a penalized equivalent of $\betaopt_r$, the seeding approach to constructing $\Lambda_n$ that we have described allows users to target particular fairness constraints but not particular risk constraints. It is straightforward to develop an analogous procedure around a penalized version of $\betaopt_u$ that allows users to seed $\Lambda_n$ with estimators that target particular risk constraints. In practice, it is not likely to matter much, since the construction of $\betahat_\lambda$ should allow users to flexibly explore the fairness-accuracy space and find an estimator that accommodates their desired constraints, if one exists in the span of the chosen basis.
\end{remark}

\begin{remark}[Arbitrary FADE weights]
An even simpler and plausibly just as effective alternative to computing the collection $\{\betahat_\lambda: \lambda \in \Lambda_n\}$ is to simply define an arbitrary set $\calB \subset  \Rb^k$, perhaps constrained to lie in the simplex or in an $L_1$ box around the origin. That is, the user could simply evaluate arbitrary sets of basis weights to see if any of them yields a reasonable predictor. This set could be similarly constructed as a grid around a particular $\betahat_r$ or $\betahat_u$, if users have specific fairness or accuracy constraints they wish to target.
\end{remark}

\begin{remark}[Resemblance to ridge regression]
$\betahat_\lambda$ resembles a ridge regression estimator. In ridge regression and other regularized estimators, however, the penalty tuning parameter $\lambda$ is expected to go to 0 as $n \rightarrow \infty$. In our setting, $\lambda$ serves to enforce fairness rather than to modulate the variance-bias tradeoff, so there is no reason for it to shrink with $n$. Without unfairness penalties, the predictor won't automatically get more fair as the data gets larger. 
\end{remark}

We now develop theoretical guarantees for the penalized FADE estimators. Let $h(n)$ denote the rate at which the product of nuisance parameter errors $\Vert \pihat - \pi \Vert \Vert \muhat_0 - \mu_0 \Vert$ grows or converges, and let $\hbar(n)$ denote the rate for $\Vert \pihat - \pi \Vert \Vert \nuhat_0 - \nu_0 \Vert$. That is,
\begin{align*}
    & \Vert \pihat - \pi \Vert \Vert \muhat_0 - \mu_0 \Vert = O_\Pb(h(n)) \\
    & \Vert \pihat - \pi \Vert \Vert \nuhat_0 - \nu_0 \Vert = O_\Pb(\hbar(n))
\end{align*}
Under ideal conditions, Assumption \ref{assumption:nuisance_rates} will hold, so that the product of nuisance parameter errors decay faster than $1/\sqrt{n}$, but the subsequent results do not require this.

\begin{assumption}[Compact superset $\Lambda$] \label{assumption:compact_lambda}
For all $n$, $\Lambda_n \subseteq \Lambda \subset  \Rb^t$ for some compact set $\Lambda$.
\end{assumption}

\begin{definition}For any $\lambda \in \Rpos^t$, the \emph{excess risk} for $\betahat_\lambda$ is
\begin{align*}
    & \Pb[(b^T\betahat_\lambda - \Yt)^2] - \Pb[(b^T\betaopt_\lambda - \Yt)^2]
\end{align*}
\end{definition} 

\begin{theorem}[Uniform rate for excess risk in the penalized setting] \label{LS_thm:series_rate_risk}
Under Assumptions \ref{assumption:psd}--\ref{assumption:bounded_basis} for the observable setting; and Assumptions \ref{assumption:psd}--\ref{assumption:bounded_basis}, \ref{assumption:consistency}--\ref{assumption:nuisance_rates}, and \ref{assumption:compact_lambda} for the counterfactual setting:
\begin{align*}
    \sup_{\lambda\in\Lambda}\left\{\Pb\left[\left(b^T\betahat_\lambda - \Yt\right)^2\right] - \Pb\left[\left(b^T\betaopt_\lambda - \Yt\right)^2\right]\right\} &= O_\Pb(\sqrt{1/n}) + O_\Pb(h(n))
\end{align*}
\end{theorem}
In other words, the excess risk goes to 0 uniformly at $\sqrt{1/n}$ or the nuisance rate $h(n)$, whichever is slower. We have a similar result for the excess unfairness, which is defined as follows.

\begin{definition}
For any $\lambda \in \Rpos^t$, the \emph{excess unfairness} for $\betahat_\lambda$ is
\begin{align*}
    & \left\{\max_{j \in 1, \ldots, t}\left(\Pb\left[g_jb^T\betahat_\lambda\right] - \Pb\left[g_jb^T\betaopt_\lambda\right]\right)\right\}
\end{align*}
We have defined excess unfairness as the max over $j$, but it makes little difference if we define it instead as the sum over $j$. Note that here we haven't used the squared unfairness.
\end{definition}

\begin{theorem}[Uniform rate for excess unfairness in the penalized setting] \label{LS_thm:series_rate_unfairness} 
Under Assumptions \ref{assumption:psd}--\ref{assumption:bounded_basis} for the observable setting; and Assumptions \ref{assumption:psd}--\ref{assumption:bounded_basis}, \ref{assumption:consistency}--\ref{assumption:nuisance_rates}, and \ref{assumption:compact_lambda} for the counterfactual setting:
\begin{align*}
    \sup_{\lambda\in\Lambda}\left\{\max_{j \in 1, \ldots, t}\left(\Pb\left[g_jb^T\betahat_\lambda\right] - \Pb\left[g_jb^T\betaopt_\lambda\right]\right)\right\} &= O_\Pb(\sqrt{1/n}) + O_\Pb(h(n))
\end{align*}
\end{theorem}

\begin{remark}[Allowing $k \rightarrow \infty$] \label{remark:k_to_infinity}
We can obtain similar theoretical results in a regime in which the basis dimension $k$ is allowed to grow to $\infty$, if we require that $\sup_{w\in\calW} \Vert b(w) \Vert = O(\sqrt{k})$ and that $k\log(k)/n \rightarrow 0$. The first requirement is a stronger version of Assumption \ref{assumption:bounded_basis}, while the second insists that $k$ not grow too fast in $n$. Under these additional requirements, we attain a rate of $O_\Pb(\sqrt{k/n}) + O_\Pb(\sqrt{k}\cdot h(n))$ in Theorems \ref{LS_thm:series_rate_risk} and \ref{LS_thm:series_rate_unfairness}. These results extend the results of \cite{belloni_new_2015} to a setting with nuisance parameters and penalty terms. As illustrated in that paper, these requirements are weak enough to allow the basis to asymptotically span rich function spaces such as the space of square integrable functions.
\end{remark}

\subsection{Risk and unfairness of a fixed predictor}
Once a user has constructed a predictor or a set of candidate predictors, they will naturally wish to estimate the risk and fairness properties of those predictors, for example before choosing one to deploy in a decision-making context. The risk and unfairness of a fixed predictor $f_\beta$ are estimated as
\begin{align*}
    \riskhat(f_\beta) &= \begin{cases}
                        \Pn[(f_\beta - Y)^2] & \text{(Observable)} \\
                        \Pn[f_\beta^2 - 2f_\beta\phihat + \phibarhat] & \text{(Counterfactual)}
                        \end{cases} \\
    \UFhat_j(f_\beta) &= \begin{cases}
                        \Pn[g_j(W, Y)f_\beta] & \text{(Observable)} \\
                        \Pn[g_j(W, \phihat)f_\beta] & \text{(Counterfactual)}
                        \end{cases}           
\end{align*}
for $j = 1, \ldots t$.

\begin{theorem}[Asymptotic normality of risk and unfairness estimators] \label{LS_thm:asymptotic_normality}
Consider fairness functions $g_j \in \{g^\text{rate}, g^\text{FPR}, g^\text{FNR}\}$. Under Assumptions \ref{assumption:psd}--\ref{assumption:bounded_basis} for the observable  (Obs.) setting; and Assumptions \ref{assumption:psd}--\ref{assumption:bounded_basis}, \ref{assumption:consistency}--\ref{assumption:nuisance_rates}, and \ref{assumption:compact_lambda} for the counterfactual (Count.) setting:
\begin{align*}
    & \sqrt{n}\left(\riskhat(f_\beta) - \risk(f_\beta)\right) \xrightarrow{d}
    \begin{cases}
    N\left(0, \var((f_\beta - Y)^2)\right) & \text{(Obs.)} \\
    N\left(0, \var(f_\beta^2 - 2f_\beta\phi + \phibar)\right) & \text{(Count.)}
    \end{cases} \\
    & \sqrt{n}\left(\UFhat_j(f_\beta) - \UF_j(f_\beta)\right) \xrightarrow{d}
    \begin{cases}
    N\left(0, \var(g_j(W, Y)f_\beta)\right) & \text{(Obs.)} \\
    N\left(0, \var\left(\Pb(\gamma_0)^{-1}\eta_0 - \Pb(\gamma_1)^{-1}\eta_1\right)\right) & \text{(Count.)}
    \end{cases}
\end{align*}
where, for $a \in \{0, 1\}$,
\begin{align*}
    \gamma_a &= \begin{cases}
    (1 - \phi)\one\{A = a\} & \text{(for $g^\text{FPR}$)} \\
    \phi\one\{A = a\}       & \text{(for $g^\text{FNR}$)}
    \end{cases} \\
    \eta_a &= \gamma_a\left(f_\beta - \frac{\Pb[\gamma_a f_\beta]}{\Pb[\gamma_a]}\right)
\end{align*}
\end{theorem}

Under Theorem \ref{LS_thm:asymptotic_normality}, asymptotically valid confidence intervals can be constructed, and asymptotically valid hypothesis tests conducted, for $\risk(f_\beta)$ and $\UF_j(f_\beta)$.

In the next three sections, we illustrate FADE on simulated and real data, in both observable and counterfactual settings. Given the correspondence between the constrained and penalized FADE forms established in Propositions \ref{proposition:constrained_to_lagrange} and \ref{proposition:lagrange_to_constrained}, we do not conduct separate simulations for these two settings. We emphasize the penalized form because of its substantial computational advantages over the constrained form.

\section{Simulations} \label{sec:simulations}
We illustrate the penalized FADE procedure in the counterfactual setting, i.e when $\Yt = Y^0$. As in a real data setting, each estimator $\betahat_\lambda$ is constructed using only observable data, but unlike in a real data setting, we use the known values of $Y^0$ to evaluate the resulting predictors.

All computations in this and subsequent sections were carried out on a 2013 MacBook Pro with a 2.4 GHz dual-core processor and 8GB of RAM.

\subsection{Data-generating process}
The data generating process is as follows, for data $Z = (A, X, D, Y^0, Y^1, Y)$.
\begin{align*}
    \Pb(A = 1) &= 0.3 \\
    X \mid A &\sim N\left(A*(1, -0.8, 4, 2)^T, I_4\right) \\
    \Pb(D = 1 \mid A, X) &= \min\{0.975, \-\ \expit((A, X)^T(0.2, -1, 1, -1, 1))\} \\
    \Pb(Y^0 = 1 \mid A, X, D) &= \expit((A, X)^T(-5, 2, -3, 4, -5)) \\
    \Pb(Y^1 = 1 \mid A, X, D) &= \expit((A, X)^T(1, -2, 3, -4, 5)) \\
    Y &= (1 - D)Y^0 + DY^1
\end{align*}
where $I_4$ denotes the $4\times 4$ identity matrix. $A = 1$ represents the minority group. There are no previously trained predictors; i.e. $S = \emptyset$, so the collected covariates consist of $W = (A, X)$. This data generating process satisfies the identifying assumptions, Assumptions \ref{assumption:consistency}--\ref{assumption:ignorability}: the last line expresses the consistency assumption; the propensity score $\pi(A, X) = \Pb(D = 1 \mid A, X)$ is upper bounded at 0.975 to satisfy positivity; and $Y^a \ind D \mid W$ for $a \in \{0, 1\}$, satisfying ignorability.

The two groups $A = 0$ and $A = 1$ differ in the distribution of covariates (Figure \ref{f:simulated data_covariate_distribution}) and decisions and outcomes (Table \ref{t:sim1_outcome_distribution}). The minority group experiences a positive decision $(D = 1)$ $18\%$ of the time, while the majority group experiences it $50\%$ of the time. Outcomes $Y^0$ and $Y$ are higher for the minority group, with a larger disparity for potential outcomes than for observable outcomes. All measures are computed on a dataset of size 50,000.

\begin{figure}[ht]
    \centering
    \includegraphics[width=\textwidth]{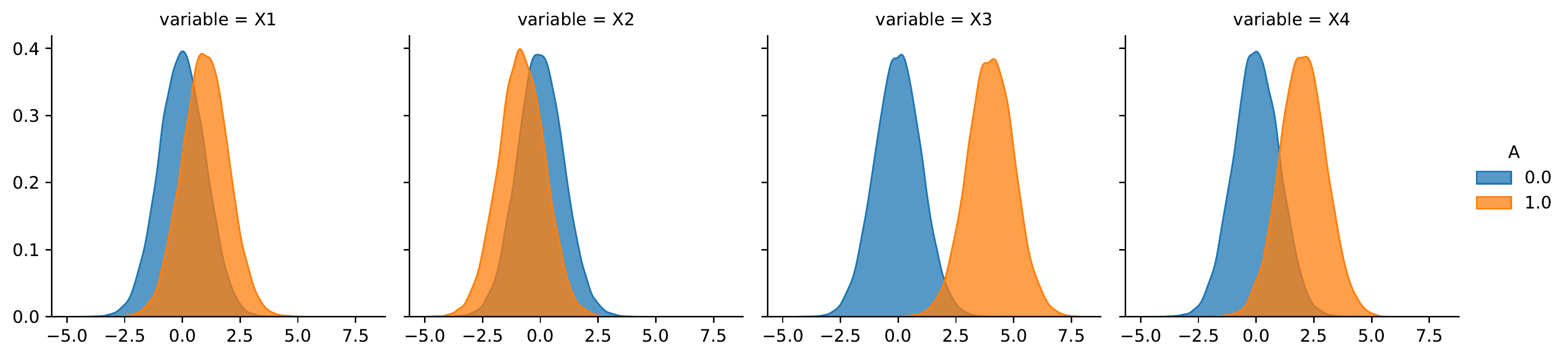}
    \caption{Conditional covariate distributions for the two groups $A = 0$ and $A = 1$ in the simulated data. Curves are kernel density estimates.}
    \label{f:simulated data_covariate_distribution}
\end{figure}

\begin{table}[ht]
    \centering
    \begin{tabular}{lrrr}
    \toprule
    $A$ & $\E[D \mid A]$ & $\E[Y^0 \mid A]$ & $\E[Y \mid A]$ \\
    \midrule
    0 & 0.50 & 0.50 & 0.67 \\
    1 & 0.18 & 0.76 & 0.71 \\
    \bottomrule
    \end{tabular}
    \vspace{0.5em}
    \caption{Distribution of decisions and outcomes for groups $A = 0$ and $A = 1$ in the simulated data.}
    \label{t:sim1_outcome_distribution}
\end{table}

As a reference point for our method, in Table \ref{t:sim1_bayes_performance} we compute the performance of the counterfactual Bayes-optimal predictor $f(A, X) = \E[Y^0 \mid A, X]$, which is defined in the data-generating process. We include both MSE, with is the measure directly targeted by our method, as well as area under the curve (AUC). The Bayes-optimal predictor is highly accurate, with an MSE of 0.05 and an AUC of 0.98. The MSE of 0.05 is a lower bound on the risk achievable by any predictor. Unsurprisingly, given the difference in the distribution of outcomes across the two groups, the Bayes-optimal predictor has a large rate-diff ($\lvert\E[f \mid A = 0] - \E[f \mid A = 1]\rvert$). The differences in generalized false positive and false negative rates, however, are relatively small. 

\begin{table}[ht]
    \centering
    \begin{tabular}{lccccc}
    \toprule
     Predictor & MSE & AUC & rate-diff &  FPR-diff &  FNR-diff \\
    \midrule
          Bayes-optimal &  0.05 &       0.98 &       0.26 &       0.07 &       0.05 \\
    \bottomrule
    \end{tabular}
    \caption{Risk and fairness measures with respect to $Y^0$ for the Bayes-optimal predictor $\E[Y^0 \mid A, X]$ in the simulated data. The predictor is highly accurate, with low MSE and high AUC. It has a relatively large rate disparity but small disparities in generalized false positive and false negative rates.}
    \label{t:sim1_bayes_performance}
\end{table}

\subsection{Base predictors and nuisance models}
We now investigate the performance of the penalized FADE procedure. We randomly sample three iid datasets of size n = 1000, representing $\datalearn, \datatrain^\text{nuis}$, and $\datatrain^\text{target}$. We train four base predictors on $\datatrain$, with $A, X$ as covariates and $Y$ as the outcome. We train only on data in which $D = 0$: under the ignorability assumption, $\E[Y \mid A, X, D = 0] = \E[Y^0 \mid A, X]$, so this results in predictors which are designed to estimate $Y^0$. The predictors consist of a random forest, a gradient boosted (GB) classifier, a Gaussian Naive Bayes model, and a ridge regression, all chosen for convenience and ease of computation. (In practice, a logistic regression would be a natural choice for a base predictor. Since the actual regression function $\E[Y^0 \mid A, X]$ is logistic, however, we do not use this model in order to avoid making the problem too easy, and to simulate a real data setting in which it is unlikely that the true regression function is known up to a finite dimensional parameter.) In addition to these predictors, we include a mean predictor, which always predicts the conditional sample mean $\Pn(Y \mid D = 0)$. This plays essentially the same role as an intercept in ordinary linear regression.

We use random forest classifiers to estimate the propensity and outcome models $\pihat$ and $\muhat_0$ on $\datatrain^\text{nuis}$. All models were trained with their default tuning parameters using the scikit-learn library in Python. Predictors $\betahat_\lambda$ are computed using $\datatrain^\text{target}$.

After a set of penalty vectors $\Lambda_n \subset \Rpos^t$ is chosen and the corresponding model coefficients $\widehat{\calB}_n = \{\betahat_\lambda: \lambda \in \Lambda_n\}$ are computed, we estimate the risk of fairness properties of every $f_\beta: \beta \in \widehat{\calB}_n$. In order to understand the true range of risk and fairness values that our method produces, we use a large test set $\datatest$ of size 10,000, and in place of $\phihat$, the nuisance quantity that would be required in a real data setting, we use the known values of $Y^0$ to compute the risk and fairness estimates. (For comparison purposes, estimates were also computed using $\mu_0$ instead of $Y^0$; the results were virtually identical.) Since the true $Y^0$ is used, there is no need to split $\datatest$ into $\datatest^\text{nuis}$ and $\datatest^\text{target}$.

Table \ref{t:sim1_base_performance} shows the performance of the base predictors as well as the ordinary unpenalized least squares (OLS) solution, i.e. the predictor $f_{\beta_\lambda}$ with $\lambda = 0$. The OLS predictor is the (estimated) MSE-minimal aggregation of the five base predictors, computed without regard for fairness. The OLS weights are $[-0.27,  0.09,  0.40, -0.11,  0.94]$; each base predictor appears to make a nontrivial contribution. The MSE of the base predictors ranges from 0.08 to 0.27. The four non-constant predictors improve substantially on the mean predictor with respect to both MSE and AUC. The mean predictor necessarily has a value of 0 for all three disparities, while the disparities of the other base predictors vary between 0.10 and 0.59. As expected, the OLS predictor has lower MSE than any of the base predictors. The performance of the OLS predictor is similar to the performance of the Bayes-optimal predictor in Table \ref{t:sim1_base_performance}: both have a small MSE, a relatively large rate disparity, and relatively small error rate disparities.

\begin{table}[ht]
    \centering
    \begin{tabular}{lrrrrr}
    \toprule
             Predictor &  MSE &  AUC &  rate-diff &  FPR-diff &  FNR-diff \\
    \midrule
          Mean &       0.27 &       0.50 &       0.00 &       0.00 &       0.00 \\
 Random Forest &       0.09 &       0.95 &       0.28 &       0.21 &       0.11 \\
 GB Classifier &       0.08 &       0.96 &       0.31 &       0.22 &       0.10 \\
   Naive Bayes &       0.17 &       0.84 &       0.52 &       0.59 &       0.40 \\
         Ridge &       0.09 &       0.98 &       0.22 &       0.09 &       0.10 \\
         \hline
 OLS &       0.07 &       0.98 &       0.26 &       0.10 &       0.08 \\
    \bottomrule
    \end{tabular}
    \caption{Performance of the five base predictors and the ordinary least squares (OLS) predictor in the simulated data. The OLS weights are $[-0.27,  0.09,  0.40, -0.11,  0.94]$. The OLS predictor substantially improves on the MSE of the base predictors. The performance profile of the OLS predictor is close to the profile of the Bayes-optimal predictor in Table \ref{t:sim1_bayes_performance}.}
    \label{t:sim1_base_performance}
\end{table}

\subsection{Results: one unfairness penalty}
We now compute a set of unfairness-penalized predictors, applying a single unfairness penalty at a time. Let
\begin{align*}
    \Lambda_{n,1} = \{0, 0.001, 0.01, 1, 10, 20, 50, 100, 500, 1000, 2000\}.
\end{align*}
For each $\lambda \in \Lambda_{n,1}$, we compute $\betahat_\lambda$ for each fairness function $g \in \{g^\text{rate}, g^\text{FPR}, g^\text{FNR}\}$. The value $\lambda = 0$ corresponds in each case to the OLS solution, so this yields a total of $(\lvert\Lambda_{n,1}\rvert - 1)*3 + 1 = 31$ predictors.

The risk and unfairness values for each predictor are plotted in Figure \ref{f:sim1_one_constraint}. The disparity corresponding to the targeted constraint is represented by a solid line, while the other two disparities and the MSE are represented by dashed lines. We emphasize that the values in this figure are computed on $\datatest$, after the predictors $\betahat_\lambda$ are computed on $\datatrain$.

As expected, as $\lambda$ increases, the targeted disparity of the resulting predictor generally decreases. The decrease is monotonic, except at one point: $\lambda = 1$ for FPR-diff, which may be a result of sampling noise. The rate difference decreases from 0.26 to 0.04. The FPR disparity decreases from 0.10 to 0, then remains at 0.01. The FNR disparity decreases from 0.08 to 0.04. When the rate disparity is penalized, the decrease in the target disparity is accompanied by a slight increase in MSE, from 0.07 to 0.09, as well as small increases in FPR-diff and FNR-diff. When FPR-diff is penalized, all three disparities fall together, while the increase in MSE is miniscule, from 0.067 to 0.071. The same is true when FNR-diff is targeted: the MSE increases from 0.067 to 0.069.

These results illustrate (1) that the penalty term successfully controls the target disparity, (2) that an increase in fairness need not come at the cost of a substantial decrease in accuracy, and (3) that a decrease in one disparity need not produce an increase in other disparities. In the second two panels, the FADE predictors are uniformly more fair than the OLS predictor, with essentially no change in accuracy. Additionally, in these two panels, even though only one disparity was penalized at a time, all three disparities decrease as $\lambda$ increases. 

\begin{figure}
    \centering
    \includegraphics[width=\textwidth]{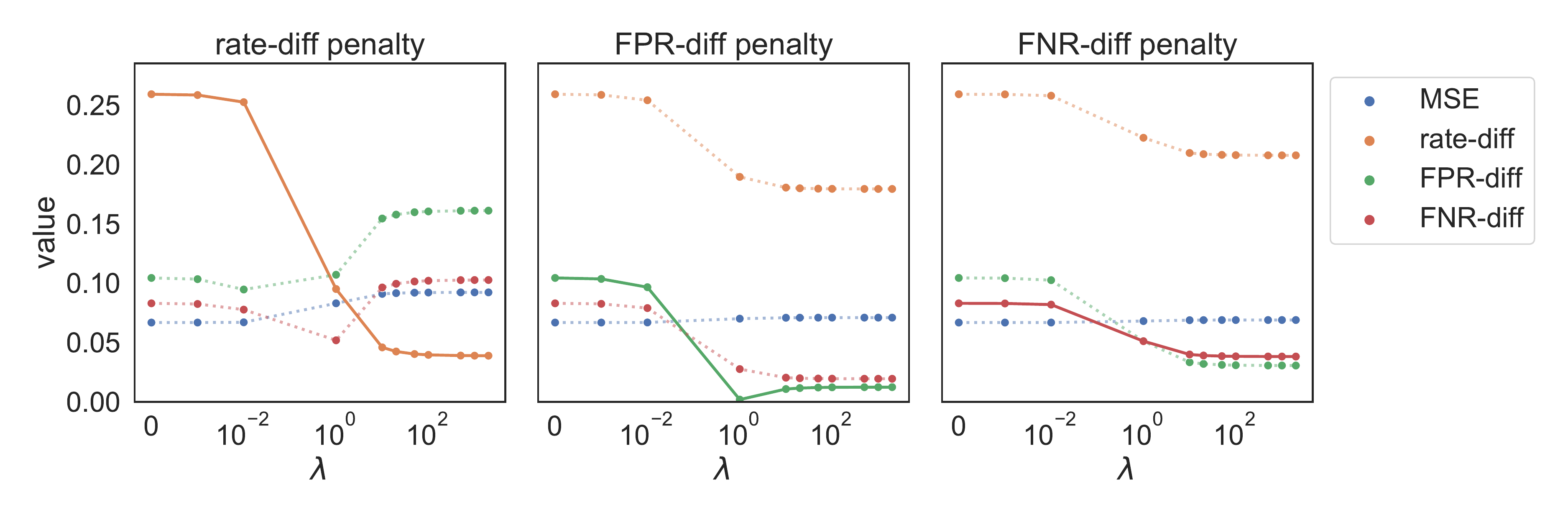}
    \caption{Risk and fairness for predictors subject to one of three penalties, in the simulated data. The x-axis represents the unfairness penalty coefficient $\lambda$, with the leftmost point ($\lambda = 0$) in each panel corresponding to the OLS solution. The y-axis represents the MSE and the disparity values of the resulting predictor $\betahat_\lambda$, computed on an independent test set of size 10,000 using the known values of $Y^0$. Solid lines indicate the metric that is penalized in training.}
    \label{f:sim1_one_constraint}
\end{figure}

\subsection{Results: multiple unfairness penalties}
We now apply all three unfairness penalties simultaneously. Define $\Lambda = \Lambda_{n,1} \times \Lambda_{n,1} \times \Lambda_{n,1} \subset \Rpos^3$. The collection $\widehat{\calB}_n = \{\betahat_\lambda: \lambda \in \Lambda\}$ now indexes $\lvert\Lambda_{n,1}\rvert^3 = 1331$ predictors. We use the same base predictors and nuisance predictors as in the previous section. The process of training the predictors, computing $\widehat{\calB}_n$, and estimating the risk and fairness for each predictor, took less than 10 seconds.

Figure \ref{f:sim1_two_constraints_MSE} plots each of the three disparities against MSE, for each of the 1331 predictors, as well as the base predictors and the OLS predictor. As expected, the OLS predictor has the smallest MSE. Fewer than 1331 dots are visible in each panel, due to the fact that many of the predictors substantially overlap in fairness-accuracy space. Nevertheless, the predictors span a wide range of performance profiles. For all three disparities, predictors exist that take the disparity to 0, with relatively small increase in MSE relative to the OLS predictor. For rate-diff and FNR-diff, these predictors notably do not appear in Figure \ref{f:sim1_one_constraint}, where the lowest value for these two disparities are 0.04. We only discover these predictors by applying multiple penalties simultaneously. All the FADE predictors are substantially more accurate than the mean predictor.

Figure \ref{f:sim1_two_constraints_fairness} plots the same 1331 predictors with respect to each pair of disparities, with color indicating MSE. This figure illustrates the interplay of three metrics at once, and is of interest to users who wish to control two disparities simultaneously. For example, users who wish to target (counterfactual) equalized odds would be interested in the bottom panel that plots FNR-diff and FPR-diff.

These views once again reveal a wide range of predictor behavior. Unsurprisingly, many of the highest MSE predictors are close to the origin, but the relationship between MSE and distance to the origin is far from monotonic. In all three panels, there is a line of predictors stretching from the OLS predictor that represent improvements in both disparities with minimal increase in MSE. In the bottom panel, for example, there are predictors with FPR-diff close to 0, FNR-diff under 0.05, and MSE under 0.10. These predictors approximately satisfy equalized odds, and they represent an increase in MSE of less than 0.03 relative to the OLS predictor.

In order to examine this more precisely, Table \ref{t:sim1_lowest_distance} shows the performance of the predictors with the minimum L2 distance from the origin in the fairness-accuracy subspace defined by MSE as well as zero to three disparities. For example, the ``MSE + rate'' row represents the predictor with the smallest $L_2$ norm in the (MSE, rate-diff) vector, while the ``MSE + rate + FPR + FNR'' row represents the predictor with the smallest $L_2$ norm in the (MSE, rate-diff, FPR-diff, FNR-diff) vector. FPR-diff and FNR-diff can be minimized, singly or jointly, with no increase in MSE relative to the OLS predictor. Rate-diff can be substantially reduced with a relatively small increase in MSE. Perhaps surprisingly, all three disparities can be jointly minimized, to 0.06 (rate-diff), 0.03 (FPR-diff), and 0.02 (FNR-diff), with only a 0.06 increase in MSE and a 0.02 decrease in AUC relative to the unpenalized OLS predictor.

\begin{figure}
    \centering
    \begin{subfigure}{0.49\textwidth}
        \includegraphics[width=\linewidth]{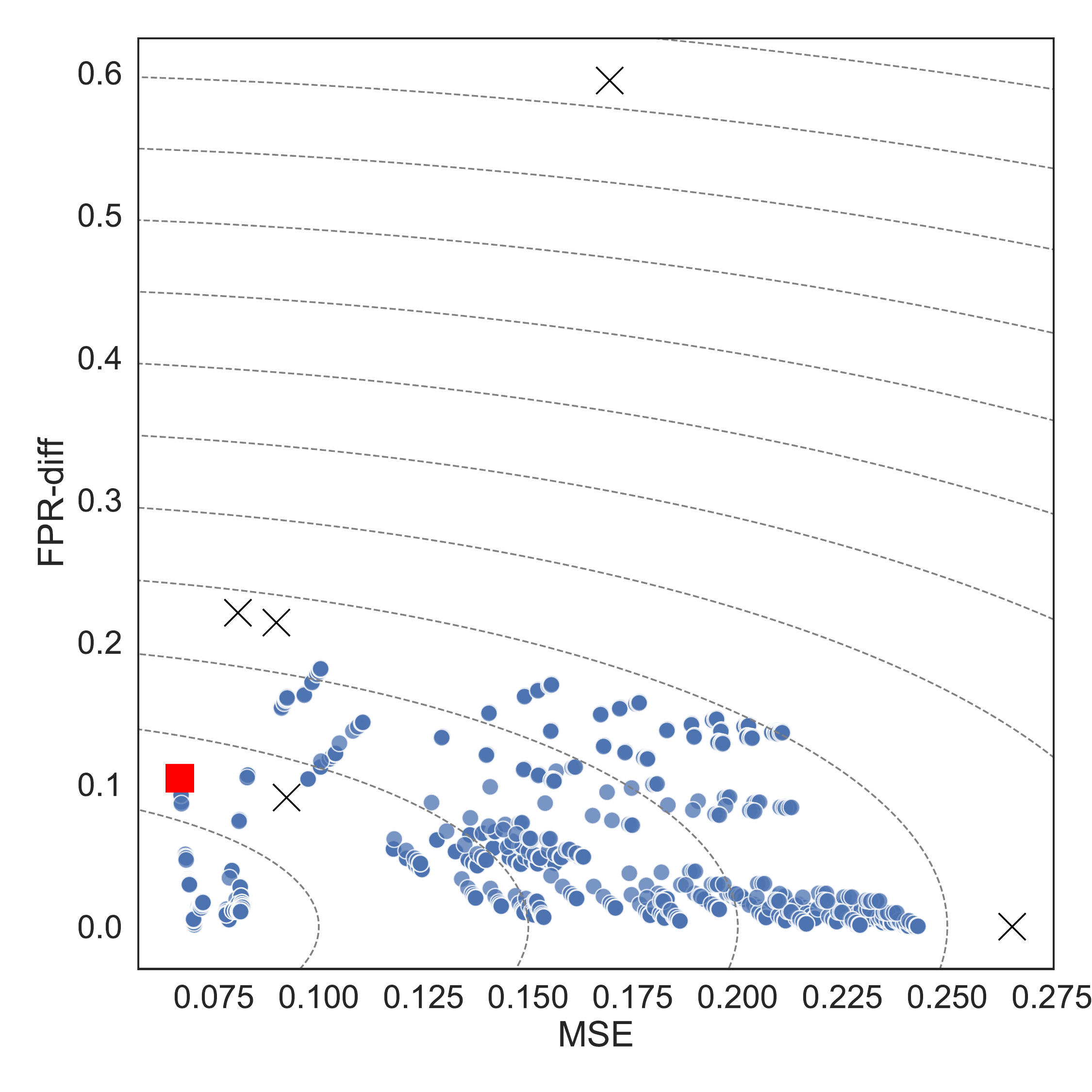}    
    \end{subfigure}
    \begin{subfigure}{0.49\textwidth}
        \includegraphics[width=\linewidth]{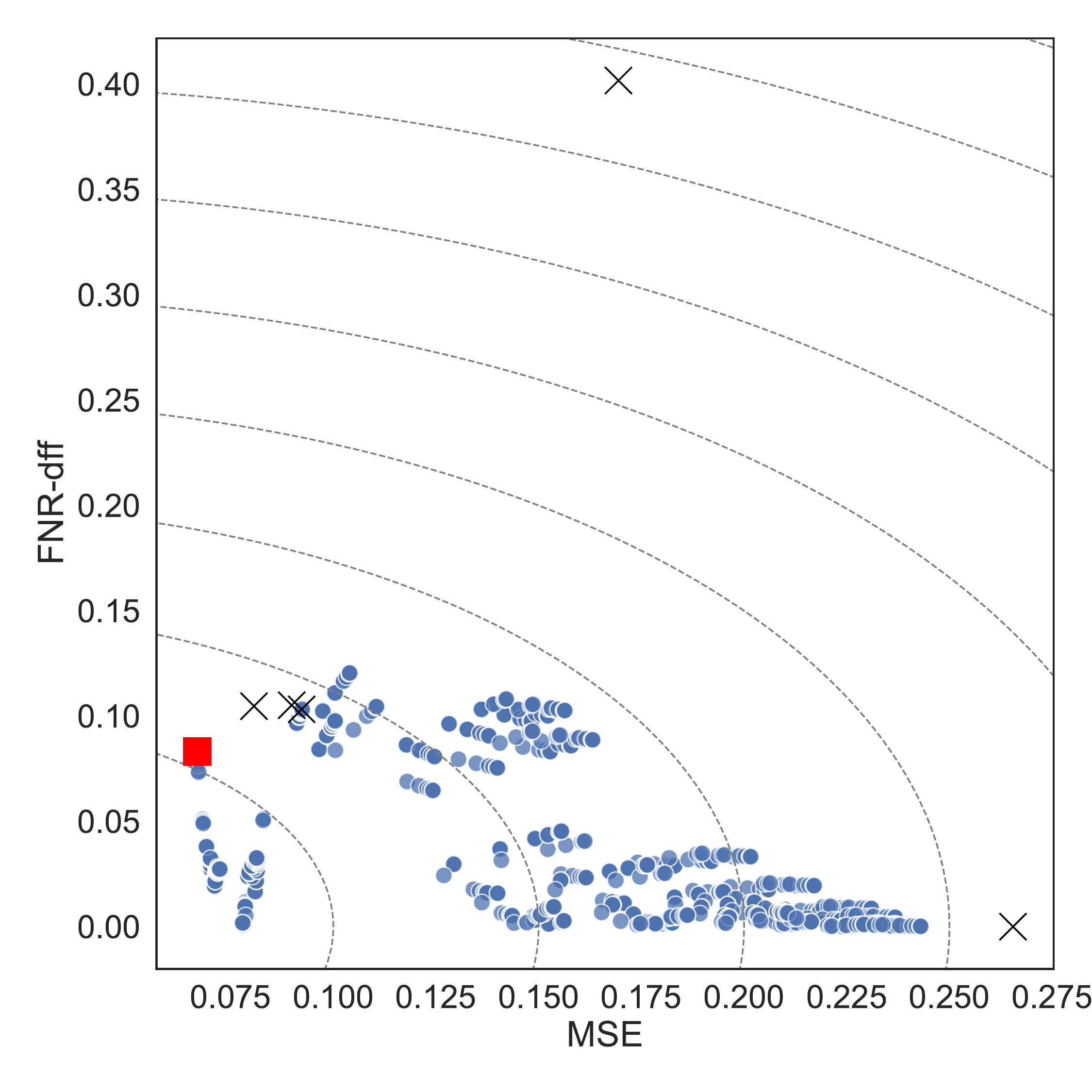}    
    \end{subfigure}
    \vspace{-0.5em}
    
    \begin{subfigure}{0.49\textwidth}
        \includegraphics[width=\linewidth]{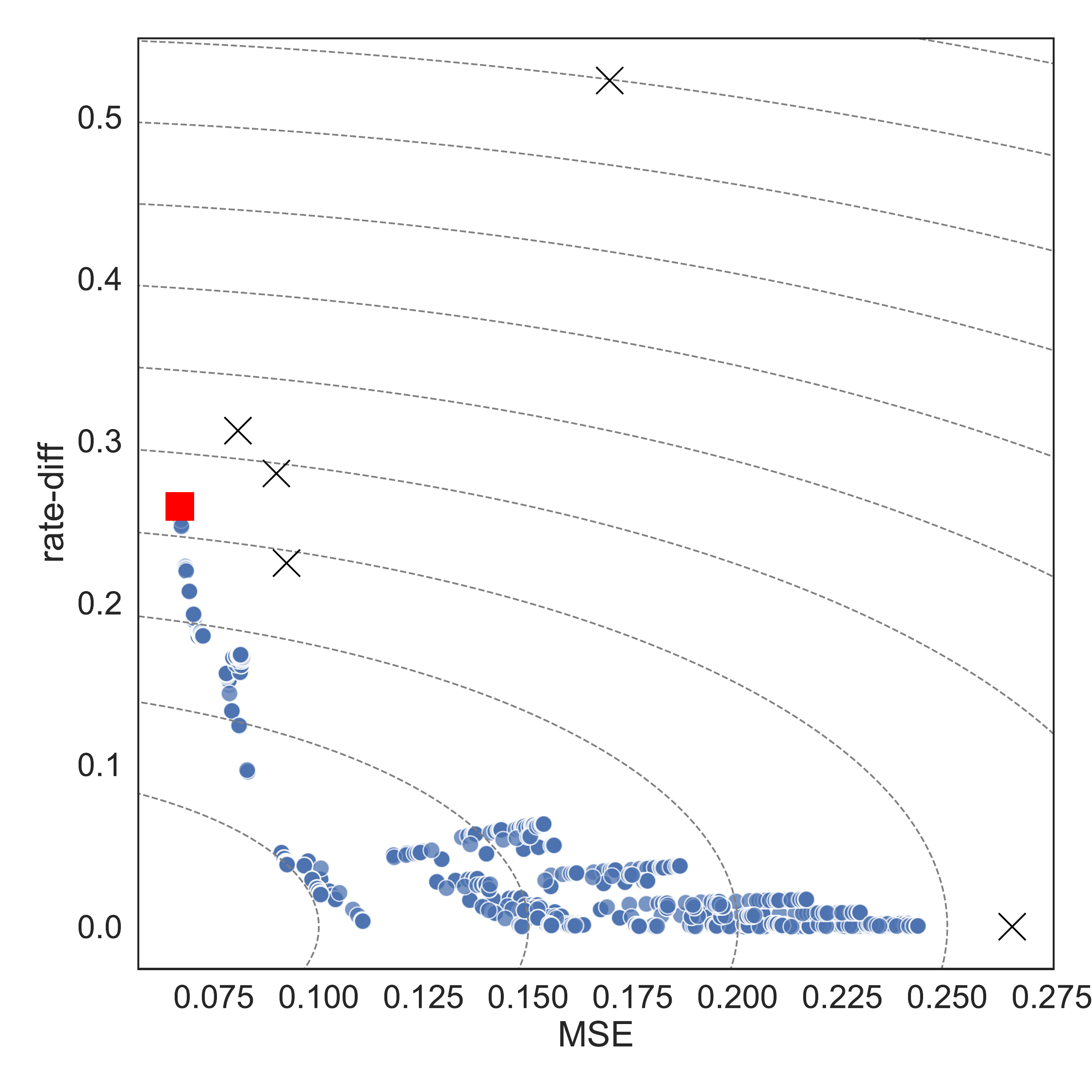}    
    \end{subfigure}
    \caption{Disparity and MSE values for each of 1331 predictors in the simulated data. Black ``X''s represent the base predictors, with the mean predictor at the bottom right of each panel. The red square is the OLS predictor, and the blue dots are the FADE predictors. Radius lines indicate distance from the origin. Despite substantial overlap, the predictors span a wide range of fairness and accuracy values. For each disparity, many predictors exist which take that disparity to 0, at a small cost in MSE relative to the OLS predictor.}
    \label{f:sim1_two_constraints_MSE}
\end{figure}

\begin{figure}
    \centering
    \begin{subfigure}{0.49\textwidth}
        \includegraphics[width=\linewidth]{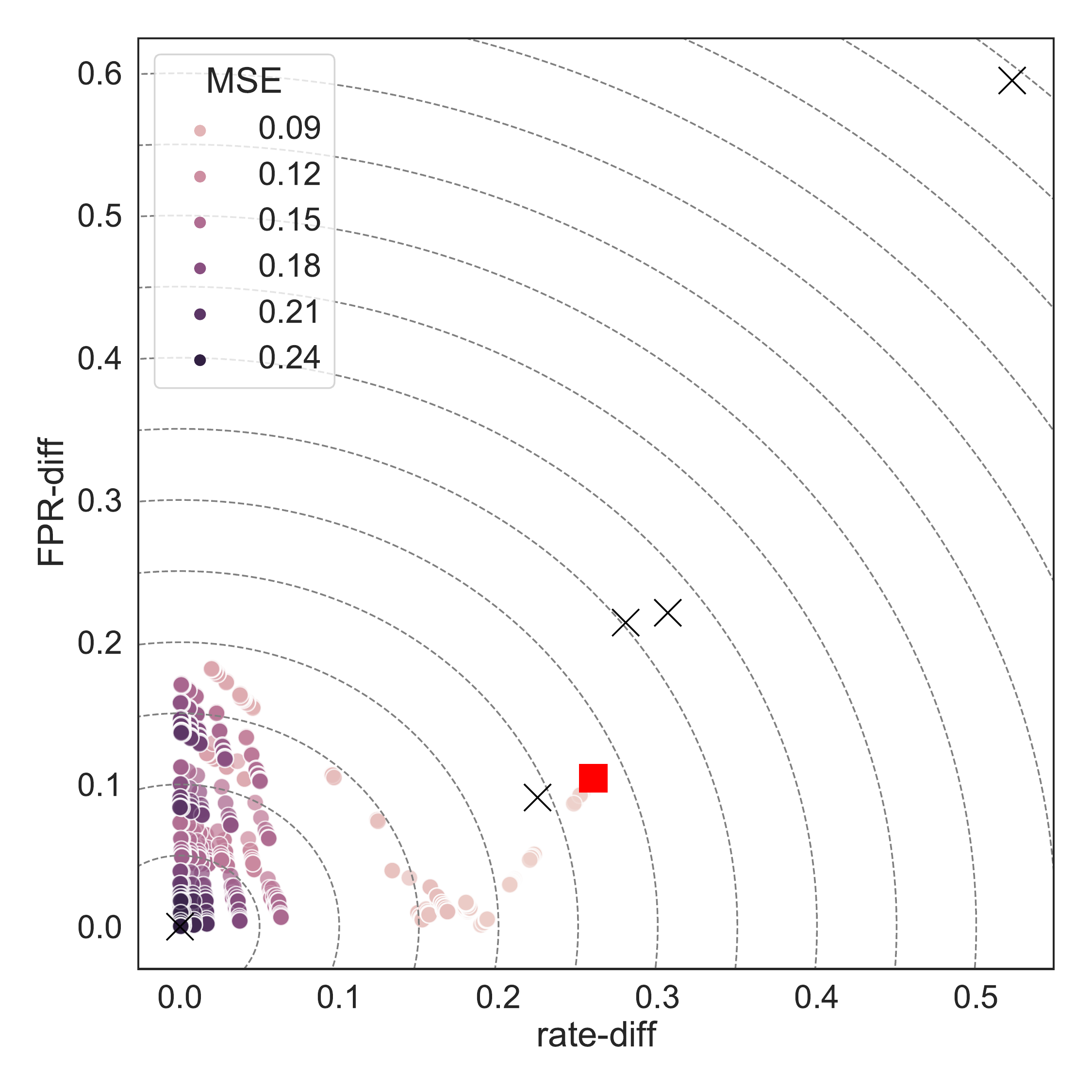}    
    \end{subfigure}
    \begin{subfigure}{0.49\textwidth}
        \includegraphics[width=\linewidth]{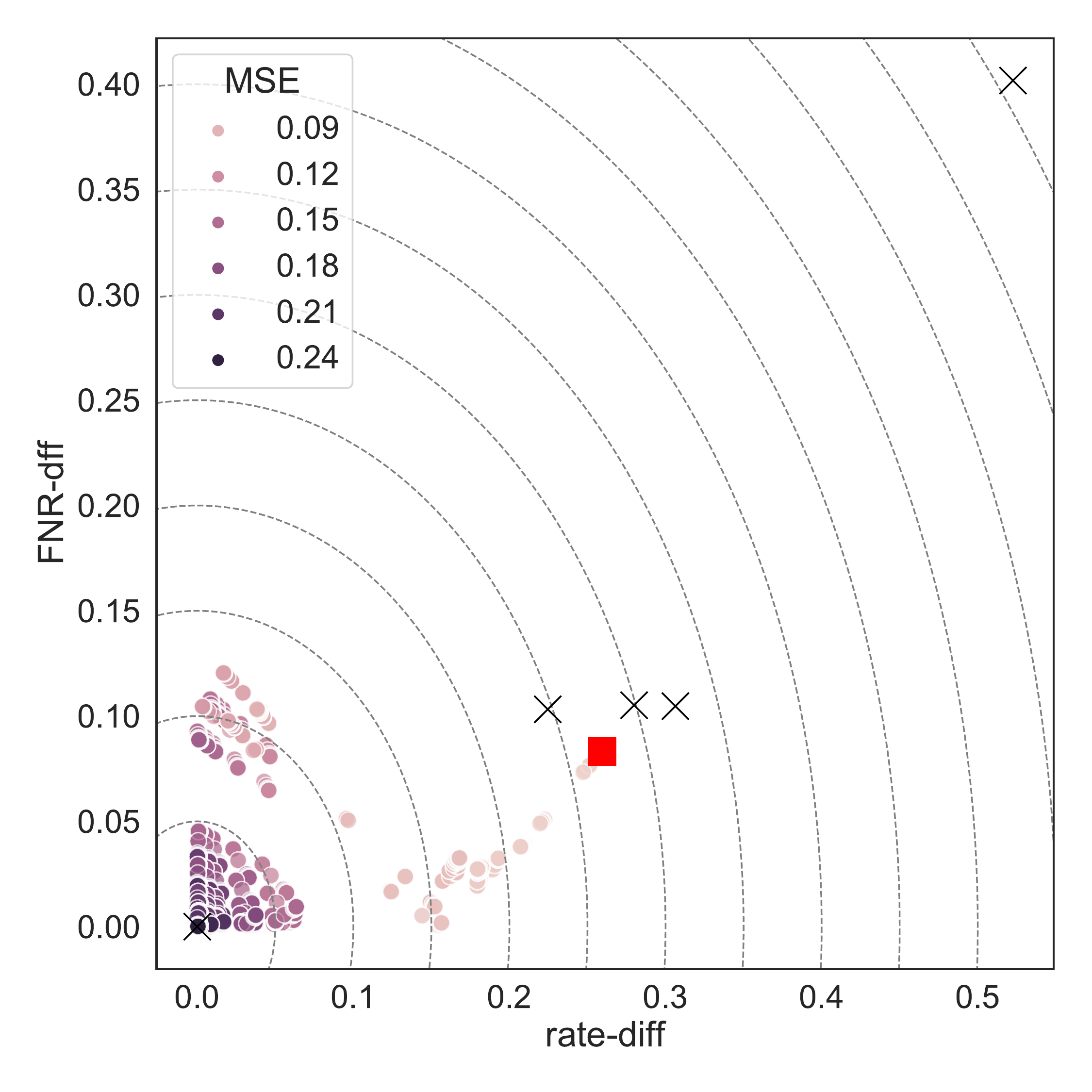}    
    \end{subfigure}
    \vspace{-0.5em}
    
    \begin{subfigure}{0.49\textwidth}
        \includegraphics[width=\linewidth]{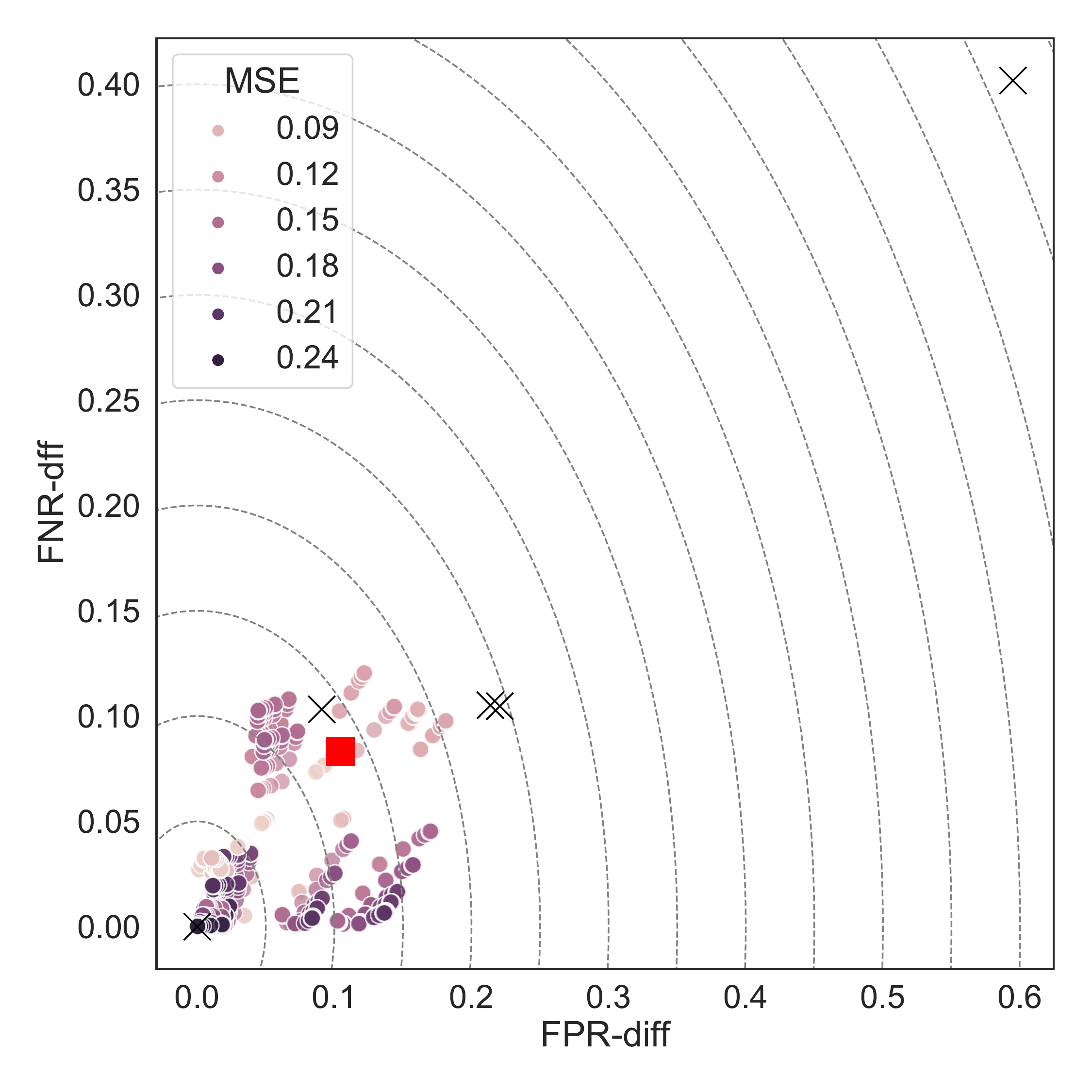}    
    \end{subfigure}
    \caption{Pairs of disparity values and MSE values for each of 1331 predictors in the simulated data. Black ``X''s represent the base predictors, with the mean predictor at the origin in each panel. The red square is the OLS predictor, and the dots are the FADE predictors. Radius lines indicate distance from the origin. Each pair of disparities can be jointly decreased with minimal increase in MSE relative to the OLS predictor.}
    \label{f:sim1_two_constraints_fairness}
\end{figure}

\begin{table}[ht]
    \centering
    \begin{tabular}{lrrrrr}
    \toprule
    Predictor  & MSE &  AUC &  rate-diff &  FPR-diff &  FNR-diff \\
    \midrule
    MSE (OLS) &       0.07 &       0.98 &       0.26 &       0.10 &       0.08 \\
    MSE + rate                       &       0.09 &       0.95 &       0.04 &       0.16 &       0.10 \\
    MSE + FPR                        &       0.07 &       0.98 &       0.19 &       0.00 &       0.03 \\
    MSE + FNR                        &       0.07 &       0.98 &       0.18 &       0.01 &       0.02 \\
    MSE + rate + FPR            &       0.12 &       0.90 &       0.04 &       0.05 &       0.09 \\
    MSE + rate + FNR            &       0.10 &       0.95 &       0.04 &       0.16 &       0.08 \\
    MSE + FPR + FNR             &       0.07 &       0.98 &       0.18 &       0.01 &       0.02 \\
    MSE + rate + FPR + FNR &       0.13 &       0.96 &       0.06 &       0.03 &       0.02 \\
    \bottomrule
    \end{tabular}
    \caption{Performance of the FADE predictors that minimize the Euclidean norm of MSE and zero or more disparities, in the simulated data. The top row represents the OLS predictor, included again for reference. All three disparities can be minimized, singly or jointly, with no impact or a small impact on MSE.}
    \label{t:sim1_lowest_distance}
\end{table}

\section{Results: recidivism risk prediction} \label{sec:compas}
We next illustrate FADE on the COMPAS dataset gathered by ProPublica \citep{Angwin2016, Angwin2016a}. The dataset comprises public arrest records, criminal records, and COMPAS scores from a single county in Florida, spanning 2013--2016. COMPAS is a collection of tools developed by the company Equivant (formerly Northpointe) designed to assess the risk of recidivism. We utilize the COMPAS scores for general, as opposed to violent, recidivism. The scores consist of risk deciles, coded 1-10, which we normalize to the range $[0.1, 1]$. COMPAS takes as input up to 137 features \citep{northpointe_practitioners_2015, rudin_age_2020}, which are unavailable in this dataest. We utilize just three features as covariates: an indicator for defendant age greater than 45, an indicator for defendant age less than 25, and the number of prior arrests, ranging from 0 to 29. Previous work has found that predictors trained using just these covariates perform similarly to COMPAS \citep{angelino_learning_2018}.

The sensitive feature is race, restricted to defendants who are coded African-American ($n = 3175$) or Caucasian ($n = 2013$). The decision variable $D$ represent pretrial release, with $D = 0$ if defendants are released and $D = 1$ if they are detained. The outcome of interest $Y^0$ is rearrest within two years, should a defendant be released pretrial. Since it difficult to assess the plausibility of the positivity and ignorability assumptions without consulting with domain experts, we conducted analyses in both the counterfactual and observable setting. The results and conclusions were largely the same, so we only include the counterfactual results.

We split the data into five datasets, each with approximately 1040 rows: $\datalearn$, $\datatrain^\text{nuis}$, $\datatrain^\text{target}$, $\datatest^\text{nuis}$, and $\datatest^\text{target}$. As base predictors, we used the four model types from the previous section as well as a logistic regression. We used random forest classifiers for the nuisance predictors in both the training and test data. Table \ref{t:compas_base_performance} gives the estimated performance of the five base predictors, COMPAS, and the OLS predictor, which spans COMPAS and the base predictors. Previous work found differences in a binarized version of COMPAS for both observable \citep{Angwin2016} and counterfactual \citep{mishler_modeling_2019} false positive vs false negative rates for African-American vs. Caucasian defendants. Those differences appear here in the generalized error rates. COMPAS also has a large rate disparity. Perhaps surprisingly, the base predictors all yield smaller disparities than COMPAS, even though they generally also have smaller MSE.

We compute FADE predictors using the same sets of penalty vectors $\Lambda$ as in the previous section. Figure \ref{f:compas_deciles_two_constraints} shows disparities and MSE values for all 1331 predictors. Most predictors fall within a narrow range of MSE values that also includes COMPAS, so the primary value of aggregation here is in reducing disparities. Nearly all the FADE predictors improve on COMPAS in terms of both risk and fairness. The top row of Figure \ref{f:compas_deciles_two_constraints} shows that all three disparities can be individually reduced to 0 with minimal cost in MSE relative to the OLS predictor, and the bottom row shows that these improvements also extend over pairs of disparities.

\begin{table}[ht]
    \centering
    \begin{tabular}{lrrrr}
    \toprule
    {} &  MSE &  rate-diff &  FPR-diff &  FNR-diff \\    \midrule
          Mean &       0.26 &       0.00 &       0.00 &       0.00 \\
 Random Forest &       0.28 &       0.05 &       0.03 &       0.02 \\
      Logistic &       0.22 &       0.06 &       0.06 &       0.00 \\
 GB Classifier &       0.23 &       0.06 &       0.01 &       0.05 \\
         Ridge &       0.22 &       0.05 &       0.05 &       0.00 \\
        COMPAS &       0.24 &       0.15 &       0.15 &       0.08 \\
        \midrule
 OLS &       0.22 &       0.09 &       0.09 &       0.05 \\
    \bottomrule
    \end{tabular}
    \caption{Estimated performance of the five base predictors, COMPAS, and the OLS predictor in the COMPAS dataset. The OLS weights are $[0.39,  0.12,  0.79,  0.10, -1.08, 0.93]$. The OLS predictor does not perform substantially better than the base predictors. COMPAS has different false positive and false negative rates for African-American vs Caucasian defendants, as well as a rate disparity. The base predictors all have smaller disparities than COMPAS, and--perhaps surprisingly--generally smaller MSE.}
    \label{t:compas_base_performance}
\end{table}

\begin{figure}
    \centering
    \begin{subfigure}{0.32\textwidth}
        \includegraphics[width=\linewidth]{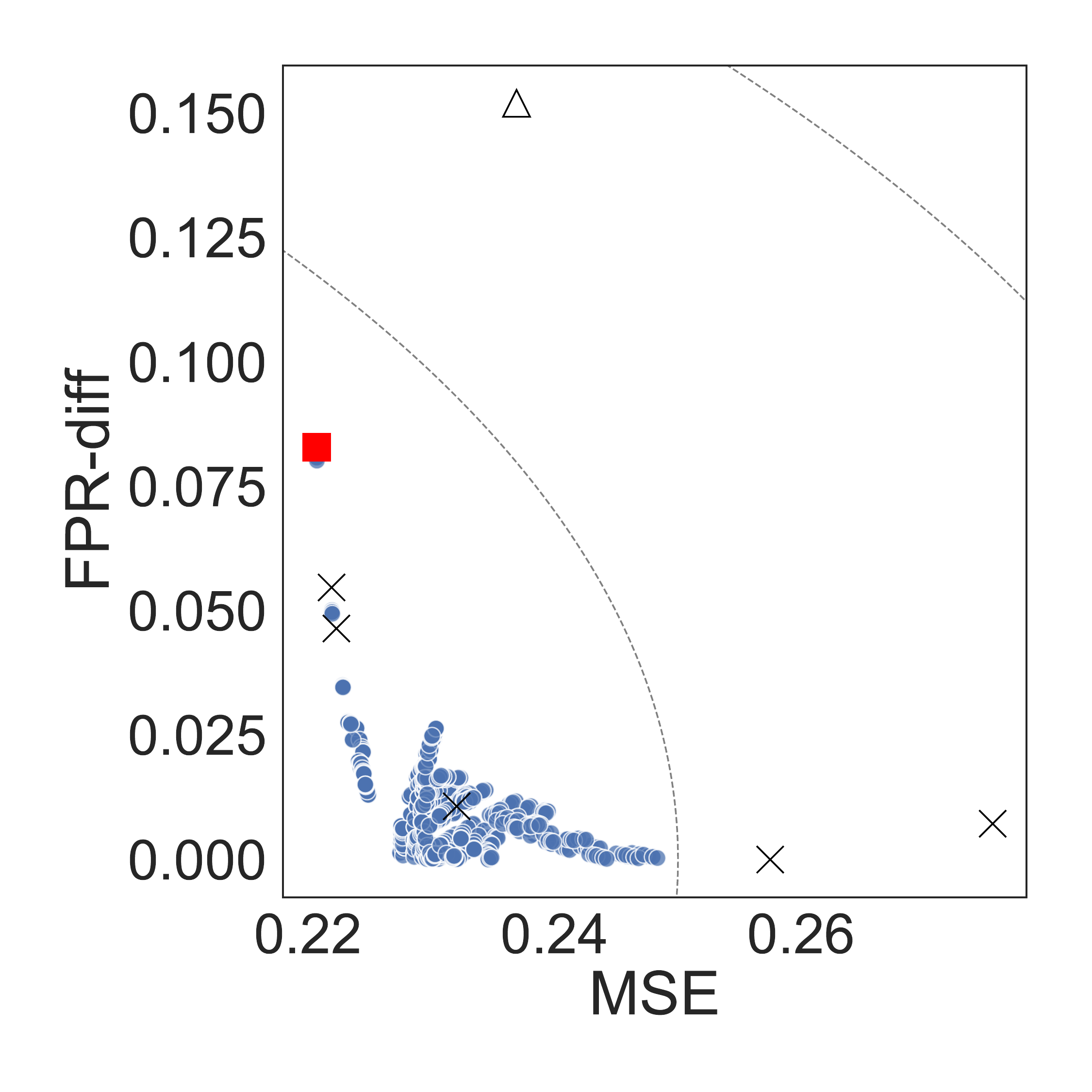}    
    \end{subfigure}
    \begin{subfigure}{0.32\textwidth}
        \includegraphics[width=\linewidth]{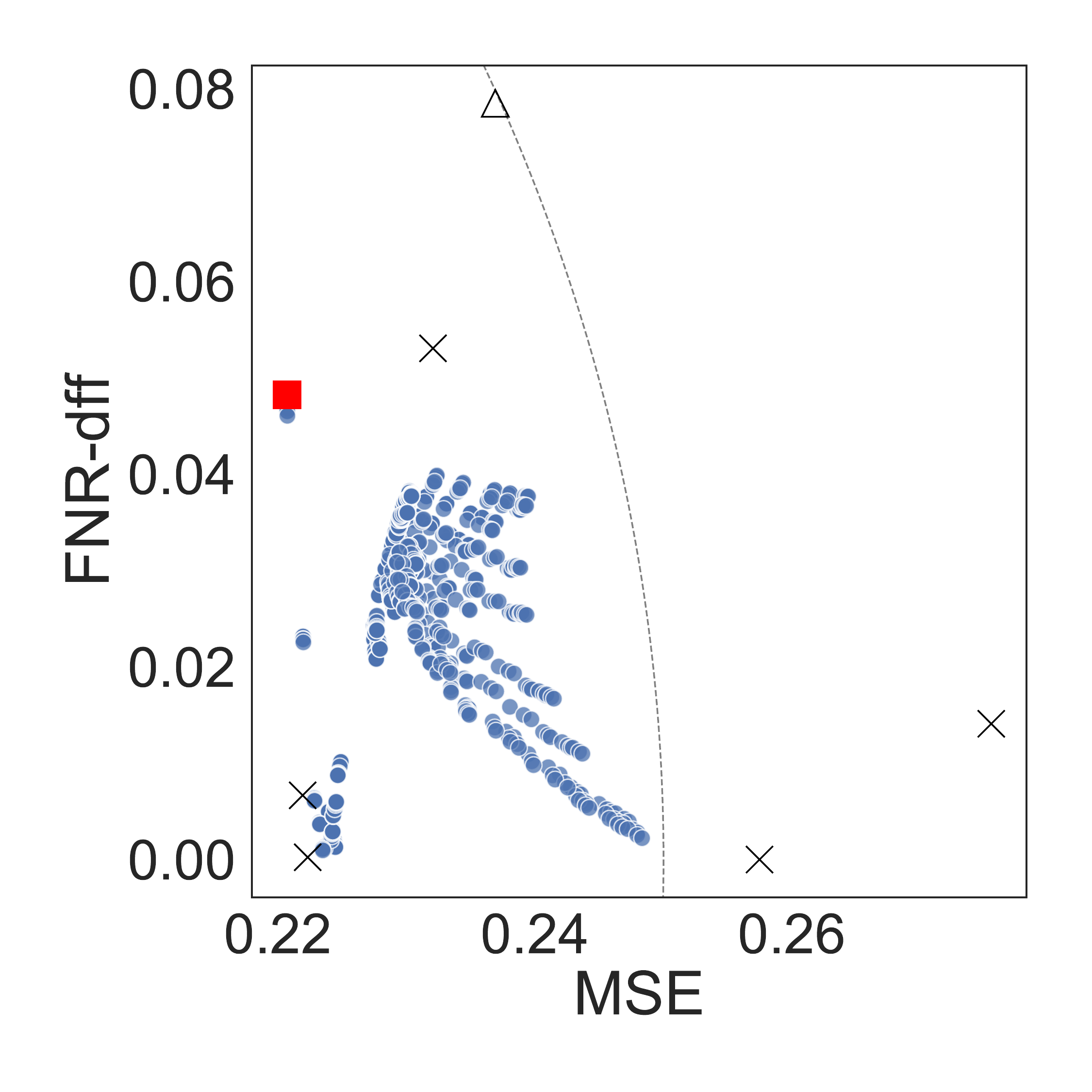}    
    \end{subfigure}
    \begin{subfigure}{0.32\textwidth}
        \includegraphics[width=\linewidth]{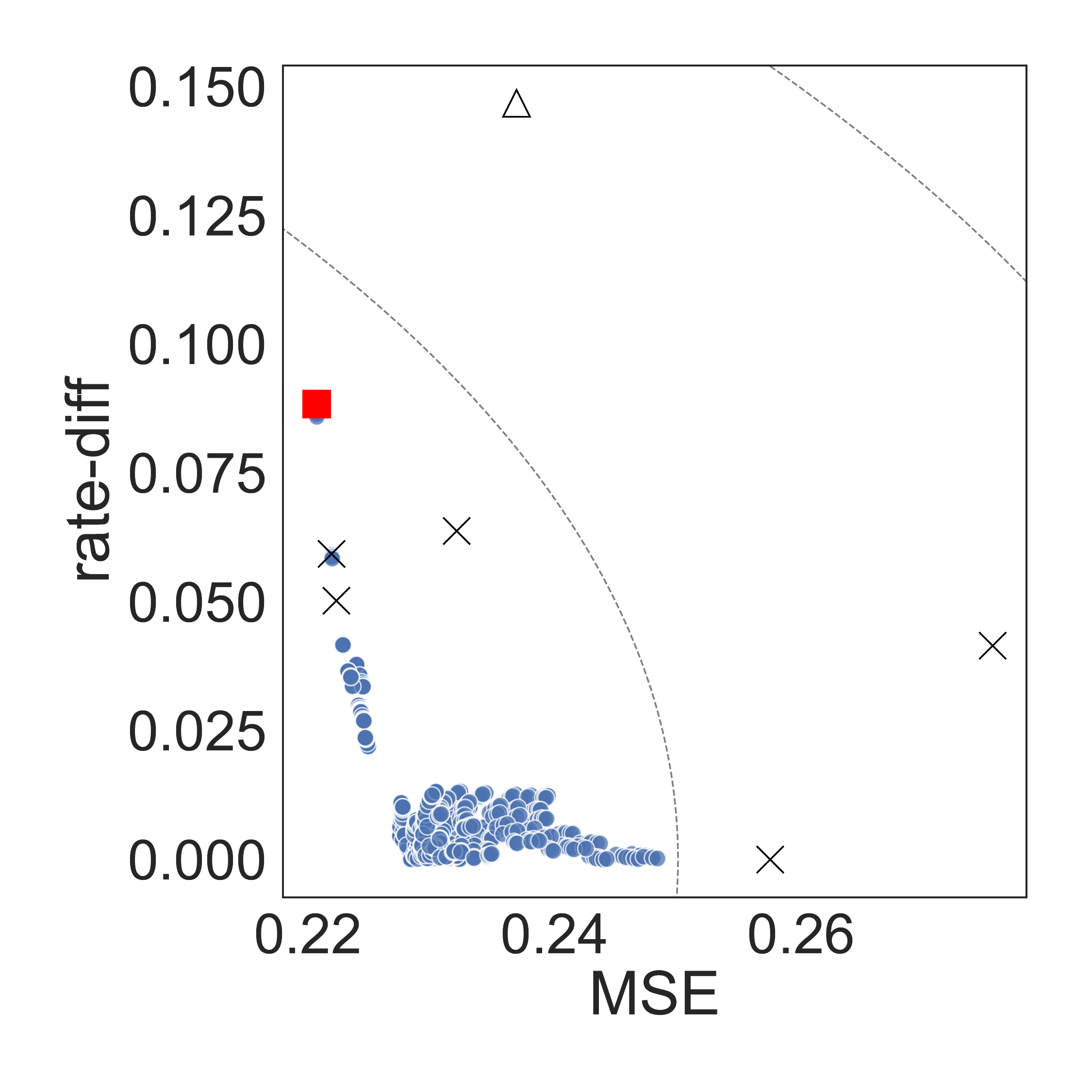}    
    \end{subfigure}
    \begin{subfigure}{0.32\textwidth}
        \includegraphics[width=\linewidth]{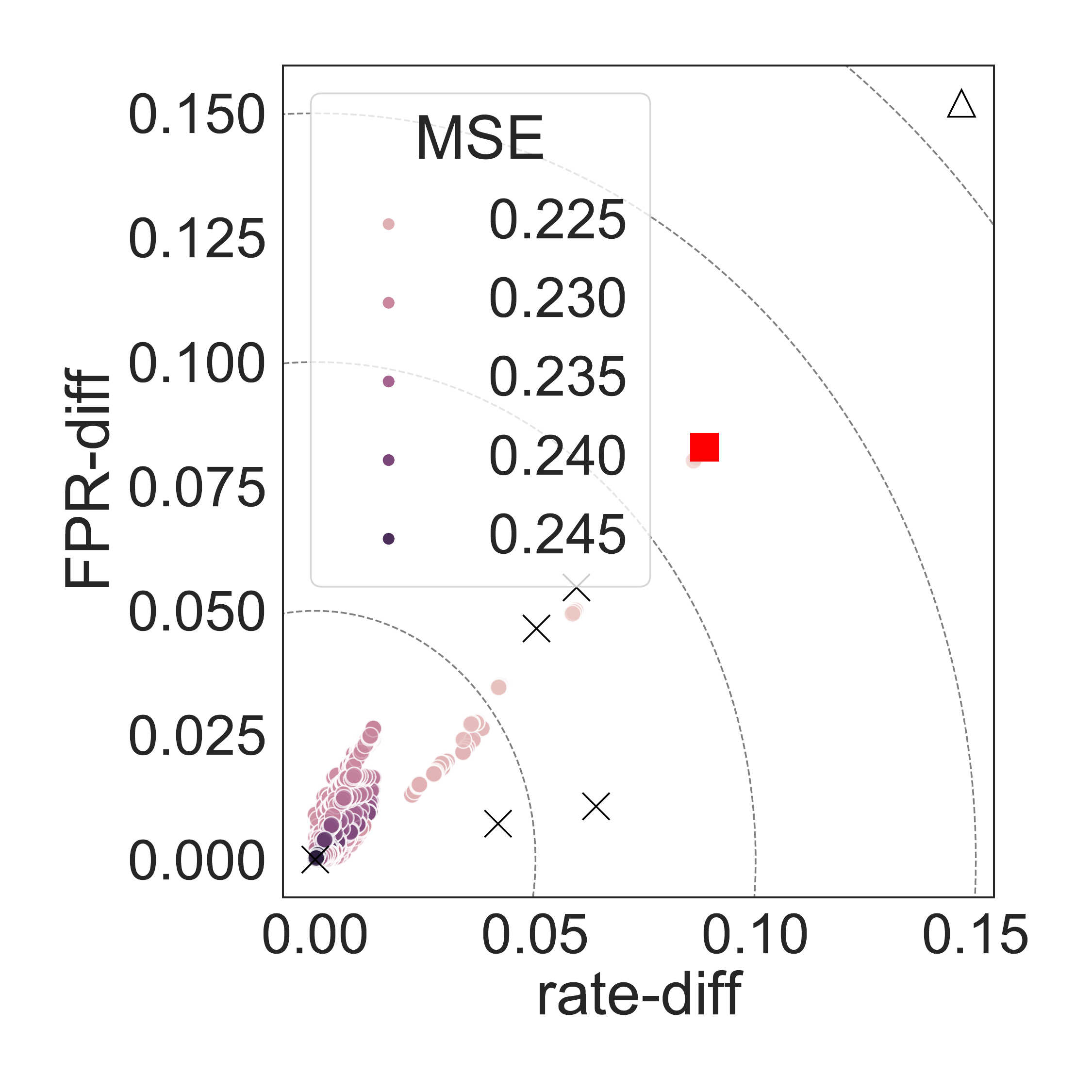}    
    \end{subfigure}
    \begin{subfigure}{0.32\textwidth}
        \includegraphics[width=\linewidth]{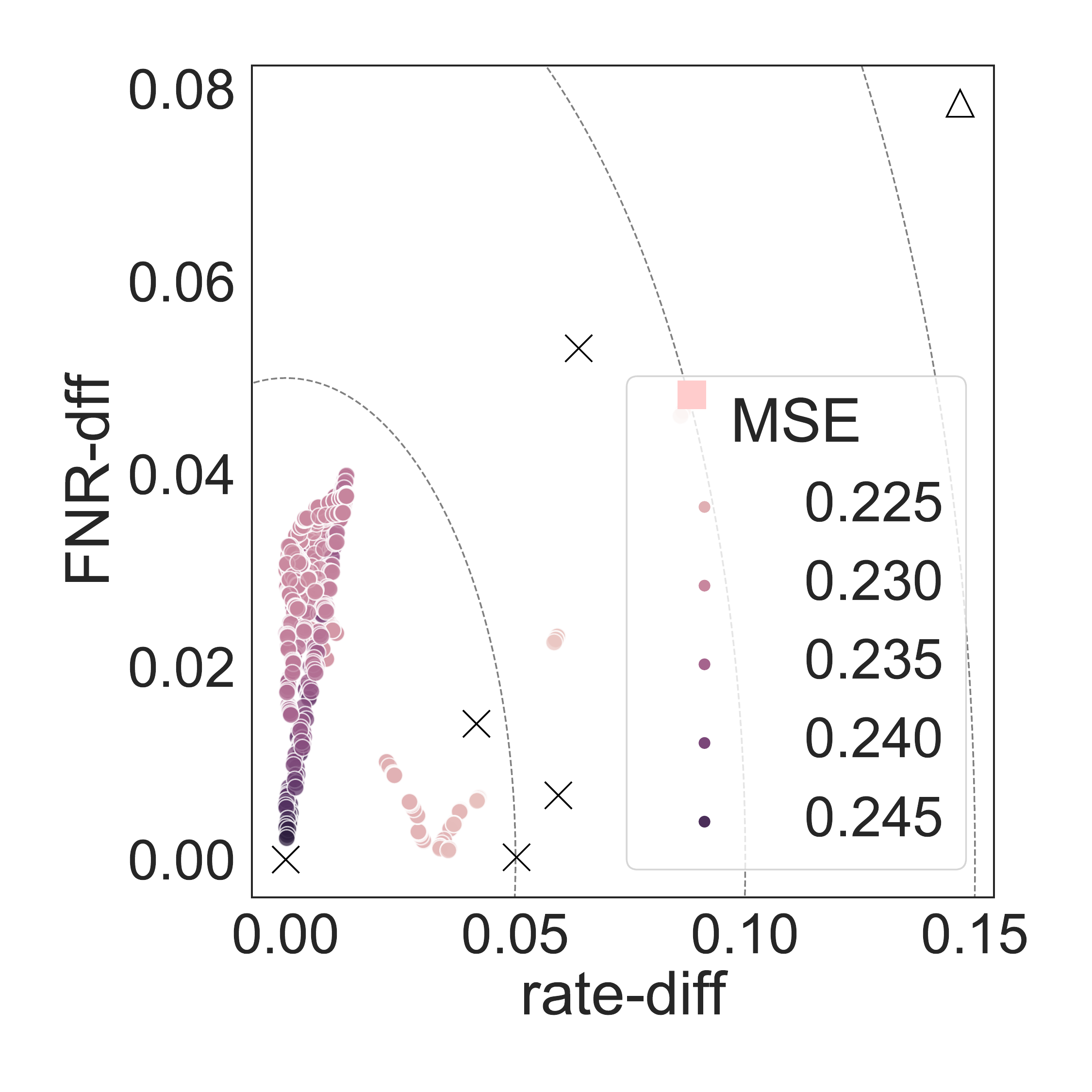}    
    \end{subfigure}
    \begin{subfigure}{0.32\textwidth}
        \includegraphics[width=\linewidth]{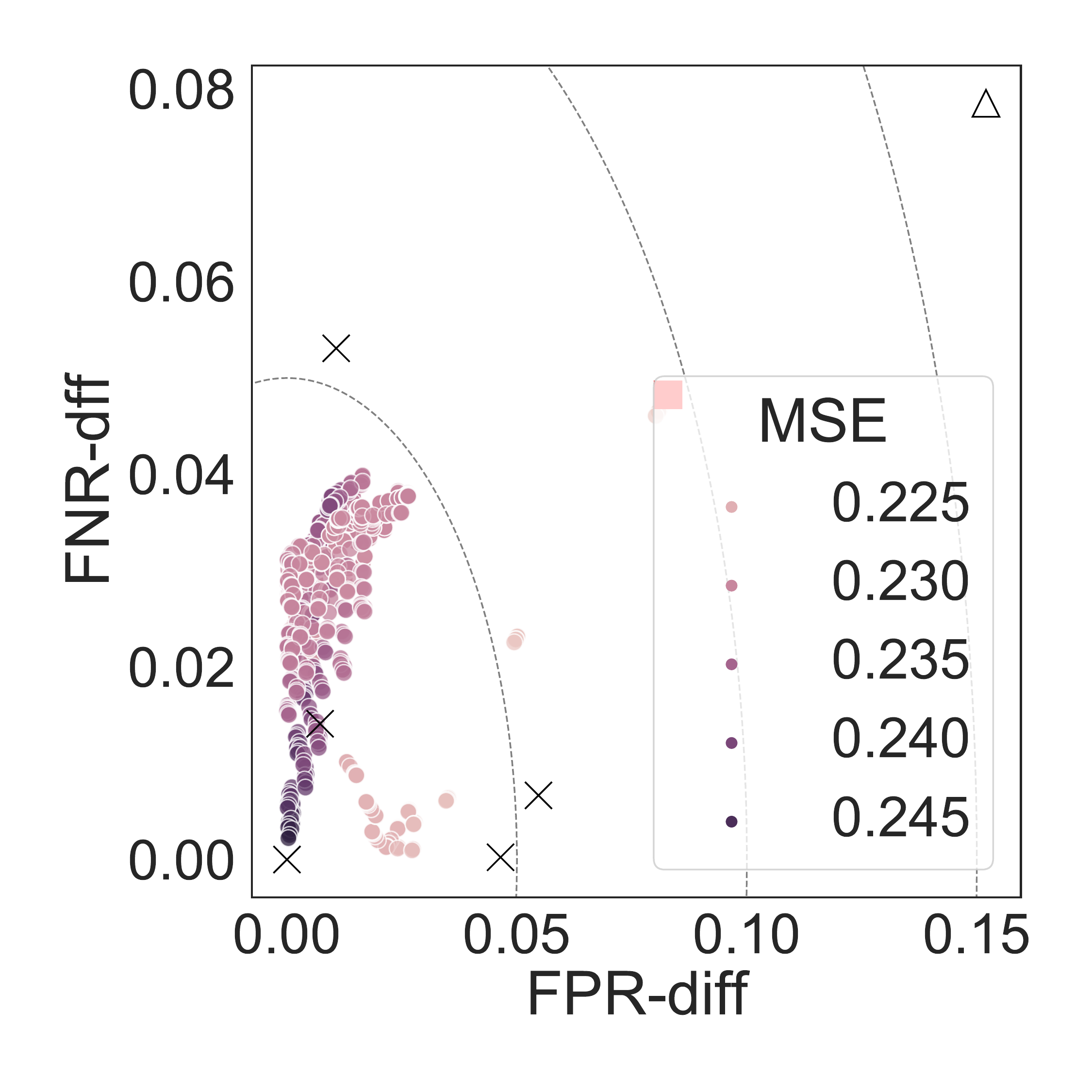}    
    \end{subfigure}
    \caption{Disparity against MSE (top row) or pairs of disparities colored by MSE (bottom row), for each of 1331 predictors in the COMPAS data. The black triangle represents COMPAS, the black ``X''s represent the other base predictors, the red square is the OLS predictor, and the blue dots are the FADE predictors. Radius lines indicate distance from the origin. Most of the predictors improve on COMPAS in terms of both MSE and the relevant disparity. For each disparity, many predictors exist which take that disparity to 0, at a small cost in MSE relative to the OLS predictor.}
    \label{f:compas_deciles_two_constraints}
\end{figure}

\section{Results: income prediction} \label{sec:adult}
Finally, we apply our method in the observable setting, using the Adult dataset \citep{Dua:2019}. This dataset comprises demographic variables derived from the 1994 U.S. Census. We consider sex as a sensitive feature, coded 0 or 1, and we utilize as covariates a set of indicator variables representing age by decade, and a set of indicator variables representing the number of years of education. The classification task is to predict whether an individual's income is over \$50K/year, for example for the purpose of deciding whether to issue a loan.

We randomly split the data into four datasets: $\datalearn$ and $\datatrain$, consisting of 14,653 and 14,652 rows; and $\datatest$ and $\datavalidate$, consisting of 9,768 and 9,769 rows. $\datavalidate$ is used to compare the performance of the selected best predictors to predictors that are developed using existing fairness methods.

This dataset does not contain any previously trained predictors. We use the same five base predictor types as in the COMPAS analysis. Additionally, we use $\datatrain$ to train three ``fair'' predictors with other fairness methods: \emph{adversarial debiasing} \citep{zhang_mitigating_2018}, \emph{reductions} \citep{agarwal_reductions_18a}, and a \emph{meta-algorithm} \citep{celis_classification_2020}. All predictors were trained using the Python library aif360, a set of tools that provide access to a range of fairness methods via a consistent interface \citep{bellamy_ai_2018}. The three chosen methods yield binary classifiers, and they are all designed to minimize \emph{rate-diff} in an observable setting with a binary outcome. Since, for a binary classifier $\widehat{Y}$, MSE is equal to classification error $\Pb(\widehat{Y} \neq Y)$, we may regard these three predictors as the result of methods that seek to minimize MSE among specific classes of binary predictors.

We construct FADE predictors using the five base predictors (``base5'') or the five base predictors and the three fair predictors (``base8''). Table \ref{t:adult_base_performance} gives the performance of the base predictors, the fair predictors, and the two OLS predictors. Compared to the base predictors, two of the three fairness methods result in a substantially lower rate-diff, which is the disparity they aim to minimize. The base predictors and the OLS predictors have lower MSE, higher AUC, and higher disparities compared to these two fair predictors.

\begin{table}[ht]
    \centering
    \begin{tabular}{lrrrrr}
    \toprule
             Predictor &  MSE &  AUC &  rate-diff &  FPR-diff &  FNR-diff \\
    \midrule
                Mean & 0.18 & 0.50 &       0.00 &      0.00 &      0.00 \\
       Random Forest & 0.14 & 0.81 &       0.19 &      0.14 &      0.24 \\
            Logistic & 0.14 & 0.81 &       0.20 &      0.14 &      0.25 \\
       GB Classifier & 0.14 & 0.82 &       0.19 &      0.13 &      0.24 \\
               Ridge & 0.14 & 0.81 &       0.17 &      0.13 &      0.16 \\
               \midrule
         Adversarial & 0.21 & 0.67 &       0.07 &      0.00 &      0.03 \\
          Reductions & 0.22 & 0.62 &       0.01 &      0.03 &      0.07 \\
                Meta & 0.30 & 0.68 &       0.18 &      0.26 &      0.26 \\
                \midrule
   OLS - base5 & 0.14 & 0.82 &       0.19 &      0.13 &      0.24 \\
 OLS - base8 & 0.14 & 0.81 &       0.20 &      0.14 &      0.25 \\
    \bottomrule
    \end{tabular}
    \caption{Estimated performance in the Adult dataset of five base predictors, three predictors trained using existing fairness methods, and the two OLS predictors, which aggregate only the five base predictors or all eight predictors. The OLS weights are $[0.03, 0.29, 0.65, 0.03, 0.01]$, for base5, and $[-0.01,  0.26,  0.56,  0.18, 0.04, -0.02, -0.03, 0]$, for base8. Only two of the three fairness methods successfully control their targeted disparity, rate-diff. The Meta predictor has a rate-diff which is comparable to the base predictors which are trained without regard to fairness. The OLS predictors perform comparably to the base predictors.}
    \label{t:adult_base_performance}
\end{table}

We compute FADE predictors using the same sets of penalty vectors $\Lambda$ as in the previous two sections. Figure \ref{f:compas_deciles_two_constraints} shows disparities and MSE values for all 1331 predictors. Fairness-accuracy tradeoffs are most evident for rate-diff and FPR-diff in the top row of this figure, where only the five base predictors are used. In the bottom row, where the fair predictors are included as basis functions, the tradeoffs essentially disappear: all three disparities can be reduced to 0 with virtually no cost in MSE.

\begin{figure}
    \centering
    \begin{subfigure}{0.31\textwidth}
        \includegraphics[width=\linewidth]{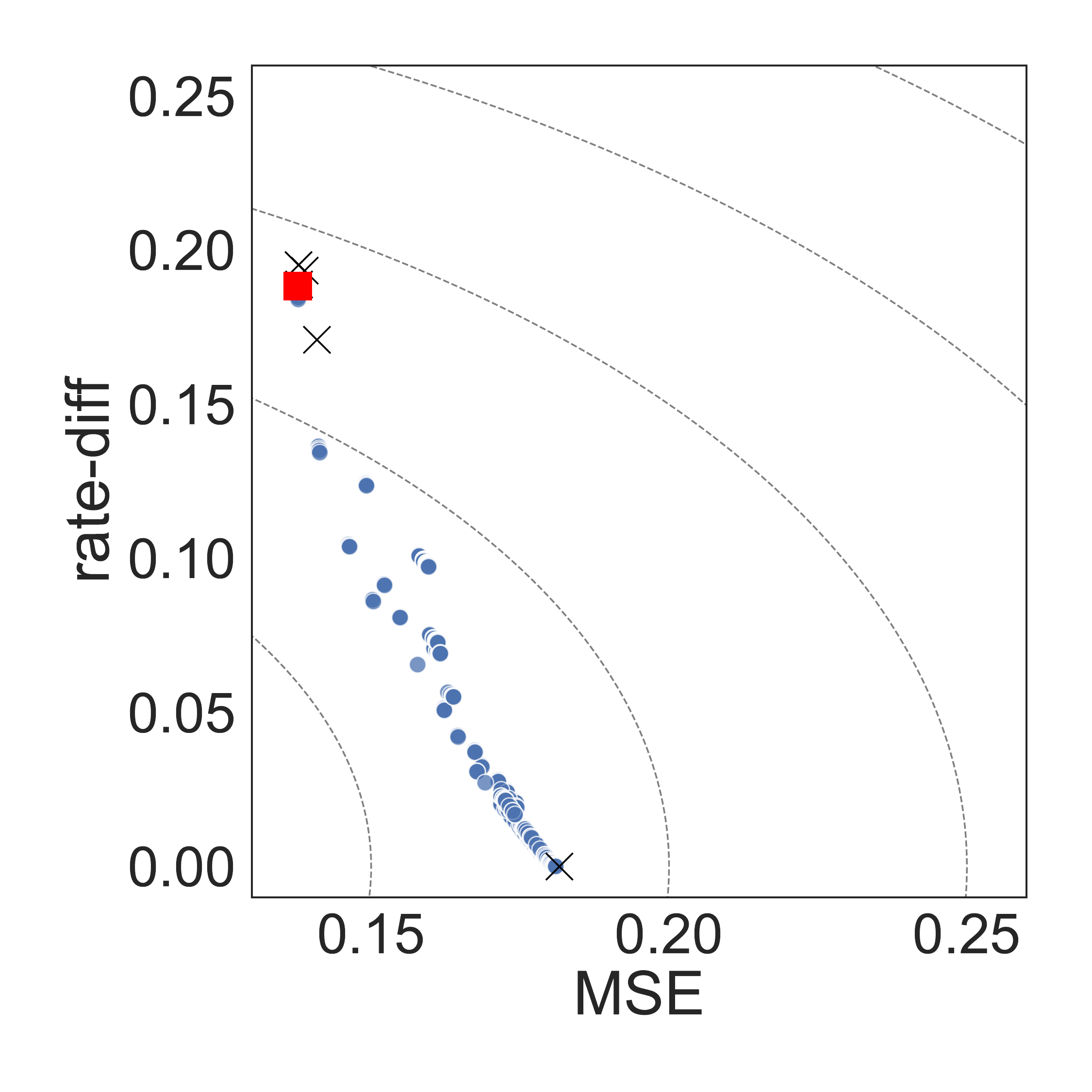}    
    \end{subfigure}
    \begin{subfigure}{0.31\textwidth}
        \includegraphics[width=\linewidth]{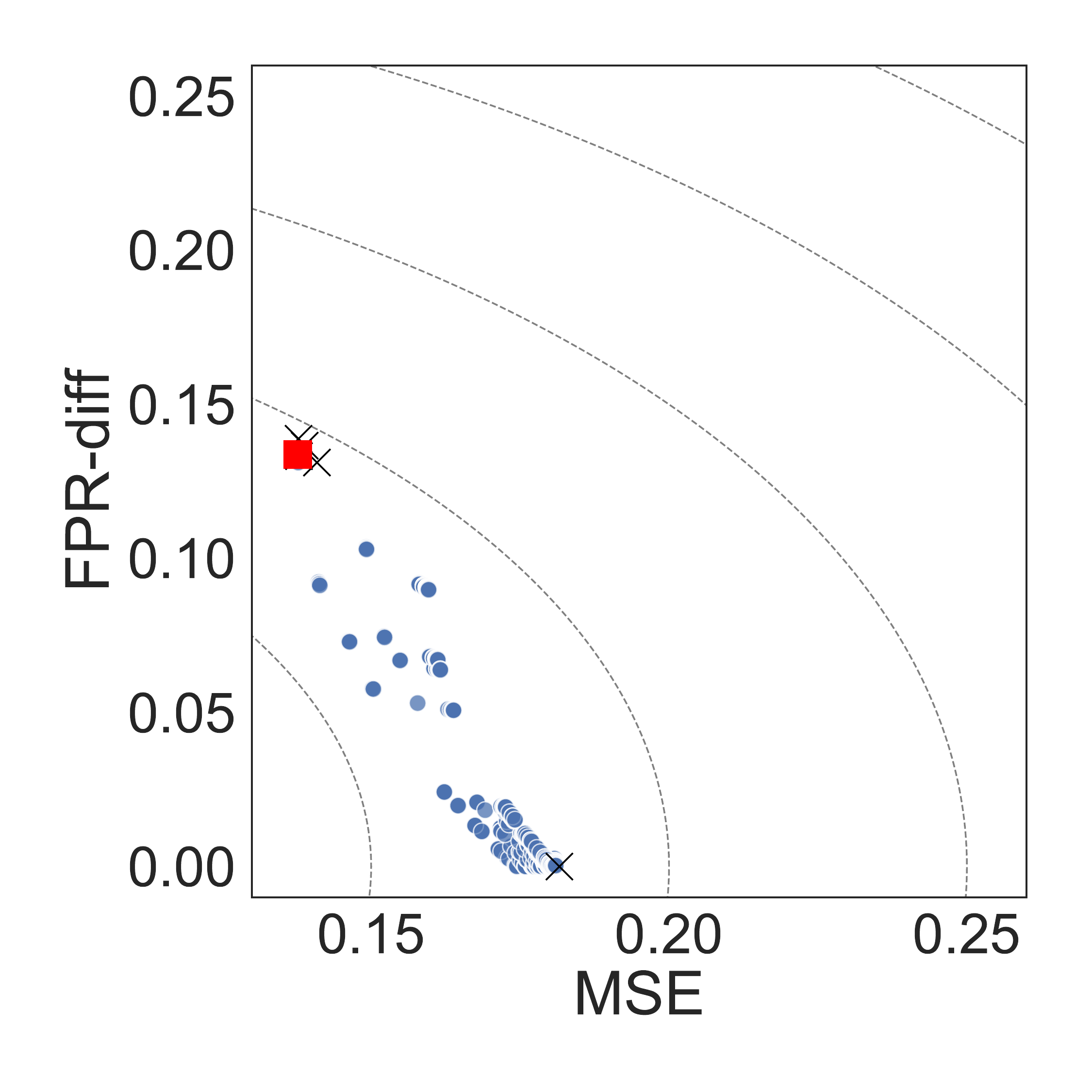}    
    \end{subfigure}
    \begin{subfigure}{0.31\textwidth}
        \includegraphics[width=\linewidth]{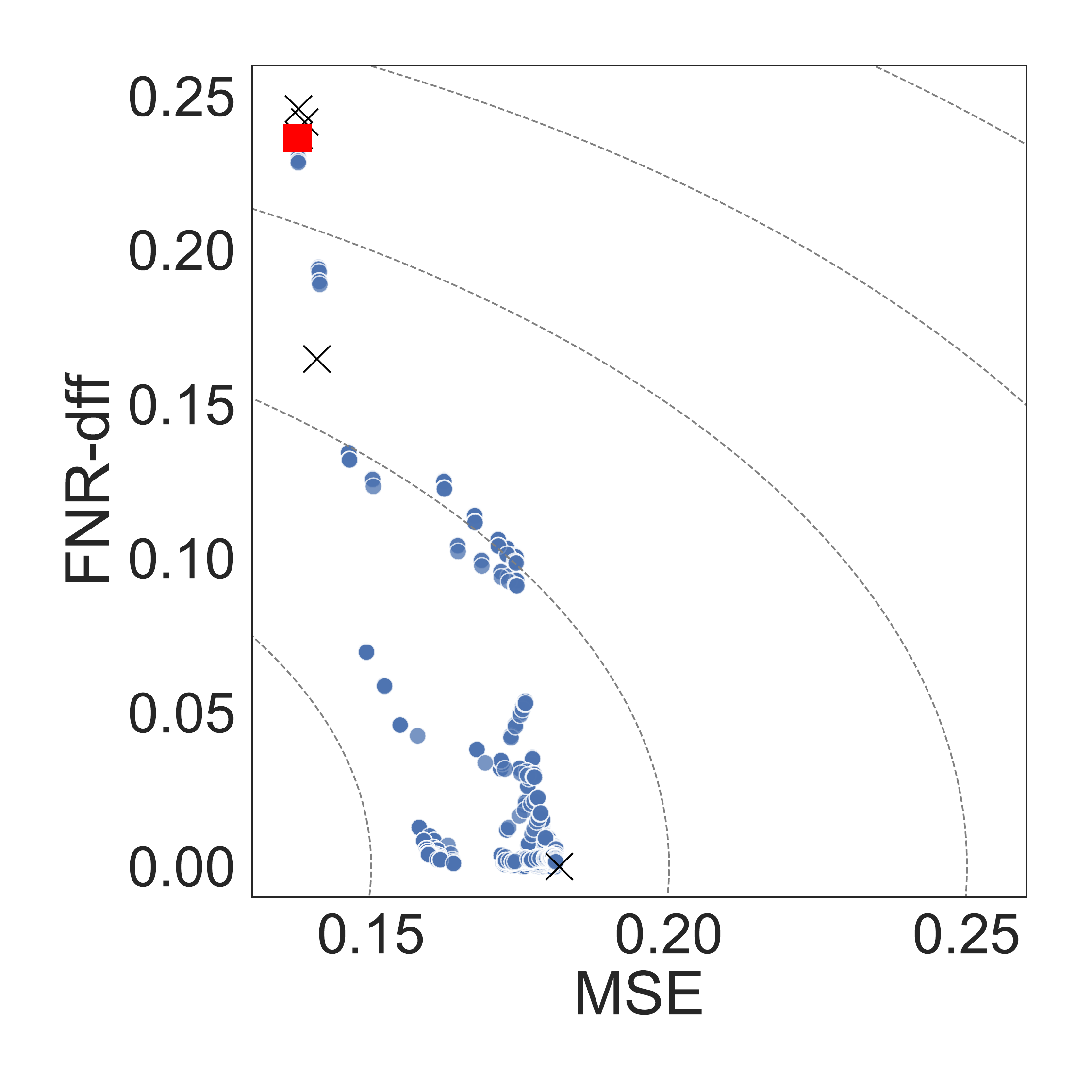}    
    \end{subfigure}
    \begin{subfigure}{0.31\textwidth}
        \includegraphics[width=\linewidth]{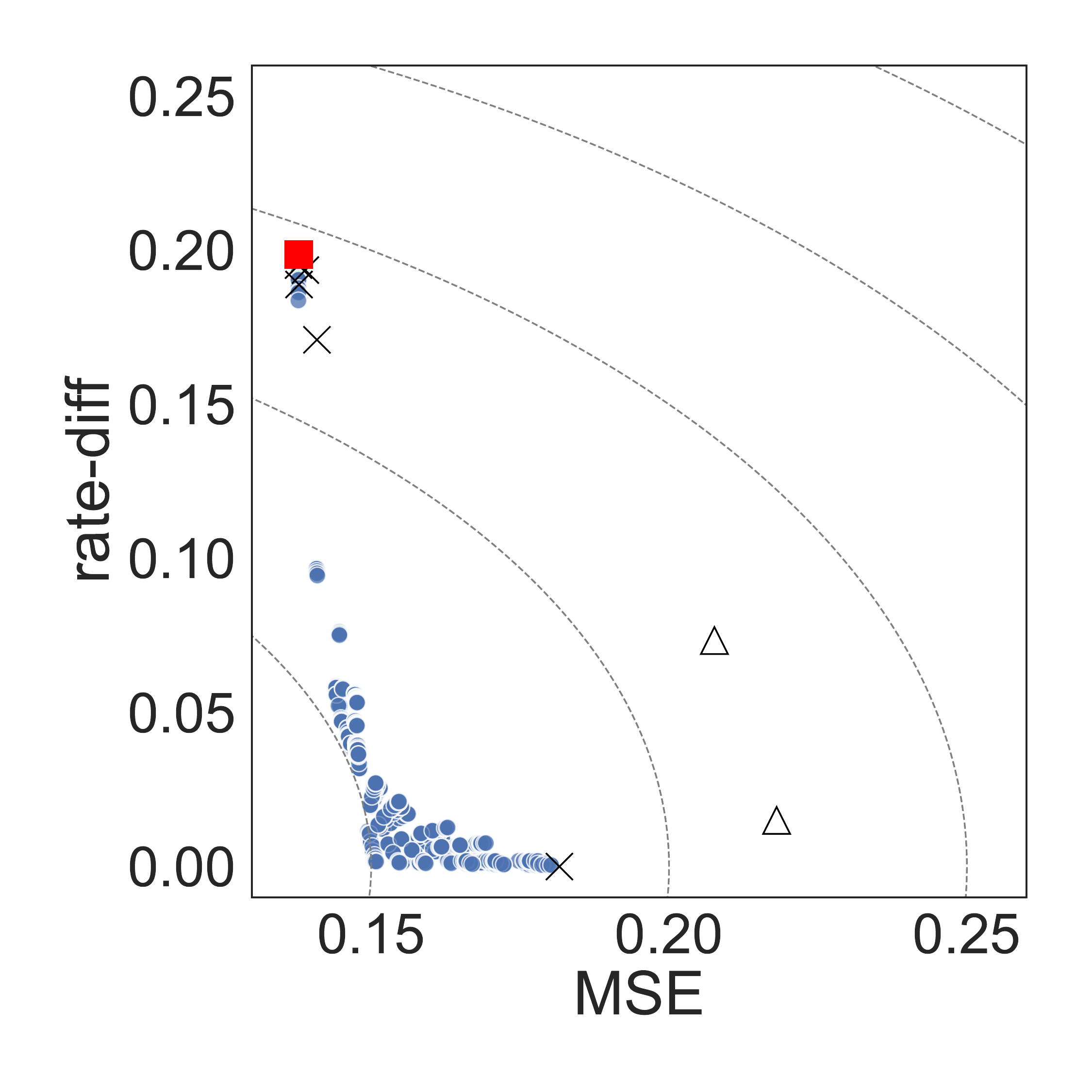}    
    \end{subfigure}
    \begin{subfigure}{0.31\textwidth}
        \includegraphics[width=\linewidth]{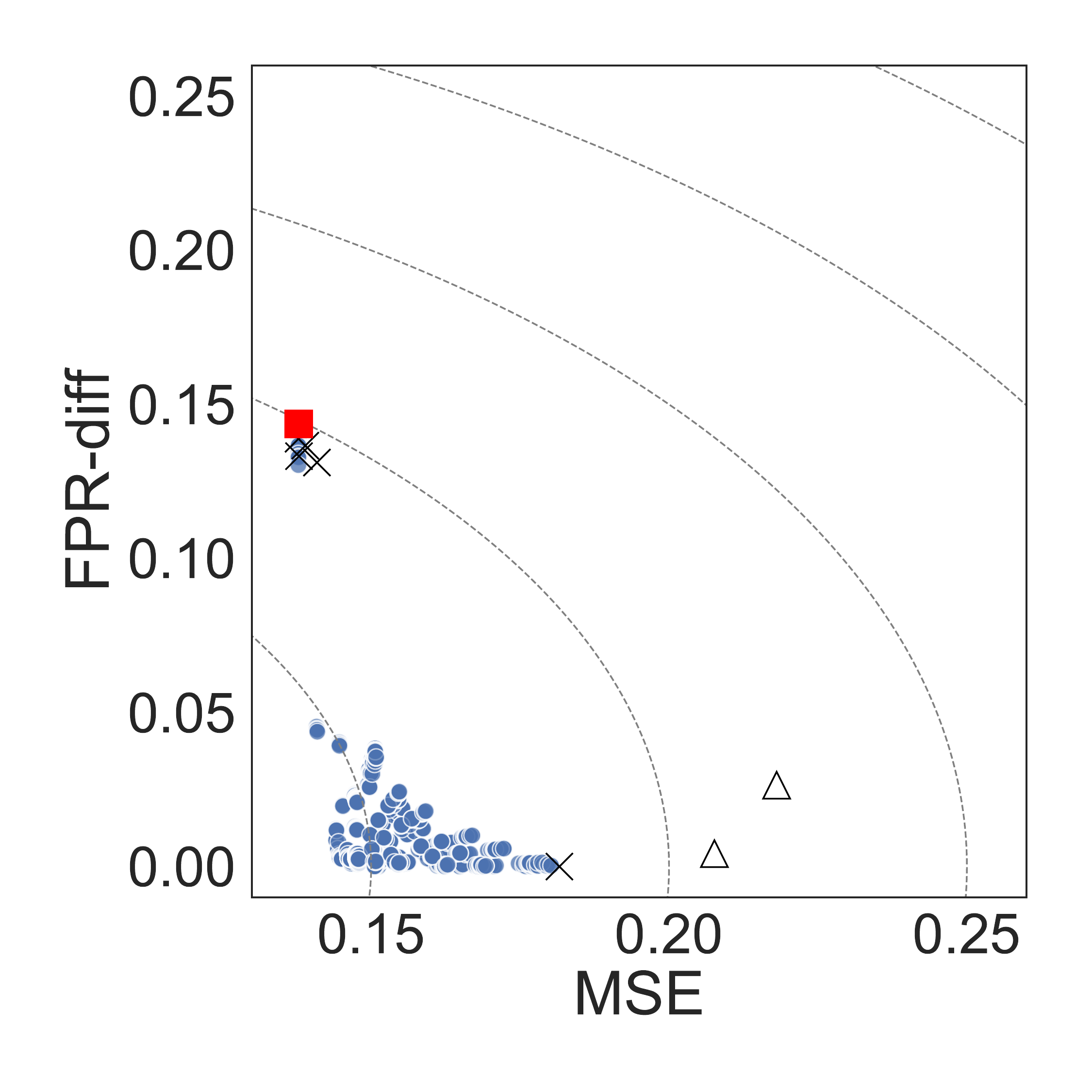}    
    \end{subfigure}
    \begin{subfigure}{0.31\textwidth}
        \includegraphics[width=\linewidth]{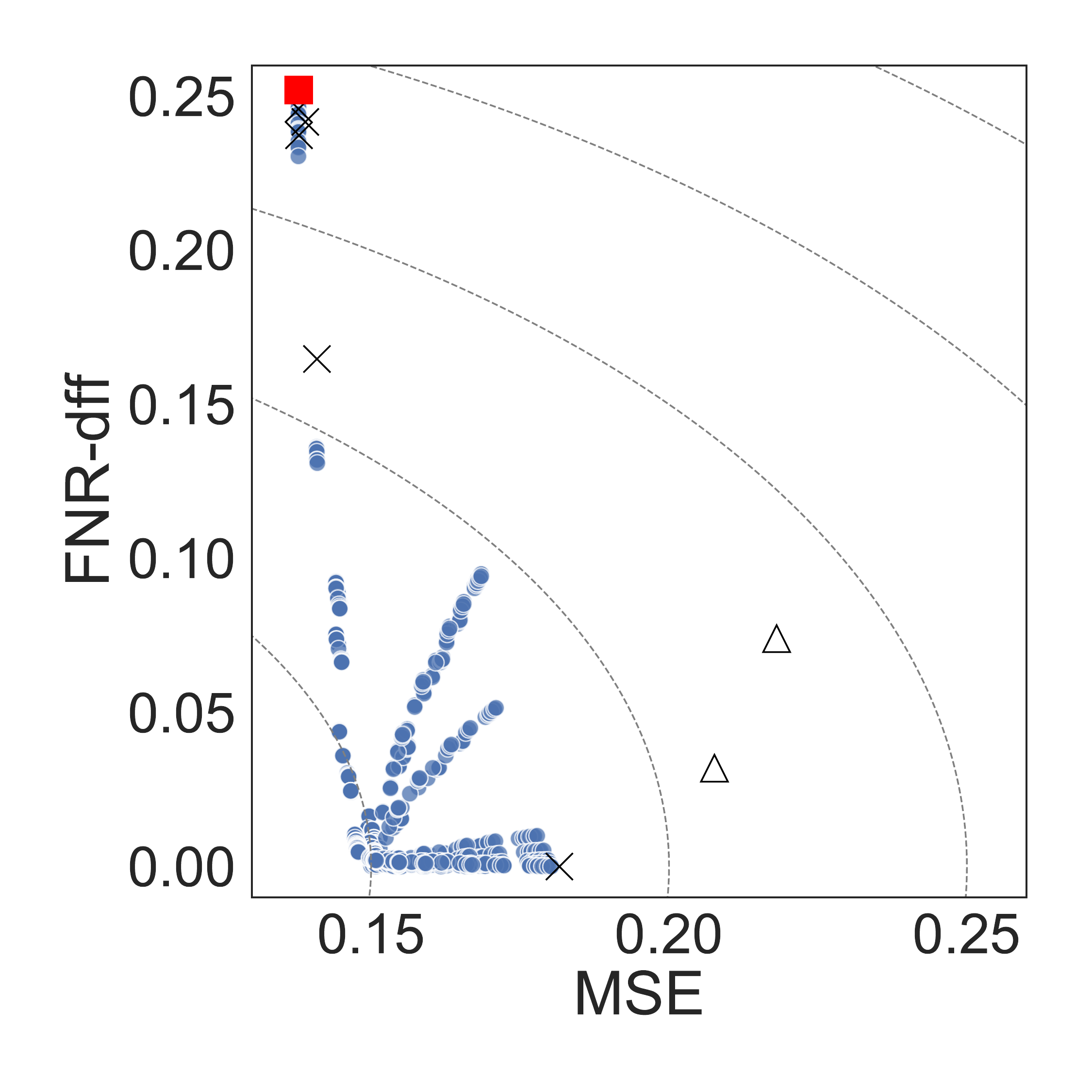}    
    \end{subfigure}
    \caption{Disparity against MSE for FADE predictors based on the five base predictors (top row) or the five base predictors plus the three fair predictors, for each of 1331 FADE predictors in the Adult data. Black ``X''s represent the base predictors, the red square is the OLS predictor, and the blue dots are the FADE predictors. Radius lines indicate distance from the origin. For the sake of legibility, the Meta predictor, which has an MSE of 0.30, is excluded. The top row exhibits small but clear fairness-accuracy tradeoffs for rate-diff and FPR-diff. The bottom row shows that with the inclusion of the fair predictors, each disparity can be taken to 0 with almost no cost in MSE relative to the OLS model.}
    \label{f:adult_binary_two_constraints}
\end{figure}

\subsection{Model validation} \label{subsec:model_validation}
Using the performance estimates from the test data, we select the FADE predictors that minimize the distance from the origin in each of seven fairness-accuracy subspaces, for both the base5 and base8 predictors. We then compute risk and fairness estimates for each of these predictors, as well as the three fair predictors, on the validation data (Table \ref{t:adult_validation}). The estimates on the test and validation data differed by no more than approximately 0.005.

Both the base5 and base8 FADE predictors are substantially more fair than the OLS predictors, while incurring very small increases in MSE. The high AUC values confirm that these are accurate predictors. All the FADE predictors have small disparities compared to the OLS predictors. Explicitly minimizing multiple disparities simultaneously is not necessarily more costly in terms of performance than minimizing a single disparity.

The FADE predictors have substantially lower MSE and higher AUC than the fair predictors, and they are in many cases more fair. The fair predictors achieve values of 0.08, 0.01, and 0.17 for rate-diff, the disparity they aim to minimize. The base5 FADE predictors that include rate-diff in their criteria achieve rate-diffs of 0.04, 0.04, 0.06, and 0.02. The corresponding base8 FADE predictors achieve rate-diffs of 0.01, 0.03, 0.01, and 0.03.

The base5 results show that FADE yields predictors that perform comparably to or better than existing fairness methods. The base8 results highlight the flexibility of our approach: multiple predictors can be aggregated, regardless of whether or not they are trained with fairness properties in mind, with different weights to target different disparities. In this case, including the fair predictors in the aggregation improves both accuracy and fairness.

Each of the fair prediction methods contains tuning parameters that can be adjusted to return different predictors, as well as settings that allow them to target different fairness constraints, such as equalized odds. However, each method can only target a single fairness constraint at once. Additionally, these methods can take substantial time to run. A single run of the Meta method took roughly 5 seconds, the Reductions method ran in approximately 15 seconds, and the Adversarial method, which relies on neural nets, took roughly a minute. By contrast, we were able to train the base predictors and compute and evaluate 1331 FADE predictors in 25 seconds.

The three fair predictors are binary by construction, whereas our method returns continuous predictors. Of course, these continuous predictors can be treated as binary, either by thresholding the output (yielding a deterministic classifier) or by treating the output as a probability and sampling from a corresponding Bernoulli distribution (yielding a randomized classifier). It is fast to compute estimates of the accuracy and fairness values from either of these two binarized classifiers and choose one which minimizes the criteria of interest. For example, the base5 FADE predictors include a predictor which, when thresholded at 0.5 to yield a deterministic binary classifier, achieves a classification error of 0.24 and disparities of 0 (to two digits). This has very slightly higher classification error than the Adversarial and Reductions predictors, but it exactly achieves equalized odds and demographic parity. 

\begin{table}[ht]
    \centering
    \begin{tabular}{llrrrrr}
    \toprule
    & Predictor &  MSE &  AUC &  rate-diff &  FPR-diff &  FNR-diff \\
    \midrule
    \multirow{7}{*}{base5}     & MSE (OLS-base5) & 0.14 & 0.82 &       0.19 &      0.13 &      0.24 \\
    & MSE + rate                       & 0.16 & 0.73 &       0.04 &      0.02 &      0.10 \\
    & MSE + FPR                        & 0.15 & 0.80 &       0.09 &      0.06 &      0.13 \\
    & MSE + FNR                        & 0.16 & 0.75 &       0.10 &      0.09 &      0.01 \\
    & MSE + rate + FPR            & 0.16 & 0.73 &       0.04 &      0.02 &      0.10 \\
    & MSE + rate + FNR            & 0.16 & 0.75 &       0.06 &      0.05 &      0.01 \\
    & MSE + FPR + FNR             & 0.16 & 0.75 &       0.06 &      0.05 &      0.01 \\
    & MSE + rate + FPR + FNR & 0.17 & 0.73 &       0.02 &      0.02 &      0.00 \\
    \midrule
    \multirow{7}{*}{base8} &  MSE (OLS-base8) & 0.14 & 0.81 &       0.20 &      0.14 &      0.25 \\
    & MSE + rate                       & 0.15 & 0.79 &       0.01 &      0.03 &      0.02 \\
    & MSE + FPR                        & 0.14 & 0.79 &       0.06 &      0.01 &      0.10 \\
    & MSE + FNR                        & 0.15 & 0.79 &       0.05 &      0.01 &      0.01 \\
    & MSE + rate + FPR            & 0.15 & 0.79 &       0.03 &      0.00 &      0.01 \\
    & MSE + rate + FNR            & 0.15 & 0.79 &       0.01 &      0.03 &      0.01 \\
    & MSE + FPR + FNR             & 0.15 & 0.79 &       0.04 &      0.00 &      0.01 \\
    & MSE + rate + FPR + FNR & 0.15 & 0.79 &       0.03 &      0.01 &      0.01 \\
    \midrule
    \multirow{3}{*}{fair} &  Adversarial & 0.21 & 0.67 &       0.08 &      0.01 &      0.04 \\
    & Reductions & 0.22 & 0.62 &       0.01 &      0.03 &      0.06 \\
    & Meta & 0.30 & 0.68 &       0.17 &      0.26 &      0.24 \\
    \bottomrule
    \end{tabular}
    \caption{Performance of the predictors that minimize the Euclidean norm of MSE and zero to three disparities, in the Adult data. The OLS predictors and the three fair predictors are included again for reference. Predictors are selected on the test data and evaluated on the validation data. The base5 FADE predictors aggregate the five base predictors, and the base8 FADE predictors aggregate all eight predictors. All three disparities can be minimized, singly or jointly, with only a small impact on MSE, and a small impact on AUC. Aggregated predictors are more accurate than the fair predictors, and have comparable or smaller values of rate-diff, the disparity that the fair predictors aim to minimize.}
    \label{t:adult_validation}
\end{table}

\section{Conclusion} \label{sec:conclusion}
We developed a framework, FAir Double Ensemble learning (FADE), for constructing fair predictors and for exploring fairness-accuracy and fairness-fairness tradeoffs. This framework is extremely flexible, allowing users to combine arbitrary sets of predictors, including previously trained predictors and newly trained ones, regardless of whether they are designed to satisfy fairness constraints or not. FADE thereby collapses the distinction between in-processing and post-processing approaches to building fair predictors. FADE can accommodate a wide range of disparities and allows users to minimize multiple disparities simultaneously. FADE also accommodates both observable and counterfactual outcomes, joining a very small set of existing methods for targeting counterfactual accuracy and counterfactual versions of fairness criteria like equalized odds.

Within the FADE framework, we developed three methods. The first two ``constrained'' methods allow users to minimize mean squared error subject to explicit fairness constraints, or minimize unfairness subject to an explicit constraint on the mean squared error. The third ``penalized'' method allows users to efficiently construct large sets of predictors and evaluate their risk and fairness properties. The penalized method enables users to efficiently explore fairness-accuracy and fairness-fairness tradeoffs in their problem setting, and it enables them to find a predictor with a favorable risk and fairness profile. Our theoretical results show that FADE predictors converge to optimal behavior at fast rates, and our empirical results show that in many cases, disparities can be substantially reduced with no tangible loss of accuracy--or even an increase in accuracy--relative to the unpenalized least squares solution or an existing benchmark predictor.

Although our penalized approach is designed to minimize mean squared error and to penalize certain classes of disparities, the resulting predictors can naturally be evaluated with respect to any accuracy or fairness metric. For example, users might wish to consider only binary classifiers, so they may wish to evaluate classification error on thresholded versions of the FADE predictors. The penalized approach provides a principled way to explore various fairness-accuracy spaces, even if the fairness and/or accuracy metrics of interest aren't explicitly represented in the penalized expression.

Finally, the efficiency of our penalized method relies on the particular closed form of the parameterized predictors, which arises as a result of the mean squared error and the squared fairness terms. However, any quadratic function that involves a positive definite matrix has a closed form solution. This form could be preserved under different accuracy metrics and fairness terms by, for example, adding a regularization term $\beta^T M \beta$ for some positive-definite matrix $M$. This suggests that our approach could be adapted to explicitly target other accuracy and/or fairness metrics.

\backmatter

\bmhead{Acknowledgments}
We are grateful to Alexandra Chouldechova, Aaditya Ramdas, Cosma Shalizi, Ilya Shpitser, and Larry Wasserman for comments on earlier versions of this work. This work was completed while Alan Mishler was a PhD student at Carnegie Mellon University.

\begin{appendices}

\section{Proof preliminaries} \label{appendix:preliminaries}
For convenience, we collect all the assumptions that appear in the paper, with their short descriptors:

\begin{align*}
    & \text{1. \-\ For all $n$, $\E[bb^T]$ is positive definite.} \tag{PSD outer product} \\
    & \text{2. \-\ Uniformly in $n$, } \sup_{w\in\calW} \Vert b(w) \Vert < \infty. \tag{Bounded basis norm} \\
    & \text{3. \-\ The set $\{\E[g_jb]\E[g_jb]^T\betaopt_r: j \in \mathcal{I}\}$ is linearly independent.} \tag{LICQ - population} \\
    & \text{4. \-\ } Y = DY^1 + (1-D)Y^0. \tag{Consistency} \\
    & \text{5. \-\ } \exists \delta \in (0, 1) \text{ s.t. } \Pb(\pi(W) \leq 1 - \delta) = 1. \tag{Positivity} \\
    & \text{6. \-\ } Y^0 \ind D \mid W. \tag{Ignorability} \\
    & \text{7. \-\ } \exists \gamma \in (0, 1) \text{ s.t. } \Pb(\pihat(A, X, \Rin) \leq 1 - \gamma) = 1. \tag{Bounded propensity estimator} \\
    & \text{8. \-\ } \Vert \pihat - \pi \Vert = o_\Pb(1) \text{ and } \Vert \muhat_0 - \mu_0 \Vert = o_\Pb(1) \text{ and } \Vert \nuhat_0 - \nu_0 \Vert = o_\Pb(1). \tag{Consistent nuisance estimators} \\
    & \text{9. \-\ } \Vert \pihat - \pi \Vert \Vert \muhat_0 - \mu_0 \Vert = o_\Pb(1/\sqrt{n}). \tag{Nuisance parameter rates} \\
    & \quad \-\ \Vert \pihat - \pi \Vert \Vert \nuhat_0 - \nu_0 \Vert = o_\Pb(1/\sqrt{n}). \\
    & \text{10. \-\ } \Lambda_n \subseteq \Lambda \subset \Rb^t \text{ for some compact } \Lambda. \tag{Compact $\Lambda$}
\end{align*}

Recall that for any function $f: \calZ \mapsto \Rb$, we defined $\Pn(f) = n^{-1}\sum_{j=1}^n f = \int f d\Pn(Z)$ and $\Pb(f) = \int f d\Pb(Z)$ as the sample and population expectations of $f$, so that for example $\Pb(\phihat) = \E[\phihat \mid \datatrain]$ or $\E[\phihat \mid \datatest]$ is the expected value of $\phihat(Z)$ once the relevant nuisance function estimate $\phihat$ has been constructed.

We state several lemmas that are used in the proofs for the constrained and penalized settings. The first is a restatement of Lemma 2 in \cite{kennedy_sharp_2020}.
\begin{lemma}[Kennedy, 2020] \label{LS_lemma:kennedy}  \hspace{0.25em}
Let $\fhat: \calZ \mapsto \Rb$ be a function estimated on a nuisance dataset $\data^\text{nuis}$ independent of $\Pn$, and let $f: \calZ \mapsto \Rb$ be another function. Assume $\var(\fhat - f \mid \data^\text{nuis}) < \infty$. Then
\begin{align*}
    (\Pn - \Pb)(\fhat - f) &= O_\Pb\left(\frac{\Vert \fhat - f \Vert}{\sqrt{n}}\right)
\end{align*}
\end{lemma}

\begin{lemma}[Double robustness] \label{LS_lemma:double_robustness}
Let $f: \calW \mapsto \Rb^p$ for any $p$ be a function with $\Vert f(W) \Vert \leq M < \infty$ for some $M$. Under Assumption \ref{assumption:positivity} (positivity), 
\begin{align*}
    \Vert f(W)(\phihat - \phi) \Vert &= O_\Pb\left(\Vert \mu_0 - \muhat_0 \Vert \Vert \pihat - \pi \Vert\right) \\
    \Vert f(W)(\phibarhat - \phibar) \Vert &= O_\Pb\left(\Vert \nu_0 - \nuhat_0 \Vert \Vert \pihat - \pi \Vert\right)
\end{align*}
It follows immediately that
\begin{align*}
    \Pb\left(f(W)(\phihat - \phi)\right) &= O_\Pb\left( \Vert \mu_0 - \muhat_0 \Vert \Vert \pihat - \pi \Vert\right) \\
    \Pb\left(f(W)(\phibarhat - \phibar)\right) &= O_\Pb\left( \Vert \nu_0 - \nuhat_0 \Vert \Vert \pihat - \pi \Vert\right)
\end{align*}
\end{lemma}
\begin{proof}
\begin{align*}
    \Pb\left(f(W)(\phihat - \phi)\right) &=\Pb\left(f(W)\left(\frac{1-D}{1-\pihat}(Y-\muhat_0) + \muhat_0 - \frac{1-D}{1-\pi}(Y-\mu_0) - \mu_0\right)\right) \\
     &= \Pb\left(f(W)\left(\frac{1-D}{1-\pihat}(\mu_0-\muhat_0) + \muhat_0 - \frac{1-D}{1-\pi}(\mu_0-\mu_0) - \mu_0\right)\right) \\
     &= \Pb\left(f(W)\left(\frac{1-\pi}{1-\pihat}(\mu_0-\muhat_0) + \muhat_0 - \mu_0\right)\right) \\
     &= \Pb\left(f(W)\left(\frac{(\mu_0-\muhat_0)(\pihat-\pi)}{1-\pi}\right)\right) \\
     &\leq \frac{1}{\delta}\Pb(f(W)(\mu_0-\muhat_0)(\pihat-\pi)) \\
     &\leq \frac{1}{\delta}\Vert f(W)\Vert \Vert \mu_0 - \muhat_0 \Vert \Vert \pihat - \pi \Vert \\
     &= O_\Pb\left( \Vert \mu_0 - \muhat_0 \Vert \Vert \pihat - \pi \Vert\right)
\end{align*}
where the second and third lines use iterated expectation, conditioning on $W$; the fifth line uses Assumption \ref{assumption:positivity} (positivity); and the sixth line uses the Cauchy-Schwarz inequality.
\end{proof}

\section{Proofs of propositions} \label{appendix:propositions}
\subsection{Proof of Proposition \ref{proposition:fairness_functions}}
\begin{proof}
Let $\alpha_0, \alpha_1 \in  \Rb$, let $h_0, h_1$ be mappings from $\{0, 1\} \times \Yt$ to $\{0, 1\}$, and let $g(W, \Yt) = \alpha_0\frac{h_0(A, \Yt)}{\E[h_0(A, \Yt)]} - \alpha_1\frac{h_1(A, \Yt)}{\E[h_1(A, \Yt)]}$. Then
\begin{align*}
    \E[f(W)h_0(A, \Yt)] &= \E[\E[f(W)h_0(A, \Yt) \mid h_0(A, \Yt)]] \\
                        &= \E[f(W) \mid h_0(A, \Yt) = 1]\Pb(h_0(A, \Yt) = 1) \\
                        &= \E[f(W) \mid h_0(A, \Yt) = 1]\E[h_0(A, \Yt)] \\
    \implies \E[f(W) \mid h_0(A, \Yt) = 1] &= \frac{\E[f(W)h_0(A, \Yt)]}{\E[h_0(A, \Yt)]}
\end{align*}
where $\E[h_0(A, \Yt)] > 0$ by assumption. By similar reasoning for $h_1$, it follows that
\begin{align*}
    \lvert\alpha_0\E[f(W) \mid h_0(A, \Yt) = 1] - \alpha_1\E[f(W) \mid h_1(A, \Yt) = 1] \rvert &= \lvert\E[g(W, \Yt)f(W)]\rvert
\end{align*}
as desired.
\end{proof}

\subsection{Proof of Proposition \ref{proposition:constrained_to_lagrange}}
\begin{proof}
Define the Lagrangian $L(\beta, v)$ of the risk-min program that defines $\betaopt_r$:
\begin{align*}
    L(\beta, v) &= \E[(b^T\beta - \Yt)^2] + \sum_{j=1}^t v_j \left\{\left(\E[g_j b^T\beta]\right)^2 - \epsilon_j^2\right\}
\end{align*}
Under Assumption \ref{assumption:psd}, the risk-min program is a strictly convex quadratic program, and it is feasible by construction. Therefore, the dual program has at least one solution $\lambda \in \Rpos^t$, and strong duality holds, so the primal solution $\betaopt_r$ and $\lambda$ jointly satisfy the KKT conditions. In particular, $\nabla_\beta L(\betaopt_r, \lambda) = 0$ (stationarity). By computing $\nabla_\beta L(\beta, u)$, we see that this equality holds iff
\begin{align*}
    \betaopt_r &= \left(\E[bb^T] + \sum_{j=1}^t \lambda_j \E[g_j b]\E[g_j b]^T\right)^{-1}\E[\Yt b] \\
               &= \betaopt_\lambda
\end{align*}
If Assumption \ref{assumption:licq} holds, then $\lambda$ is unique by Theorem 2 in \cite{wachsmuth_licq_2013}.
\end{proof}

\subsection{Proof of Proposition \ref{proposition:lagrange_to_constrained}}
\begin{proof}
Assumption \ref{assumption:psd} again ensures that $\betaopt_\lambda$ exists and is unique. The KKT conditions for the risk-min program are satisfied by setting $\epsilon_j^2 = (\Pn[\ghat_j b^T\betahat_\lambda])^2$, which then implies that
\begin{alignat*}{2}
    & \betaopt_\lambda = && \argmin_{\beta\in \Rb^k}\E[(b^T\beta - \Yt)^2] \\
    & && \text{subject to } \-\ (\E[g_j b^T\beta])^2 \leq (\Pn[\ghat_j b^T\betahat_\lambda])^2, \quad j = 1, \ldots t \\
              & = && \betaopt_r
\end{alignat*}
\end{proof}

\subsection{Proof of Proposition \ref{proposition:identification}}
\begin{proof}
Beginning with the risk, note that
\begin{align*}
    \var(Y^0) &= \E[(Y^0)^2] - (\E[Y^0])^2 \\
              &= \E\{\E[(Y^0)^2 \mid W]\} - (\E\{\E[Y^0 \mid W]\})^2 \\
              &= \E\{\E[Y^2 \mid W, D = 0]\} - (\E\{\E[Y \mid W, D = 0]\})^2 \\
              &= \E[\nu_0] - (\E[\mu_0])^2
\end{align*}
where the second line uses iterated expectation and the third line follows from the consistency and ignorability assumptions. We then have
\begin{align*}
    \E[(f - Y^0)^2] &= \E\{\E[f^2 - 2Y^0 + (Y^0)^2 \mid W]\} \\
                    &= \E[f^2 - 2\mu_0 + \nu_0] \\
                    &= \E[(f - \mu_0)^2] + (\E[\nu_0] - \E[\mu_0^2]) ]] \\
                    &= \E[(f - \mu_0)^2] + \var(Y^0)
\end{align*}
The last and third-to-last lines give the equalities in the proposition. Turning to the FPR-diff, by Definition \ref{definition:FPR-diff} and Proposition \ref{proposition:fairness_functions}, we have
\begin{align*}
    \E[g^\text{cFPR}f(W)] &= \E\left[\left\{ \frac{(1 - Y^0)(1 - A)}{\E[(1 - Y^0)(1 - A)]} - \frac{(1 - Y^0)A}{\E[(1 - Y^0)A]}\right\}f(W)\right] \\
    &= \E\left[\E\left\{ \frac{(1 - Y^0)(1 - A)}{\E[(1 - Y^0)(1 - A)]} - \frac{(1 - Y^0)A}{\E[(1 - Y^0)A]} \-\ \Big\vert \-\ W\right\}f(W)\right] \\
    &= \E\left[\left\{ \frac{(1 - \mu_0)(1 - A)}{\E[(1 - \mu_0)(1 - A)]} - \frac{(1 - \mu_0)A}{\E[(1 - \mu_0)A]}\right\}f(W)\right]
\end{align*}
where the last line again follows from the consistency and ignorability assumptions. The result for $g^\text{cFNR}$ follows by identical reasoning.
\end{proof}

\section{Proof of Theorem \ref{LS_thm:asymptotic_normality}} \label{sec:proof_asymptotic_normality}
We prove this theorem first, since the result is used in the proofs of the other theorems. In the observable setting, the theorem follows immediately from the central limit theorem, so the subsequent derivations are for the counterfactual setting.

\subsection{Asymptotic normality of the risk estimator} \label{subsec:asymptotic_normality_risk}
For any fixed predictor $f_\beta$, we have
\begin{align*}
    \riskhat(f_\beta) - \risk(f_\beta) =&\-\ \Pn[f_\beta^2 - (2b^T\beta)\phihat + \phibarhat] - \Pb[f_\beta^2 - (2b^T\beta)\phi + \phibar] \\
    =&\-\ (\Pn - \Pb)\left\{f_\beta^2 - 2f_\beta\phi + \phibar\right\} + \\
     &\-\ (\Pn - \Pb)\left\{2f_\beta(\phi - \phihat) + (\phibarhat - \phibar)\right\} + \\
     &\-\ \Pb\left\{2f_\beta(\phi - \phihat) + (\phibarhat - \phibar)\right\}
\end{align*}
The second term of the last equality is $O_\Pb(\Vert \muhat_0 - \mu_0 \Vert \Vert \pihat - \pi \Vert/\sqrt{n}) = o_\Pb(1/\sqrt{n})$ by Lemma \ref{LS_lemma:kennedy}, Lemma \ref{LS_lemma:double_robustness}, and Assumption \ref{assumption:nuisance_rates}. The third term is $o_\Pb(1/\sqrt{n})$ by Lemma \ref{LS_lemma:double_robustness} and Assumption \ref{assumption:nuisance_rates}. We therefore have
\begin{align}
    \riskhat(f_\beta) - \risk(f_\beta) &= (\Pn - \Pb)\left\{f_\beta^2 - 2f_\beta\phi + \phibar\right\} + o_\Pb(1/\sqrt{n}) \label{LS_eq:true_risk_clt}
\end{align}
and the result follows by the central limit theorem.

\subsection{Asymptotic normality of the unfairness estimators} \label{subsec:asymptotic_normality_unfairness}
Since $g^\text{rate}$ does not depend on the outcome, we have $\ghat^\text{ind} = g^\text{rate}$, and the result follows immediately from the central limit theorem. We now prove the result for $g^\text{FPR}$ in the counterfactual setting. We have
\begin{align}
    \Pn(\ghat_j f_\beta) - \Pb(g_j f_\beta) &= \left\{\frac{\Pn[\gammahat_0 f_\beta]}{\Pn[\gammahat_0]} - \frac{\Pn[\gammahat_1 f_\beta]}{\Pn[\gammahat_1]}\right\} - \left\{\frac{\Pb[\gamma_0 f_\beta]}{\Pb[\gamma_0]} - \frac{\Pb[\gamma_1 f_\beta]}{\Pb[\gamma_1]}\right\} \label{LS_eq:unfairness_expansion1}
\end{align}
Considering just the $\gammahat_0$ and $\gamma_0$ terms, we have
\begin{align}
    & \frac{\Pn[\gammahat_0 f_\beta]}{\Pn[\gammahat_0]} - \frac{\Pb[\gamma_0 f_\beta]}{\Pb[\gamma_0]} \nonumber = \frac{\Pn[\gammahat_0 f_\beta]\Pb[\gamma_0] - \Pb[\gamma_0]\Pn[\gammahat_0]}{\Pn[\gammahat_0]\Pb[\gamma_0]} \nonumber \\
    &= \frac{\Pb[\gamma_0]\Big(\Pn[\gammahat_0  f_\beta] - \Pb[\gamma_0  f_\beta]\Big) - \Pb[\gamma_0  f_\beta]\Big(\Pn[\gammahat_0] - \Pb[\gamma_0]\Big)}{\Pn[\gammahat_0]\Pb[\gamma_0]} \nonumber \\
    &= \Pn[\gammahat_0]^{-1}\Big\{\underbrace{\left(\Pn[\gammahat_0  f_\beta] - \Pb[\gamma_0  f_\beta]\right)}_{(1)} - \frac{\Pb[\gamma_0 f_\beta]}{\Pb[\gamma_0]}\underbrace{(\Pn[\gammahat_0] - \Pb[\gamma_0])}_{(2)}\Big\} \label{LS_eq:unfairness_expansion2}
\end{align}
Terms (1) and (2) in \eqref{LS_eq:unfairness_expansion2} can be expanded as follows:
\begin{align*}
    (1) &= (\Pn - \Pb)\gamma_0 f_\beta + (\Pn - \Pb)((\gammahat_0 - \gamma_0)f_\beta) + \Pb((\gammahat_0 - \gamma_0) f_\beta) \\
    (2) &= (\Pn - \Pb)\gamma_0 + (\Pn - \Pb)(\gammahat_0 - \gamma_0) + \Pb(\gammahat_0 - \gamma_0)
\end{align*}
In both these expressions, the second term is $O_\Pb(\Vert \phihat - \phi \Vert/\sqrt{n}) = o_\Pb(1/\sqrt{n})$ by Lemma \ref{LS_lemma:kennedy} and Assumption \ref{assumption:nuisance_rates}, and the third term is $o_\Pb(1/\sqrt{n})$ by Lemma \ref{LS_lemma:double_robustness} and Assumption \ref{assumption:nuisance_rates}. Under Assumption \ref{assumption:bounded_propensity_estimator}, $\Pn[\gammahat_0]^{-1}$ is bounded, while $\Pb[\gamma_0 f_\beta]/\Pb[\gamma_0]$ is bounded under Assumption \ref{assumption:positivity}. Therefore, we can rewrite \eqref{LS_eq:unfairness_expansion1} as
\begin{align*}
    & \Pn[\gammahat_0]^{-1}(\Pn - \Pb)\left\{\gamma_0\left(f_\beta - \frac{\Pb[\gamma_0 f_\beta]}{\Pb[\gamma_0]}\right)\right\} + o_\Pb(1/\sqrt{n}) \\
    = & \-\ \Pn[\gammahat_0]^{-1}(\Pn - \Pb)\eta_0 + o_\Pb(1/\sqrt{n}) \label{LS_eq:unfairness_expansion_final}
\end{align*}
We can therefore rewrite $\Pn(\ghat_j f_\beta) - \Pb(g_j f_\beta)$ as
\begin{align*}
    \Pn[\gammahat_0]^{-1}(\Pn - \Pb)\eta_0 + \Pn[\gammahat_1]^{-1}(\Pn - \Pb)\eta_1 +  o_\Pb(1/\sqrt{n})
\end{align*}
Note that the analysis of term (2) in \eqref{LS_eq:unfairness_expansion2} yields that $\Pn[\gammahat_0] - \Pb[\gamma_0] = o_\Pb(1)$. Applying the central limit theorem to the vector $(\eta_0, \eta_1)$, followed the continuous mapping theorem, Slutsky's theorem, and the delta method, we have
\begin{align*}
    \sqrt{n}\left(\Pn(\ghat_j f_\beta) - \Pb(g_j f_\beta)\right) \xrightarrow{d} N\left(0, \var\left(\Pb(\gamma_0)^{-1}\eta_0 - \Pb(\gamma_1)^{-1}\eta_1\right)\right)
\end{align*}
as desired. The result for $g^\text{FNR}$ follows by identical reasoning.

\section{Proofs for constrained FADE} \label{appendix:proofs_constrained}
We state two additional lemmas that are only used in the constrained setting. The first lemma gives sufficient conditions under which the optimal value of an estimated convex problem converges at a particular rate to the optimal value of the target convex program. It is a adaptation of Theorem 3.5 in \cite{shapiro_asymptotic_1991} that follows immediately from Theorems 2.1 and 3.4 in that same paper.

\begin{lemma} \label{LS_lemma:shapiro} (Shapiro, 1991) \hspace{0.25em}
Let $\Theta$ be a compact subset of $\Rb^k$. Let $C(\Theta)$ denote the set of continuous real-valued functions on $\Theta$, with $\mathcal{L} = C(\Theta) \times \ldots \times C(\Theta)$ the $r$-dimensional Cartesian product. Let $\psi(\theta) = \left(\psi_{0}, \ldots, \psi_{r}\right) \in \mathcal{L}$ be a vector of convex functions. Consider the quantity $\alpha^*$ defined as the solution to the following convex optimization program: 
\begin{align*}
    \alpha^* = \min_{\theta\in\Theta} \quad & \psi_0(\theta) \\
              \text{subject to }
                 & \psi_j(\theta) \leq 0, \-\ j = 1, \ldots, r
\end{align*}
Assume that Slater's condition holds, so that there is some $\theta \in \Theta$ for which the inequalities are satisfied and non-affine inequalities are strictly satisfied, i.e. $\psi_j(\theta) < 0$ if $\psi_j$ is non-affine. Now consider a sequence of approximating programs, for $n = 1, 2, \ldots$:
\begin{align*}
    \widehat{\alpha}_n = \min_{\theta\in\Theta} \quad & \psihat_{0n}(\theta) \\
              \text{subject to }
                 & \psihat_{jn} (\theta) \leq 0, \-\ j = 1, \ldots, r
\end{align*}
with $\psihat_n(\theta) := \left(\psihat_{0n}, \ldots, \psihat_{rn}\right) \in \mathcal{L}$. Assume that $f(n)(\psihat_n - \psi)$ converges in distribution to a random element $W \in \mathcal{L}$ for some real-valued function $f(n)$. Then:
\begin{align*}
    f(n)(\widehat{\alpha}_n - \alpha_0) \rightsquigarrow L
\end{align*}
for a particular random variable $L$. It follows that $\widehat{\alpha}_n - \alpha_0 = O_\Pb(1/f(n))$.
\end{lemma}

\subsection{Intermediate result}
The next lemma applies Lemma \ref{LS_lemma:shapiro} to the risk-min and unfair-min settings. For analytical purposes, we suppose that for each $k$, the quantities $\betaopt_r, \betahat_r, \betaopt_u, \betahat_u$ are constrained to lie in some (arbitrarily large) compact set $\Theta_k \subseteq \Rb^k$. Since $k \not\rightarrow \infty$, ultimately $\Theta_k$ is fixed to some set $\Theta$. For example, $\Theta$ could be given by box constraints defined by the largest and smallest numbers the machine can represent. Since this is a device for asymptotic analysis, we do not express it in the actual optimization. Under Assumption \ref{assumption:bounded_basis}, it follows that $b^T\beta$ is uniformly bounded in $\Theta$. (Recall that in practice, the output of any predictor will be truncated to lie in $[\ell_y, u_y]$.)

Per Proposition \ref{proposition:identification} and Remark \ref{remark:nu0}, we can write the objective function for the risk-min parameter $\betaopt_r$ equivalently as $\Pb[(b^T\beta)^2 - 2(b^T\beta)\phi]$ in the counterfactual setting (or $\Pb[(b^T\beta)^2 - 2(b^T\beta)Y]$ in the observable setting), since the term $\Pb[\phi^2]$ (or $\Pb[Y^2]$) drops out of the minimization. We utilize this form for analysis.

Denote by $\psi_0, \ldots \psi_{t+1}$ and $\psihat_0, \ldots \psihat_{t+1}$ the population and empirical risk and unfairness functions, each of which is a mapping from $\Theta$ to $\Rb$. For the counterfactual setting, these are given by
\begin{equation}
\begin{aligned}[c]
    & \psi_0(\beta) = \Pb[(b^T\beta)^2 - 2(b^T\beta)\phi] \\
    & \psi_j(\beta) = (\Pb[g_j b^T\beta)])^2 \\
    & \psi_{t+1}(\beta) = \Pb[(b^T\beta)^2 - (2b^T\beta)\phi + \phibar]
\end{aligned}
\qquad
\begin{aligned}[c]
    & \psihat_0 = \Pn[(b^T\beta)^2 - 2(b^T\beta)\phihat)] \\
    & \psihat_j(\beta) = (\Pn[\ghat_j b^T\beta])^2, \qquad j = 1, \ldots t \\
    & \psi_{t+1}(\beta) = \Pn[(b^T\beta)^2 - (2b^T\beta)\phihat + \phibarhat]
\end{aligned}    
\end{equation}
The observable setting substitutes $Y$ for $\phi$, $Y^2$ for $\phibar$, and $g_j$ for $\ghat_j$. Let $\calC(\Theta)$ denote the set of continuous real-valued functions on $\Theta$, with $\calL(\Theta) = \calC(\Theta) \times \ldots \times \calC(\Theta)$ the Cartesian product (with suitable dimension). Let $\psi_{(\bullet)}, \psihat_{(\bullet)}: \Theta \mapsto \calL(\Theta)$ be the vectors of functions that define the population and empirical optimization problem, for $\bullet \in \{r, u\}$, representing the risk-min and unfair-min problems. That is, for risk-min, define
\begin{align*}
    \psi_{(r)} &= \big(\psi_0(\beta), \-\ \psi_1(\beta), \ldots, \psi_t(\beta)\big)^T \\
    \psihat_{(r)} &= \left(\psihat_0(\beta), \-\ \psihat_1(\beta), \ldots, \psihat_t(\beta)\right)^T
\end{align*}
and for unfair-min, define
\begin{align*}
    \psi_{(u)} &= \left(\sum_{j=1}^t \alpha_j \psi_j(\beta), \-\ \psi_{t+1}(\beta)\right)^T \\
    \psihat_{(u)} &= \left(\sum_{j=1}^t \alpha_j \psihat_j(\beta), \-\ \psihat_{t+1}(\beta)\right)^T
\end{align*}
The first element in each of $\psi_{(r)}$ and $\psi_{(u)}$ is the objective function, and the remaining elements are the constraint functions.

\begin{lemma}[Convergence rates of estimated functions] \label{LS_lemma:function_convergence}
Under Assumptions \ref{assumption:psd} and \ref{assumption:bounded_basis} for the observable setting, and Assumptions \ref{assumption:psd} and \ref{assumption:consistency}--\ref{assumption:ignorability} for the counterfactual setting, there exist random elements $C_r, C_u$ taking values in the appropriate space $\calL(\Theta)$ such that 
\begin{align*}
     & \sqrt{n}(\psihat_{(r)} - \psi_{(r)}) \xrightarrow{d} C_r \\
     & \sqrt{n}(\psihat_{(u)} - \psi_{(u)}) \xrightarrow{d} C_u
\end{align*}
where the convergence is in $L_2$ norm.
\end{lemma}

\begin{proof}
We will utilize the fact that the class $\{b^T\beta: \beta \in \Theta\}$ is $\Pb$-Donsker, since $b^T\beta$ is parametric and Lipschitz in $\beta$ under Assumption \ref{assumption:bounded_basis}.

In the observable setting, we have
\begin{align*}
    \psihat_{(r)} - \psi_{(r)} &= (\Pn - \Pb)\left((b^T\beta)^2 - 2(b^T\beta)Y, \-\ g_1 b^T\beta, \-\ \ldots \-\ g_t b^T\beta\right)^T
\end{align*}
so that the result follows immediately from the central limit theorem and the Donsker condition. We now turn to the counterfactual setting. First, consider the objective function $\psi_0$.
\begin{align*}
    \psihat_0(\beta) &- \psi_0(\beta) = \Pn\left\{(b^T\beta)^2 - 2(b^T\beta)\phihat\right\} - \Pb\left\{(b^T\beta)^2 - 2(b^T\beta)\phi\right\} \\
    &= (\Pn - \Pb)\left\{(b^T\beta)^2\right\} - \left\{\Pn(b^T\beta\phihat) - \Pb(b^T\beta\phi)\right\} \\
    &= (\Pn - \Pb)\left\{(b^T\beta)^2 - \phi\right\} + (\Pn - \Pb)(2(b^T\beta)(\phi - \phihat)) + \Pb(2(b^T\beta)(\phi - \phihat))
\end{align*}
The second term is $O_\Pb(\Vert \muhat_0 - \mu_0 \Vert \Vert \pihat - \pi \Vert/\sqrt{n}) = o_\Pb(1/\sqrt{n})$ by Lemma \ref{LS_lemma:kennedy}, Lemma \ref{LS_lemma:double_robustness}, and Assumption \ref{assumption:nuisance_rates}. The third term is $o_\Pb(1/\sqrt{n})$ by Lemma \ref{LS_lemma:double_robustness} and Assumption \ref{assumption:nuisance_rates}. We therefore have
\begin{align}
     \psihat_0(\beta) - \psi_0(\beta) &= (\Pn - \Pb)\left\{(b^T\beta)^2 - \phi\right\} + o_\Pb(1/\sqrt{n}) \label{LS_eq:objective_risk_clt}
\end{align}
We now consider the unfairness functions $\psi_j, j = 1, \ldots t$. We have
\begin{align}
    &\psihat_j(\beta) - \psi_j(\beta) = \left\{\Pn(\ghat_j b^T\beta) + \Pb(g_j b^T\beta)\right\}\left\{\Pn(\ghat_j b^T\beta) - \Pb(g_j b^T\beta)\right\} \nonumber \\
    &= \left\{\Pn(\ghat_j b^T\beta) + \Pb(g_j b^T\beta)\right\}
    \left(\Pn(\gammahat_0)^{-1}, \-\ \Pn(\gammahat_1)^{-1}\right)
    (\Pn - \Pb)\begin{pmatrix}
    \eta_0 \\ \eta_1
    \end{pmatrix} + 
    o_\Pb(1/\sqrt{n}) \label{LS_eq:fairness_clt}
\end{align}
where the second line follows the derivation in Section \ref{subsec:asymptotic_normality_unfairness}, coupled with the fact that $\Pn(\ghat_j b^T\beta) + \Pb(g_j b^T\beta) = o_\Pb(1)$. Finally, the analysis of $\psi_{t+1}$ is already given in Section \ref{subsec:asymptotic_normality_risk}:
\begin{align}
    \psihat_{t+1} - \psi_{t+1} &= (\Pn - \Pb)\left((b^T\beta)^2 - 2(b^T\beta)\phi + \phibar\right) + o_\Pb(1/\sqrt{n}) \label{LS_eq:true_risk_clt2}
\end{align}
Suppose we have a single fairness function $g_j$. Combining \eqref{LS_eq:objective_risk_clt}, \eqref{LS_eq:fairness_clt}, and \eqref{LS_eq:true_risk_clt2}, we have shown that $\psihat_{(r)} - \psi_{(r)}$ can be written as
\begin{align*}
    &\psihat_{(r)} - \psi_{(r)} = M
    (\Pn - \Pb)
    \begin{pmatrix}
    (b^T\beta)^2 - \phi \\
    \eta_0 \\
    \eta_1
    \end{pmatrix} + 
    \begin{pmatrix}
    o_\Pb(1/\sqrt{n}) \\ o_\Pb(1/\sqrt{n})
    \end{pmatrix}, \text{ where } \\
    & M = \begin{bmatrix}
            1 & 0 & 0 \\
            0 & \left\{\Pn(\ghat_j b^T\beta) + \Pb(g_j b^T\beta)\right\}\Pn(\gammahat_0)^{-1} & -\left\{\Pn(\ghat_j b^T\beta) + \Pb(g_j b^T\beta)\right\}\Pn(\gammahat_1)^{-1}
            \end{bmatrix}
\end{align*}
Applying the central limit theorem, Slutsky's theorem, the continuous mapping theorem, and the delta method, we have that $\sqrt{n}(\psihat_{(r)}(\beta) - \psi_{(r)}(\beta))$ converges to a normal distribution for any fixed $\beta$. Under the Donsker condition, this convergence is uniform over $\beta$, and $\sqrt{n}(\psihat_{(r)} - \psi_{(r)})$ converges to a Gaussian process. Equivalent reasoning applies in the case of multiple fairness functions, and to $\sqrt{n}(\psihat_{(u)} - \psi_{(u)})$.
\end{proof}

We now prove Theorems \ref{LS_thm:excess_risk_constrained} and \ref{LS_thm:excess_unfairness_constrained}. The two proofs proceed along similar lines. We will again utilize the fact that $\{b^T\beta: \beta \in \Theta\}$ is $\Pb$-Donsker, so that the empirical process $\{\sqrt{n}(\Pn - \Pb)(b^T\beta): \beta \in \Theta\}$ converges to a Gaussian process.

\subsection{Proof of Theorem \ref{LS_thm:excess_risk_constrained} (Excess risk in constrained FADE)}
\begin{proof}
We consider the risk-min problem first. We expand the excess risk by adding and subtracting the objective function at the solution $\betahat_r$:
\begin{align*}
    &\Pb\left[(b^T\betahat_r)^2 - 2(b^T\betahat_r)\phihat\right] - \Pb\left[(b^T\betaopt_r)^2 - 2(b^T\betaopt_r)\phi\right]
    \\
    =& \-\ \Pb\left[(b^T\betahat_r)^2 - 2(b^T\betahat_r)\phihat\right] -  \Pn\left[(b^T\betahat_r)^2 - 2(b^T\betahat_r)\phihat\right] \-\ + \\
    & \-\ \Pn\left[(b^T\betahat_r)^2 - 2(b^T\betahat_r)\phihat\right] - \Pb\left[(b^T\betaopt_r)^2 - 2(b^T\betaopt_r)\phi\right]
\end{align*}
The second term is $O_\Pb(1/\sqrt{n})$ by Lemma \ref{LS_lemma:function_convergence} and Shapiro's theorem. The first term is just $\psi_0(\betahat_r) - \psihat_0(\betahat_r)$, which is $O_\Pb(1/\sqrt{n})$ by \eqref{LS_eq:objective_risk_clt} in the proof of Lemma \ref{LS_lemma:function_convergence} coupled with the Donsker condition. Hence, the excess risk is $O_\Pb(1/\sqrt{n})$, as claimed.

We now turn to the unfair-min problem. The excess risk is
\begin{align*}
         & \-\ \Pb[(b^T\betahat_u)^2 - 2(b^T\betahat_u)\phi + \phibar)] - \epsilon^2 \\
    \leq & \-\ \Pb[(b^T\betahat_u)^2 - 2(b^T\betahat_u)\phi + \phibar)] - \Pn[(b^T\betahat_u)^2 - 2(b^T\betahat_u)\phihat + \phibarhat)] \\
    = -&(\Pn - \Pb)\left[(b^T\betahat_u)^2 + (2b^T\betahat_u)\phi + \phibar\right] \-\ + \\
    & (\Pn - \Pb)\left[(2b^T\betahat_u)(\phihat - \phi) + (\phibarhat - \phibar)\right] \-\ + \\
    & \Pb\left[(2b^T\betahat_u)(\phihat - \phi) + (\phibarhat - \phibar)\right]
\end{align*}
The first term is $O_\Pb(1/\sqrt{n})$ by the central limit theorem and the Donsker condition. The second term is $O_\Pb(\Vert \phihat - \phi \Vert/\sqrt{n}) = o_\Pb(1/\sqrt{n})$ by Lemma \ref{LS_lemma:kennedy}, Lemma \ref{LS_lemma:double_robustness}, and Assumption \ref{assumption:consistent_nuisance_estimators}. The last term is $o_\Pb(1/\sqrt{n})$ by Lemma \ref{LS_lemma:double_robustness}. The excess risk is therefore $O_\Pb(1/\sqrt{n})$, as claimed.
\end{proof}

\subsection{Proof of Theorem \ref{LS_thm:excess_unfairness_constrained} (Excess unfairness in constrained FADE)}
\begin{proof}
We consider the unfair-min problem first. We expand the excess unfairness by adding and subtracting the objective function at the solution $\betahat_u$:
\begin{align}
    & \sum_{j=1}^t \alpha_j (\Pb[g_j b^T\betahat_u])^2 - \sum_{j=1}^t \alpha_j (\Pb[g_j b^T\betaopt_u])^2 \nonumber \\
    &= \sum_{j=1}^t \alpha_j \left\{(\Pb[g_j b^T\betahat_u])^2 - (\Pn[\ghat_j b^T\betahat_u])^2\right\} + \sum_{j=1}^t \alpha_j\left\{(\Pn[\ghat_j b^T\betahat_u])^2 - (\Pb[g_j b^T\betaopt_u])^2 \right\} \label{LS_eq:excess_unfairness_expansion}
\end{align}
Again, the second term is $O_\Pb(1/\sqrt{n})$ by Lemma \ref{LS_lemma:function_convergence} and Shapiro's theorem. The first term is equal to $\sum_{j=1}^t \alpha_j (\psi_j(\betahat) - \psihat_j(\betahat))$, which is $O_\Pb(1/\sqrt{n})$ by \eqref{LS_eq:fairness_clt} in the proof of Lemma \ref{LS_lemma:function_convergence} coupled with the Donsker condition. The excess unfairness is therefore $O_\Pb(1/\sqrt{n})$, as claimed.

We now turn to the risk-min problem. The excess unfairness for constraint $j$ is
\begin{align*}
    & \-\ (\Pb[g_j b^T\betahat_u])^2 - \epsilon^2 \\
    \leq & \-\ (\Pb[g_j b^T\betahat_u])^2 - (\Pn[\ghat_j b^T\betahat_u])^2 \\
    = & \-\ O_\Pb(1/\sqrt{n})
\end{align*}
where the last line simply uses the analysis for term (1) from \eqref{LS_eq:excess_unfairness_expansion}. The excess unfairness is therefore $O_\Pb(1/\sqrt{n})$, as claimed.
\end{proof}

\section{Proofs for penalized FADE}
Throughout this section, let
\begin{align*}
    \boldQhat_\lambda &= \Pn(bb^T) + \sum_{j=1}^t \lambda_j\Pn(\ghat_j b)\Pn(\ghat_j b)^T \\
    \boldQ_\lambda &= \Pb(bb^T) + \sum_{j=1}^t \lambda_j\Pb(g_j b)\Pb(g_j b)^T
\end{align*}
so that
\begin{align*}
    & \beta_\lambda = \boldQ_\lambda^{-1}\Pb(b\Yt) \\
    & \betahat_\lambda = \begin{cases} \boldQhat_\lambda^{-1}\Pn(bY) & \text{(Observable)} \\
                                     \boldQhat_\lambda^{-1}\Pn(b\phihat)      & \text{(Counterfactual)}
                        \end{cases}
\end{align*}
Under the assumptions of Theorems \ref{LS_thm:series_rate_risk} and \ref{LS_thm:series_rate_unfairness}, we prove several preliminary results that are used in the theorem proofs.

\begin{lemma} (Bounded norm for $\boldQhat_\lambda^{-1}$). \label{LS_lemma:Qhat_inv_norm}
\begin{align*}
    \Pb(\Vert \boldQhat_\lambda^{-1} \Vert) \leq C \rightarrow 1 \text{ for some constant } C 
\end{align*}
\end{lemma}
\begin{proof}
This follows from Assumption \ref{assumption:psd} plus the consistency of $\boldQhat_\lambda$ for $\boldQ_\lambda$. It follows that $\Vert \boldQhat_\lambda^{-1} \Vert = O_\Pb(1)$.
\end{proof}

\begin{lemma} \label{LS_lemma:penalized_estimator_convergence}
Fix a $\lambda \in \Lambda$. Then
\begin{align*}
     \Vert \betahat_\lambda - \betaopt_\lambda \Vert &= \begin{cases}
     O_\Pb(\sqrt{1/n}) & \text{(Observable)} \\
     O_\Pb(\sqrt{1/n}) + O_\Pb\left(h(n)\right) & \text{(Counterfactual)}
     \end{cases}
\end{align*}
\end{lemma}

\begin{proof}
In the observable setting, we have
\begin{align*}
    \betahat_\lambda - \betaopt_\lambda &= (\boldQhat_\lambda^{-1} - \boldQ_\lambda^{-1})\Pb(bY) \\
                     &= \boldQhat_\lambda^{-1}(\boldQ_\lambda - \boldQhat_\lambda)\boldQ^{-1}\Pb(bY) \\
                     &= \boldQhat_\lambda^{-1}(\boldQ_\lambda - \boldQhat_\lambda)\betaopt_\lambda \\
     &= \boldQhat_\lambda^{-1}\Big\{(\Pn - \Pb)(bb^T\betaopt_\lambda) \-\ + \\
     & \qquad \qquad \sum_{j=1}^t \left[\lambda_j (\Pn - \Pb)(bg_j)\Pb(g_jb^T\betaopt_\lambda) + \Pn(bg_j)(\Pn - \Pb)(g_jb^T\betaopt_\lambda)\right]\Big\}                    
\end{align*}
The norm of each term in the braces is $O_\Pb(1/\sqrt{n})$ by the central limit theorem. By Lemma \ref{LS_lemma:Qhat_inv_norm}, $\boldQhat_\lambda^{-1}$ doesn't contribute to the rate, so $\Vert \betahat_\lambda - \betaopt_\lambda \Vert = O_\Pb(1/\sqrt{n})$ as claimed.

In the counterfactual setting, we have
\begin{align}
    \betahat_\lambda - \betaopt_\lambda &= \boldQhat_\lambda^{-1}\Pn(b\phihat) - \boldQ_\lambda^{-1}\Pb(b\phi) \nonumber \\
                     &= (\boldQhat_\lambda^{-1} - \boldQ_\lambda^{-1})\Pb(b\phi) + \boldQhat_\lambda^{-1}(\Pn(b\phihat) - \Pb(b\phi)) \nonumber \\
                     &= \boldQhat_\lambda^{-1}(\boldQ_\lambda - \boldQhat_\lambda)Q^{-1}\Pb(b\phi) + \boldQhat_\lambda^{-1}(\Pn(b\phihat) - \Pb(b\phi)) \nonumber \\
                     &= \underbrace{\boldQhat_\lambda^{-1}(\boldQ_\lambda - \boldQhat_\lambda)\betaopt_\lambda}_{(1)} + \underbrace{\boldQhat_\lambda^{-1}(\Pn(b\phihat) - \Pb(b\phi))}_{(2)} \label{LS_eq:betahat_expansion}
\end{align}
The norm of term (2) in \eqref{LS_eq:betahat_expansion} is $O_\Pb(1/\sqrt{n}) + O_\Pb(h(n))$ by Lemma \ref{LS_lemma:double_robustness} and Lemma \ref{LS_lemma:Qhat_inv_norm}. For term (1), ignoring the leading $\boldQhat_\lambda^{-1}$ for now, we have
\begin{align*}
    & (\boldQ_\lambda - \boldQhat_\lambda)\betaopt_\lambda = \underbrace{(\Pn - \Pb)(bb^T\betaopt_\lambda)}_{(a)} \-\ + \\ 
    & \qquad \sum_{j=1}^t \lambda_j\big[\underbrace{(\Pn(b\ghat_j) - \Pb(bg_j))\Pb(g_jb^T\betaopt_\lambda)}_{(b)} + \underbrace{\Pn(b\ghat_j)(\Pn(\ghat_j b^T\betaopt_\lambda) - \Pb(gb^T\betaopt_\lambda)}_{(c)}\big]
\end{align*}
The norm of term (a) is $O_\Pb\left(\sqrt{1/n}\right)$ by the central limit theorem. Terms (b) and (c) decompose as follows:
\begin{align*}
    (b) &= \Pb(g_jb^T\betaopt_\lambda)\left\{(\Pn - \Pb)(bg_j) + (\Pn - \Pb)(b(\ghat_j - g_j)) + \Pb(b(\ghat_j - g_j))\right\} \\
    (c) &= \Pn(b\ghat_j)\left\{(\Pn - \Pb)(g_j b^T\betaopt_\lambda) + (\Pn - \Pb)(\ghat_j - g_j)(b^T\betaopt_\lambda) + \Pb((\ghat_j - g_j)b^T\betaopt_\lambda)\right\}
\end{align*}
The norms of the first term in braces in each of these two expressions is $O_\Pb(1/\sqrt{n})$ by the central limit theorem. The norm of the second term is $o_\Pb(1/\sqrt{n})$ by Lemma \ref{LS_lemma:kennedy} and Assumption \ref{assumption:consistent_nuisance_estimators}. The norm of the third term is $O_\Pb(1/\sqrt{n}) + O_\Pb(h(n))$ by Lemma \ref{LS_lemma:double_robustness}. Using Lemma \ref{LS_lemma:Qhat_inv_norm}, the consistency of $\Pn(b\ghat_j)$ for $\Pb(bg_j)$, and the boundedness of $\Pb(g_jb^T\betaopt_\lambda)$, we have
\begin{align*}
    \Vert \betahat_\lambda - \betaopt_\lambda \Vert &= O_\Pb(\sqrt{1/n}) + O_\Pb\left(h(n)\right)
\end{align*}
as claimed.
\end{proof}

We now prove the two theorems. We will use the fact that under Assumption \ref{assumption:psd}, $\betahat_\lambda - \betaopt_\lambda$ is Lipschitz in $\lambda$, and $\Lambda$ is compact, so the set $\{\betahat_\lambda - \betaopt_\lambda: \lambda \in \Lambda\}$ is Donsker.

\subsection{Proof of Theorem \ref{LS_thm:series_rate_risk} (Excess risk in penalized FADE)}
Fix a $\lambda \in \Lambda$. We have
\begin{align*}
    \Pb\left[\left(b^T\betahat_\lambda - \Yt\right)^2\right] &- \Pb\left[\left(b^T\betaopt_\lambda - \Yt\right)^2\right] = \Vert b^T\betahat_\lambda - \Yt\Vert^2 - \Vert b^T\betaopt_\lambda - \Yt \Vert^2 \\
    & \hspace{-1em} = \left(\Vert b^T\betahat_\lambda - \Yt\Vert - \Vert b^T\betaopt_\lambda - \Yt \Vert\right)\left(\Vert b^T\betahat_\lambda - \Yt\Vert + \Vert b^T\betaopt_\lambda - \Yt \Vert\right)
\end{align*}
Since $\betahat_\lambda$ is consistent for $\betaopt_\lambda$ (by Lemma \ref{LS_lemma:penalized_estimator_convergence}), the second factor is $O_\Pb(1)$, so we can just consider the first factor.
\begin{align*}
    \Vert b^T\betahat_\lambda - \Yt\Vert - \Vert b^T\betaopt_\lambda - \Yt \Vert &\leq \Vert b^T\betahat_\lambda - b^T\betaopt_\lambda \Vert \\
    &= O_\Pb\left(\Vert \betahat_\lambda - \betaopt_\lambda \Vert\right) \\
    &= O_\Pb(\sqrt{1/n}) + O_\Pb(h(n))
\end{align*}
where the first line uses the reverse triangle inequality, the second line uses Assumption \ref{assumption:bounded_basis}, and the third line uses Lemma \ref{LS_lemma:penalized_estimator_convergence}. Under the Donsker condition, the convergence is uniform over $\Lambda$:
\begin{align*}
    \sup_{\lambda\in\Lambda}\left\{\Pb\left[\left(b^T\betahat_\lambda - \Yt\right)^2\right] - \Pb\left[\left(b^T\betaopt_\lambda - \Yt\right)^2\right]\right\} &= O_\Pb(\sqrt{1/n}) + O_\Pb(h(n))
\end{align*}

\subsection{Proof of Theorem \ref{LS_thm:series_rate_unfairness} (Excess unfairness in penalized FADE)}
Fix a $\lambda \in \Lambda$. The excess unfairness for $g_j$ is
\begin{align*}
    \Pb\left[g_jb^T\betahat_\lambda\right] - \Pb\left[g_jb^T\betaopt_\lambda\right] &= \Pb[g_jb^T(\betahat_\lambda - \betaopt_\lambda)] \\
    &\leq \Pb[\lvert g_jb^T(\betahat_\lambda - \betaopt_\lambda)\rvert] \\
    &= \sqrt{(\Pb[\lvert g_jb^T(\betahat_\lambda - \betaopt_\lambda)\rvert])^2} \\
    &\leq \sqrt{\Pb[(g_jb^T(\betahat_\lambda - \betaopt_\lambda))^2]} \\
    &=O_\Pb\left( \sqrt{\Pb[(\betahat_\lambda - \betaopt_\lambda)^2]}\right) \\
    &= \Vert \betahat_\lambda - \betaopt_\lambda \Vert \\
    &= O_\Pb(\sqrt{1/n}) + O_\Pb(h(n))
\end{align*}
where the last line uses Lemma \ref{LS_lemma:penalized_estimator_convergence}. Under the Donsker condition, the convergence is uniform over $\Lambda$:
\begin{align*}
    \sup_{\lambda\in\Lambda}\left\{\max_{j \in 1, \ldots, t}\left(\Pb\left[g_jb^T\betahat_\lambda\right] - \Pb\left[g_jb^T\betaopt_\lambda\right]\right)\right\} &= O_\Pb(\sqrt{1/n}) + O_\Pb(h(n))
\end{align*}

\section{Bases with dimension \texorpdfstring{$k \geq n$}{larger than n}} \label{appendix:growing_bases}
We can generalize our estimators slightly to accommodate case where where $k \geq n$, meaning the dimension of the basis is greater than the sample size, as is the case for example with smoothing splines or RKHSs. We simply add an appropriate penalty matrix term $\lambda_0\beta^T\boldK\beta$ to the penalized estimator expression or to the objective function for the constrained estimators, where $\boldK$ is a $k \times k$ smoothing matrix. In the former case, for example, the estimator $\betahat_\lambda$ becomes
\begin{align}
    & \-\ \argmin_{\beta\in\Rb^k} \Pn[(b^T\beta - \phihat)^2]  + \lambda_0\beta^T\boldK\beta + \sum_{j=1}^t \lambda_j (\Pn[\ghat_j b^T\beta])^2 \\
    = & \-\ \left(\Pn(bb^T) + \lambda_0\boldK + \sum_{j=1}^t \lambda_j\Pn(\ghat_j b)\Pn(\ghat_j b)^T\right)^{-1}\Pn(b\phihat) \label{LS_eq:closed_form_with_smoothing}
\end{align}
For instance, in a smoothing spline setting, $b$ represents a spline basis, and $\boldK_{ij} = \int_\calW b_i^{''}(w)b_j^{''}(w)dw$. In an RKHS, we'd have $b_i = \sum_{j=1}^n k(\cdot, w_j)$ and $\boldK_{ij} = k(w_i, w_j)$. In a ridge setting we'd have $\boldK = I$. The penalty term ensures the invertibility of the large matrix in \eqref{LS_eq:closed_form_with_smoothing}, and it preserves the fast computability of a large set of solutions $\betahat_\lambda$. 

This penalty term may also be useful to prevent overfitting even in cases where $k < n$, if $k$ is close to $n$ or if the basis is very expressive. 

\end{appendices}

\bibliography{references}

\end{document}